\documentclass[twoside,11pt]{article}

\usepackage{jmlr2e}
\usepackage{amssymb}
\usepackage{natbib}
\usepackage{color}
\usepackage{amsfonts}       
\usepackage{url}
\usepackage{amsmath}
\usepackage{amssymb}
\usepackage{wrapfig}
\usepackage{mathtools}
\usepackage{tikz}
\usepackage{subfig}
\usepackage{algorithmic}
\usepackage{subfig}
\usepackage{footmisc}
\usepackage{bibentry}
\nobibliography*
\usepackage{mathtools}
\DeclarePairedDelimiter{\ceil}{\lceil}{\rceil}
\usepackage{physics}
\usepackage{mathrsfs}
\mathtoolsset{showonlyrefs}
\usepackage{etoolbox}
\usepackage{makecell}
\usepackage[ruled]{algorithm2e}
\DontPrintSemicolon
\usepackage{enumerate}
\usepackage{booktabs}
\usepackage{lastpage}

\newcommand{\fS}{\mathcal{S}}
\newcommand{\fA}{\mathcal{A}}

\newcommand{\fB}{\mathcal{B}}

\newcommand{\fW}{\mathcal{W}}
\newcommand{\fF}{\mathcal{F}}
\newcommand{\fO}{\mathcal{O}}
\newcommand{\fH}{\mathcal{H}}
\newcommand{\R}{\mathbb{R}}
\newcommand{\E}{\mathbb{E}}
\newcommand{\nsa}{{|\fS \times \fA|}}
\newcommand{\ns}{{|\fS|}}
\newcommand{\na}{{|\fA|}}

\newcommand{\bop}{\mathcal{T}}

\newcommand{\pbop}[2]{\Pi_{f_{n, #1, #2}}\bop_{#2}}
\newcommand{\indot}[2]{{\left<#1, #2\right>}}

\newcounter{assu_counter}
\numberwithin{assu_counter}{section}
\newtheorem{assumption}[assu_counter]{Assumption}

\jmlrheading{23}{2022}{1-\pageref{LastPage}}{8/21; Revised
4/22}{5/22}{21-0934}{Shangtong Zhang and Shimon Whiteson}
\ShortHeadings{Truncated Emphatic Temporal Difference Methods}{Zhang and Whiteson}

\firstpageno{1}

\begin{document}

\title{Truncated Emphatic Temporal Difference Methods for Prediction and Control}

\author{\name Shangtong Zhang \email shangtong.zhang@cs.ox.ac.uk \\
       \addr 
       Department of Computer Science \\
       University of Oxford\\
       Wolfson Building, Parks Rd, Oxford, OX1 3QD, UK \\
       \name Shimon Whiteson \email shimon.whiteson@cs.ox.ac.uk \\
       \addr
       Department of Computer Science \\
       University of Oxford \\
       Wolfson Building, Parks Rd, Oxford, OX1 3QD, UK
       }

\editor{Marc Bellemare}

\maketitle

\begin{abstract}
Emphatic Temporal Difference (TD) methods are a class of off-policy Reinforcement Learning (RL) methods involving the use of followon traces. 
Despite the theoretical success of emphatic TD methods in addressing the notorious deadly triad of off-policy RL,
there are still two open problems.
First, followon traces typically suffer from large variance,
making them hard to use in practice. 
Second, though \citet{yu2015convergence} confirms the asymptotic convergence of some emphatic TD methods for prediction problems,
there is still no finite sample analysis for any emphatic TD method for prediction, much less control.
In this paper, 
we address those two open problems simultaneously via using \emph{truncated followon traces} in emphatic TD methods.
Unlike the original followon traces, which depend on all previous history,
truncated followon traces depend on only finite history, reducing variance and enabling the finite sample analysis of our proposed emphatic TD methods for both prediction and control.
\end{abstract}


\section{Introduction}
\emph{Off-policy} learning, where an agent learns a policy of interest (target policy) while following a different policy (behavior policy), is arguably one of the most important techniques in Reinforcement Learning (RL, \citealt{sutton2018reinforcement}).
Off-policy learning can improve the sample efficiency \citep{lin1992self,sutton2011horde} and safety \citep{dulac2019challenges} of RL algorithms.
However, they can be unstable if combined with
\emph{function approximation} and \emph{bootstrapping},
two arguably indispensable ingredients for RL algorithms to work at scale.
This instability is known as  the notorious deadly triad (Chapter 11 of \citealt{sutton2018reinforcement}). 

Emphatic Temporal Difference methods are a class of off-policy Temporal Difference (TD, \citealt{sutton1988learning}) methods first proposed by \citet{sutton2016emphatic} to address the deadly triad.
Compared with gradient TD methods \citep{sutton2009convergent,sutton2009fast},
another class of off-policy TD methods that address the deadly triad,
emphatic TD (ETD) methods usually have better asymptotic performance guarantees \citep{kolter2011fixed,hallak2016generalized}.
The key idea of emphatic TD methods is the \emph{followon} trace, 
a recursively computed scalar depending on all previous history that reweights the naive off-policy TD (Chapter 11.1 of \citealt{sutton2018reinforcement}) updates,
first introduced in ETD($\lambda$) \citep{sutton2016emphatic}.
In addition to ETD($\lambda$), variants have been proposed such as ETD($\lambda, \beta$) \citep{hallak2016generalized}, which allows for additional bias-variance tradeoff, and NETD
\citep{jiang2021}, which copes with multi-step TD methods like VTrace \citep{espeholt2018impala}.
Emphatic TD methods have enjoyed both theoretical and empirical success.
For example,
\citet{yu2015convergence} confirms the asymptotic convergence of ETD($\lambda$) under general conditions;
\citet{jiang2021} demonstrate state-of-the-art performance of several NETD-based algorithms in certain Arcade Learning Environment \citep{bellemare13arcade} settings.
Nonetheless,
there are still two open problems.
\begin{enumerate}
  \item 
  The followon trace can have infinite variance,
  as demonstrated by \citet{sutton2016emphatic}.
  As a result,
  \citet{sutton2018reinforcement} report that
  though ETD($\lambda$) is proven to be convergent,
  in Baird's counterexample \citep{baird1995residual},
  a commonly used benchmark for testing off-policy RL algorithms,
  ``\emph{it is nigh impossible to get consistent results in computational experiments}'' (Chapter 11.9 of \citealt{sutton2018reinforcement}) for ETD($\lambda$).
  To lower the variance introduced by the followon trace,
  \citet{hallak2016generalized} introduce an additional hyperparameter $\beta$ for bias-variance trade-off in computing the followon trace,
  resulting in ETD($\lambda, \beta$).
  When $\beta$ is sufficiently small,
  \citet{hallak2016generalized} prove that the variance of the followon trace is bounded.
  \citet{hallak2016generalized},
  however,
  also require $\beta$ to be sufficiently large such that the expected update of ETD($\lambda, \beta$) is contractive,
  which plays a key role in bounding the performance of the fixed point of ETD($\lambda, \beta$).
  Unfortunately,
  there is no guarantee that 
  such a $\beta$ (i.e., a $\beta$ that is both sufficiently small and sufficiently large) always exists.
  Later on,
  \citet{zhang2019provably} propose to \emph{learn} the expectation of the followon trace directly by employing a second function approximator
  and use the \emph{learned} followon trace to reweight the naive off-policy TD updates.
  However,
  little can be said about the quality of the learned followon trace.
  It, 
  therefore,
  remains an open problem to design a theoretically grounded method to reduce the variance introduced by the followon trace. 
  \item Twenty years after the seminal work \citet{tsitsiklis1997analysis} confirming the asymptotic convergence of TD$(\lambda)$,
finite sample analysis of TD methods were obtained for both prediction \citep{dalal2018finite,lakshminarayanan2018linear,bhandari2018finite,srikant2019finite} and control \citep{zou2019finite}.
Though \citet{yu2015convergence} confirms the asymptotic convergence of ETD$(\lambda)$,
we still do not have finite sample analysis for any emphatic TD method even for prediction problems, much less control.
\end{enumerate}

In this paper,
we address these two problems simultaneously by using \emph{truncated} followon traces instead of the original followon trace in Section~\ref{sec prediction} for prediction problems 
and extend the results to control problems in Sections~\ref{sec control eavi} and~\ref{sec control sarsa}.
Truncated traces are introduced by \citet{DBLP:journals/siamco/Yu12,yu2015convergence,yu2017convergence} 
as an intermediate mathematical tool in proofs
to understand the asymptotic behavior of some least-square TD methods (e.g., off-policy LSTD($\lambda$) in \citealt{DBLP:journals/siamco/Yu12}, emphatic LSTD($\lambda$) in \citealt{yu2015convergence}) and gradient TD methods (e.g., GTD($\lambda$) in \citealt{sutton2009fast}) for prediction.
In this paper,
we instead use truncated followon traces \emph{algorithmically} as a tool for variance reduction for both prediction and control.
Whereas the original followon trace depends on all previous history,
the truncated followon trace depends on only \emph{finite} history.
Consequently,
the variance of truncated followon traces is immediately bounded.
We refer to emphatic TD methods that involve this truncated followon traces as \emph{truncated emphatic TD methods}. 
Moreover,
we show that under certain conditions on their length,
truncated followon traces maintain all the desirable properties of the original followon trace, enabling us to analyse truncated emphatic TD methods both asymptotically and non-asymptotically,
for both prediction and control.

In particular, this paper makes the following contributions.
First,
we propose the Truncated Emphatic TD algorithm for off-policy prediction
and provide both asymptotic and nonasymptotic convergence analysis.
This is the first finite sample analysis for emphatic TD methods.
Second,
we propose the Truncated Emphatic Expected SARSA algorithm for off-policy control and provide both asymptotic and nonasymptotic analysis.
This is the first emphatic TD algorithm for off-policy control.
Third,
we empirically study truncated emphatic TD methods in both synthetic Markov Decision Processes (MDPs) and nonsynthetic control problems, 
confirming their efficacy in practice. 


\section{Background}
In this paper,
all vectors are column.
A matrix $M$ (not necessarily symmetric) is said to be positive definite (p.d.) if there exists a constant $\lambda > 0$ such that $x^\top M x \geq  \lambda x^\top x$ holds for any $x$.
It is well known that $M$ is p.d.\ if and only if $M+M^\top$ is p.d.
$M$ is negative definite (n.d.) if and only if $-M$ is p.d. 
For a vector $x$ and a p.d.\ matrix $M$,
we use $\norm{x}_M \doteq \sqrt{x^\top M x}$
to denote the vector norm induced by $M$.
We also use $\norm{\cdot}_M$ to denote the corresponding induced matrix norm.
We use $\norm{\cdot}$ as shorthand for $\norm{\cdot}_I$ where $I$ is the identity matrix,
i.e.,
$\norm{\cdot}$ is the standard $\ell_2$-norm.
We use $diag(x)$ to denote a diagonal matrix whose diagonal entry is $x$ and write
$\norm{\cdot}_{x}$ as shorthand for $\norm{\cdot}_{diag(x)}$ when $diag(x)$ is p.d.
We use $\norm{\cdot}_\infty$ and $\norm{\cdot}_1$ to denote the standard infinity norm and $\ell_1$-norm respectively.
We use $\indot{\cdot}{\cdot}$ to denote the inner product in Euclidean spaces,
i.e., $\indot{x}{y} \doteq x^\top y$.
We use functions and vectors interchangeably when it does not confuse,
e.g., if $f$ is a function from $\fS$ to $\R$,
we also use $f$ to denote a vector in $\R^\ns$,
whose $s$-th element is $f(s)$.

We consider an infinite horizon MDP with a finite state space $\fS$,
a finite action space $\fA$,
a reward function $r: \fS \times \fA \to \R$,
a transition kernel $p: \fS \times \fS \times \fA \to [0, 1]$,
an initial state distribution $p_0: \fS \to [0, 1]$,
and a discount factor $\gamma \in [0, 1)$.
At time step 0,
an initial state $S_0$ is sampled according to $p_0$.
At time step $t$, 
an agent at a state $S_t$ takes an action $A_t$ according to $\pi(\cdot | S_t)$,
where $\pi: \fA \times \fS \to [0, 1]$ is the policy being followed by the agent.
The agent then receives a reward $R_{t+1} \doteq r(S_t, A_t)$ and proceeds to a successor state $S_{t+1}$ sampled from $p(\cdot |S_t, A_t)$.

The return at time step $t$ is defined as
\begin{align}
G_t \doteq \sum_{i=1}^\infty \gamma^{i - 1}R_{t+i},
\end{align}
which allows us to define the state value and action value functions respectively as
\begin{align}
v_\pi(s) &\doteq \E \left[G_t | S_t=s, \pi, p \right], \\
q_\pi(s, a) &\doteq \E \left[G_t | S_t = s, A_t = a, \pi, p \right].
\end{align}
The value function $v_\pi$ is the unique fixed point of the Bellman operator $\bop_\pi$:
\begin{align}
\bop_\pi v \doteq r_\pi + \gamma P_\pi v,
\end{align}
where $r_\pi \in \R^\ns$ is the reward vector induced by the policy $\pi$, i.e., $r_\pi(s) \doteq \sum_a \pi(a|s) r(s, a)$.
Prediction and control are two fundamental problems in RL.

\subsection{Prediction}
The goal of prediction is to estimate the value function of a given policy $\pi$,
perhaps with the help of parameterized function approximation.
In this paper,
we consider linear function approximation and assume access to a feature function $x: \fS \to \R^K$,
which maps a state into a $K$-dimensional numerical feature.
We then use $x(s)^\top w$ as our estimate for $v_\pi(s)$,
where $w \in \R^K$ is the parameter vector to be learned.
Arguably,
one of the most important methods for prediction is TD,
which updates $w$ iteratively as
\begin{align}
\label{eq on-policy td}
w_{t+1} &\doteq w_t + \alpha_t (R_{t+1} + \gamma x_{t+1}^\top w_t - x_t^\top w_t ) x_t \\
&= w_t + \alpha_t \underbrace{x_t(\gamma x_{t+1}^\top - x_t^\top)}_{M_t}w_t + \alpha_t R_{t+1} x_t,
\end{align}
where $\qty{\alpha_t}$ is a sequence of learning rates and $x_t$ is shorthand for $x(S_t)$.
The expectation of the update matrix $M_t$ w.r.t.\ $d_\pi$, the invariant state distribution of the chain induced by $\pi$, 
is 
\begin{align}
M \doteq \E_{S_t \sim d_\pi, A_t \sim \pi(\cdot|S_t), S_{t+1} \sim p(\cdot | S_t, A_t)}[M_t] = X^\top D_\pi (\gamma P_\pi - I) X,
\end{align}
where $X \in \R^{\ns \times K}$ is the feature matrix whose $s$-th row is $x(s)^\top$, 
$D_\pi \doteq diag(d_\pi)$, 
and $P_\pi \in \R^{\ns \times \ns}$ is the state transition matrix under the policy $\pi$, i.e.,
\begin{align}
P_\pi(s, s') \doteq \sum_a \pi(a|s)p(s'|s, a).
\end{align}
\citet{tsitsiklis1997analysis} prove that $M$ is n.d.\ under mild conditions.
Consequently,
standard Ordinary Differential Equation (ODE) based convergence results 
(e.g., Theorem 2 of \citealt{tsitsiklis1997analysis}, Proposition 4.8\footnote{For completeness, we include this proposition as Theorem \ref{thm ndp} in Section \ref{sec ndp}.} of \citealt{bertsekas1996neuro})
can be used to show that
the iterates $\qty{w_t}$ generated by \eqref{eq on-policy td} converge almost surely (a.s.).

So far we have focused on the on-policy setting,
where the policy to be evaluated is the same as the policy used for action selection during interaction with the environment. 
In the off-policy setting,
those two policies can, however, be different,
allowing extra flexibility.
We use $\pi$ to denote the policy to be evaluated (target policy)
and $\mu$ to denote the policy used for action selection (behavior policy).

Since the action selection is performed according to $\mu$ instead of $\pi$
(i.e., $A_t \sim \mu(\cdot | S_t), R_{t+1} \doteq r(S_t, A_t), S_{t+1} \sim p(\cdot | S_t, A_t)$),
we can reweight the update made in \eqref{eq on-policy td} by the importance sampling ratio $\rho_t \doteq \frac{\pi(A_t | S_t)}{\mu(A_t | S_t)}$,
yielding the following off-policy TD updates:
\begin{align}
\label{eq off-policy td}
w_{t+1} &\doteq w_t + \alpha_t \rho_t (R_{t+1} + \gamma x_{t+1}^\top w_t - x_t^\top w_t ) x_t \\
&= w_t + \alpha_t \underbrace{\rho_t x_t(\gamma x_{t+1}^\top - x_t^\top)}_{M_t}w_t + \alpha_t \rho_t R_{t+1} x_t.
\end{align}
The expectation of the update matrix $M_t$ in \eqref{eq off-policy td} w.r.t.\ $d_\mu$,
the invariant state distribution of the chain induced by $\mu$,
is then
\begin{align}
M \doteq \E_{S_t \sim d_\mu, A_t \sim \mu(\cdot | S_t), p(\cdot | S_t, A_t)}[M_t] = X^\top D_\mu (\gamma P_\pi - I) X,
\end{align}
where $D_\mu \doteq diag(d_\mu)$.
Unfortunately,
this $M$  is not guaranteed to be n.d.\ and 
the possible divergence of \eqref{eq off-policy td} is well documented in Baird's counterexample \citep{baird1995residual}.

\citet{sutton2016emphatic} propose ETD($\lambda$) to address this divergence issue.
In its simplest form with $\lambda=0$,
ETD(0)
further reweights the update in \eqref{eq off-policy td} by the \emph{followon trace} $F_t$:
\begin{align}
\label{eq etd}
F_t &\doteq i(S_t) + \gamma \rho_{t-1} F_{t-1}, \\
w_{t+1} &\doteq w_t + \alpha_t \rho_t F_t (R_{t+1} + \gamma x_{t+1}^\top w_t - x_t^\top w_t ) x_t \\
&= w_t + \alpha_t \underbrace{\rho_t F_t x_t(\gamma x_{t+1}^\top - x_t^\top)}_{M_t}w_t + \alpha_t \rho_t F_t R_{t+1} x_t.
\end{align}
where $i: \fS \to (0, +\infty)$ is the \emph{interest} function representing the user's preference for different states.
The motivation for introducing $F_t$ is to ensure that 
the corresponding limiting update $M \doteq \lim_{t\to \infty} \E[M_t]$, assuming the limit exists for now,
is n.d.\ such that standard ODE-based convergent results (e.g., Theorem \ref{thm ndp}) can be used to show convergence.
In fact, \citet{sutton2016emphatic} show
\begin{align}
M = X^\top D_f (\gamma P_\pi - I) X,
\end{align}
where $D_f \doteq diag(f)$ with 
\begin{align}
  \label{eq f}
  f \doteq (I - \gamma P_\pi^\top)^{-1}D_\mu i.
\end{align}
\citet{sutton2016emphatic} prove that this $M$ is n.d. 
and the convergence of ETD($\lambda$) is later on established by \citet{yu2015convergence}.
It is worth mentioning that one important step in computing this $M$ is to show
\begin{align}
  \label{eq rich limit}
  \lim_{t\to\infty} \E\left[F_t | S_t = s\right] = d_{\mu}(s)^{-1} f(s).
\end{align} 

\subsection{Control}
The goal for control is to find an optimal policy $\pi^*$ such that $v_{\pi^*}(s) \geq v_\pi(s)$ holds for any $\pi$ and $s$.
Though there can be more than one optimal policy,
all of them share the same optimal value function,
which is referred to as $v_*$.
One classical approach for finding $v_*$ is \emph{value iteration} (see, e.g., \citealt{puterman2014markov}).
Given an arbitrary vector $v \in \R^\ns$,
value iteration updates $v$ iteratively as
\begin{align}
v_{k+1} \doteq \bop_{\pi_{v_k}} v_k,
\end{align}
where we use $\pi_{v_k}$ to denote the greedy policy w.r.t.\ $v_k$.
Let 
\begin{align}
q_{v_k}(s, a) \doteq r(s, a) + \gamma \sum_{s'}p(s'|s, a)v_{\pi_k}(s'),
\end{align}
then at a state $s$,
$\pi_{v_k}$ selects an action action greedily w.r.t.\ $q_{v_k}(s, \cdot)$.
It is well known (see, e.g., \citealt{puterman2014markov}) that $\lim_{k\to\infty} v_k = v_*$.

With function approximation,
we have $v_k \doteq X w_k$,
where $w_k$ is the parameter at the $k$-th iteration.
When doing value iteration,
however, $\bop_{\pi_{v_k}} v_k$ may not lie in the column space of $X$.
Consequently, 
an additional projection operator is used to project $\bop_{\pi_{v_k}} v_k$ back to the column space of $X$,
yielding \emph{approximate value iteration} \citep{de2000existence},
which updates $v_k$ as
\begin{align}
\label{eq on-policy avi}
v_{k+1} \doteq \fH(v_k) \doteq \Pi_{d_{\pi_{v_k}}} \bop_{\pi_{v_k}} v_k,
\end{align}
where
\begin{align}
\Pi_{d_{\pi_{v_k}}}y \doteq X \arg\min_w \norm{Xw - y}^2_{d_{\pi_{v_k}}}
\end{align}
is the projection operator to the column space of $X$ w.r.t.\ to the norm induced by the invariant state distribution $d_{\pi_{v_k}}$ under the current policy $\pi_{v_k}$.
Unfortunately,
if $\pi_{v}$ is greedy w.r.t.\ $v$,
\citet{de2000existence} show that the approximate value iteration operator $\fH$ does not necessarily have a fixed point.
However,
if the policy $\pi_{v}$ is continuous in $v$, e.g., $\pi_v$ is a softmax policy such that
\begin{align}
  \label{eq softmax policy v}
\pi_v(a | s) \doteq \frac{\exp \left(r(s, a) + \gamma \sum_{s'} p(s'|s, a)v(s')\right)}{\sum_{s_0, a_0} \exp \left( r(s_0, a_0) + \gamma \sum_{s_1} p(s_1 | s_0, a_0) v(s_1)\right)},
\end{align}
\citet{de2000existence} show that there exists at least one $w_*$ such that 
\begin{align}
Xw_* = \fH(Xw_*). 
\end{align}

In RL, one way to implement approximate value iteration incrementally is SARSA \citep{rummery1994line},
which updates $w$ iteratively as
\begin{align}
\label{eq on-policy sarsa}
w_{t+1} \doteq w_t + \alpha_t (R_{t+1} + \gamma x(S_{t+1}, A_{t+1})^\top w_t - x(S_t, A_t)^\top w_t) x(S_t, A_t),
\end{align}
where we have overloaded $x$ as a function from $\fS \times \fA$ to $\R^K$ to denote the state-action feature. 
We then use $x(s, a)^\top w$ as our estimate for the action value function.
In the above SARSA update, actions are selected such that $A_t \sim \pi_{w_{t-1}}(\cdot | S_t)$,
where $\pi_{w_{t-1}}$ denotes a policy depending on the action value estimate $x(s, a)^\top w_{t-1}$, e.g., a softmax policy
\begin{align}
\pi_{w_{t-1}}(a | s) \doteq \frac{\exp(x(s, a)^\top w_{t-1})}{\sum_{s_0, a_0}\exp\left(x(s_0, a_0)^\top w_{t-1}\right)}.
\end{align}
\citet{melo2008analysis} and \citet{zou2019finite} provide asymptotic convergence analysis and finite sample analysis of SARSA respectively,
under mild conditions.

\section{Open Problems in Emphatic TD Methods}
\label{sec open problems}

In this section,
we discuss in detail two open problems of emphatic TD methods.
First, though ETD($\lambda$) is proven to be convergent,
the large variance of $F_t$ makes it hard to use directly.
There are several attempts to address this variance.
\citet{hallak2016generalized} propose to replace $F_t$ with $F_{t, \beta}$,
which is computed recursively as
\begin{align}
  \label{eq hallak trace}
  F_{t, \beta} \doteq i(S_t) + \beta \rho_{t-1} F_{t-1, \beta},
\end{align}
where $\beta \in (0, 1)$ is an additional hyperparameter.
The resulting ETD($\lambda, \beta$) then updates $\qty{w_t}$ iteratively as
\begin{align}
  \label{eq hallak etd}
w_{t+1} \doteq w_t + \alpha_t F_{t, \beta} \rho_t (R_{t+1} + \gamma x_{t+1}^\top w_t - x_t^\top w_t ) x_t.
\end{align}
Theorem 1 of \citet{hallak2016generalized} states that there exists a problem-dependent constant $\beta_\text{upper}$ such that $\beta \leq \beta_\text{upper}$ implies that the variance of $F_{t, \beta}$ is bounded.
Further, Proposition 1 of \citet{hallak2016generalized} states that there exists a problem-dependent constant $\beta_\text{lower}$ such that $\beta \geq \beta_\text{lower}$ implies that the expected update corresponding to~\eqref{eq hallak etd} is contractive,
which plays a key role in bounding the performance of the fixed point of~\eqref{eq hallak etd},
assuming~\eqref{eq hallak etd} converges.
Unfortunately,
there is no guarantee that $\beta_\text{lower} \leq \beta_\text{upper}$ always holds,
i.e,
the desired $\beta$ does not always exist.
\citet{zhang2019provably} instead propose to use a second function approximator to learn the expectation of the followon trace directly.
For example, let $x(s)^\top \theta$ be the estimate for the expectation of the followon trace;
\citet{zhang2019provably} replace $F_t$ in the ETD update \eqref{eq etd} by $x(S_t)^\top \theta$ and update $w$ iteratively as 
\begin{align}
w_{t+1} \doteq w_t + \alpha_t (x_t^\top \theta) \rho_t (R_{t+1} + \gamma x_{t+1}^\top w_t - x_t^\top w_t ) x_t.
\end{align}
\citet{zhang2019provably} use Gradient Emphasis Learning (GEM) to learn $\theta$.
GEM shares the same idea as gradient TD methods.
Though \citet{zhang2019provably} confirm the convergence of GEM,
like any off-policy gradient TD method, 
little can be said about the quality of its solution,
i.e.,
the error $\left|x(s)^\top \theta - \lim_{t\to\infty}\E\left[F_t | S_t = s\right] \right|$ can be arbitrarily large as long as the feature matrix $X$ cannot perfectly represent the expected followon trace \citep{kolter2011fixed}. 
\citet{jiang2021} propose to 
clip the importance sampling ratio $\rho_t$ when computing the followon trace $F_t$ to reduce variance.
However,
nothing can be said about the convergence of the resulting algorithm due to the bias introduced by clipping.
Despite these attempts,
it remains an open problem to reduce the variance of emphatic TD methods introduced by the followon trace in a theoretically grounded way.

Second, 
the analysis of ETD($\lambda$) in \citet{yu2015convergence} is only asymptotic.
So far no finite sample analysis is available for any emphatic TD method.
The finite sample analysis of TD($\lambda$) in \citet{bhandari2018finite} cannot be easily extended to ETD($\lambda$).
Key to the finite sample analysis of TD($\lambda$) is Lemma 17 of \citet{bhandari2018finite},
which establishes the boundedness of the eligibility trace used in on-policy TD($\lambda$).
This immediately implies the boundedness of the second moments of the eligibility trace,
which is a key bound for error terms.
However, 
such boundedness cannot be expected for the followon trace in ETD($\lambda$) since \citet{sutton2016emphatic} already show that the variance of the followon trace can be unbounded.
It thus remains an open problem to provide a finite sample analysis 
for emphatic TD methods for prediction problems.
For control problems, 
the policy usually changes every time step (cf. the SARSA algorithm \eqref{eq on-policy sarsa}).
Consequently, the induced chain is not stationary.
The asymptotic convergence analysis for ETD($\lambda$) in \citet{yu2015convergence},
however,
relies on 
the strong law of large numbers on \emph{stationary} chains. 
It thus remains unclear whether the asymptotic convergence analysis of \citet{yu2015convergence} can be extended to the control setting.
Hence,
providing a finite sample analysis for emphatic TD methods for control problems is even more challenging.

\section{Prediction: Truncated Emphatic TD}
\label{sec prediction}
In this paper,
we address the two open problems in Section \ref{sec open problems} simultaneously by replacing the original followon trace in emphatic TD methods with truncated followon traces. 
Assuming $F_{-1} \doteq 0$,
the original followon trace $F_t$ in \eqref{eq etd} can be expanded as
\begin{align}
  \label{eq etd trace expanded}
F_t &= i_t + \gamma \rho_{t-1}F_{t-1} \\
&=i_t + \gamma \rho_{t-1} i_{t-1} + \gamma^2 \rho_{t-1} \rho_{t-2} F_{t-2} \\
&=i_t + \gamma \rho_{t-1} i_{t-1} + \gamma^2 \rho_{t-1} \rho_{t-2} i_{t-2} + \gamma^3 \rho_{t-1}\rho_{t-2}\rho_{t-3} F_{t-3}\\
&= \dots \\
&= \sum_{j=0}^t \gamma^j \rho_{t-j:t-1} i_{t-j},
\end{align} 
where $i_t$ is shorthand for $i(S_t)$ and 
\begin{align}
\rho_{j:k} \doteq \begin{cases}
\rho_j \rho_{j+1} \cdots \rho_{k} & j \leq k\\
1 & j > k
\end{cases}
\end{align}
is shorthand for the product of importance sampling ratios.
Clearly, $F_t$ depends on all the history from time steps 0 to $t$.
The idea of truncated followon traces,
introduced in \citet{DBLP:journals/siamco/Yu12,yu2015convergence,yu2017convergence}, 
is,
for a fixed length $n$,
to compute the followon trace $F_t$ as if $F_{t-n-1}$ was 0. 
More specifically,
let $F_{t,n}$ be the truncated followon traces of length $n$;
we have
\begin{align}
\label{eq truncated trace}
F_{t, n} \doteq
\begin{cases}
\sum_{j=0}^n \gamma^j \rho_{t-j:t-1} i_{t-j} & t \geq n \\
F_t & t < n
\end{cases}.
\end{align}
For example, if $n=2$,
we then compute $F_{t, 2}$ for any $t$ as
\begin{align}
  F_{t, 2} = i_t + \gamma \rho_{t-1} i_{t-1} + \gamma^2 \rho_{t-1} \rho_{t-2} i_{t-2}.
\end{align}
In this paper, 
we propose to replace $F_t$ with $F_{t, n}$ in emphatic TD methods. 
Apparently, for a fixed $n$,
the variance of $F_{t, n}$ is guaranteed to be bounded.
By contrast,
\citet{sutton2016emphatic} show that the variance of $F_t$ can be infinite.
We
refer to emphatic TD methods using the truncated traces as \emph{truncated emphatic TD methods}. 
For example,
Truncated Emphatic TD is given in Algorithm \ref{alg etd},
where we adopt the convention that $i_t = \rho_t = 0$ for any $t < 0$.
\begin{algorithm}
$S_0 \sim p_0(\cdot)$ \;
$t \gets 0$ \;
\While{True}{
  Sample $A_t \sim \mu(\cdot | S_t)$ \;
  Execute $A_t$, get $R_{t+1}, S_{t+1}$ \;
  $\rho_t \gets \frac{\pi(A_t | S_t)}{\mu(A_t | S_t)}$ \;
  $F_{t, n} \gets 0$ \;
  \For{$k = 0, \dots, n$}{
    $F_{t, n} \gets i_{t-n+k} + \gamma \rho_{t-n+k-1} F_{t, n}$
  }
  $w_{t+1} \gets w_t + \alpha_t F_{t, n} \rho_t (R_{t+1} + \gamma x_{t+1}^\top w_t - x_t^\top w_t ) x_t $ \;
  $t \gets t + 1$ \;
}
\caption{\label{alg etd}Truncated Emphatic TD}
\end{algorithm}

To compute $F_{t, n}$, 
one needs to store 2$n$ extra scalars: $\rho_{t-1}, \dots, \rho_{t-n}, i_{t-1}, \dots, i_{t-n}$.
Such memory overhead is inevitable even for naive on-policy multi-step TD methods (Chapter 7.1 of \citealt{sutton2018reinforcement}).
The computation of $F_{t,n}$ can indeed be done incrementally at the cost of maintaining one more extra scalar:
\begin{align}
\Delta_{t} &\doteq \frac{\rho_{t}i_{t}}{\rho_{t-n-1}i_{t-n-1}} \Delta_{t-1}, \\
F_{t, n} &\doteq i_{t} + \gamma \rho_{t-1} F_{t-1, n} - \Delta_{t}.
\end{align}
Overall,
we argue that compared with ETD(0) in \citet{sutton2016emphatic},
the additional memory and computational cost of Truncated Emphatic TD is negligible.
We now analyze Truncated Emphatic TD with
the following assumptions:
\begin{assumption}
\label{assu ergodic}
The Markov chain induced by the behavior policy $\mu$ is ergodic.
\end{assumption}
\begin{assumption}
\label{assu coverage}
$\mu(a|s) > 0$ holds for any $(s, a)$.
\end{assumption}
\begin{assumption}
\label{assu full rank}
The feature matrix $X$ has full column rank.
\end{assumption}
Assumptions \ref{assu ergodic} and \ref{assu full rank} are standard in the off-policy RL literature.
Assumption \ref{assu coverage} can indeed be weakened to the canonical coverage assumption $\pi(a | s) > 0 \implies \mu(a | s) > 0$ \citep{yu2015convergence}. 
Then for a state $s$, 
we can simply consider only actions $a$ such that $\mu(a | s) > 0$,
i.e.,
different states have different action spaces.
All the analysis presented in this section still hold.
We use Assumption \ref{assu coverage} mainly to simplify presentation.

When analyzing the original ETD,
we have to consider the chain $\qty{(F_t, S_t, A_t)}$ evolving in the space $\R \times \fS \times \fA$ 
(see, e.g., \citealt{yu2015convergence}).
The space $\R$ is not even countable,
making it hard to analyze the chain $\qty{(F_t, S_t, A_t)}$ even with Assumption \ref{assu ergodic}.
With the truncated followon trace $F_{t, n}$,
we only need to consider the chain $\qty{(S_{t-n}, A_{t-n}, \dots, S_t, A_t)}$
which evolves in a \emph{finite} space $(\fS \times \fA)^n$.
The ergodicity of this chain follows immediately from Assumption \ref{assu ergodic}.
Once the ergodicity is established,
we can analyze the limiting update matrix under the corresponding invariant distribution. 

The additional hyperparameter $n$ in \eqref{eq truncated trace} defines a hard truncation.
By contrast,
the additional hyperparameter $\beta$ in \eqref{eq hallak trace}
defines a soft truncation.
As discussed in Section~\ref{sec open problems},
a desired $\beta$ does not always exist since we require $\beta$ to be both sufficiently large and sufficiently small.
By contrast,
we will show soon that a desired $n$ always exists because we only require $n$ to be sufficiently large.
Further,
to analyze ETD($\lambda, \beta$) with the soft truncation,
we still need to work on the chain $\qty{(F_{t, \beta}, S_t, A_t)}$,
whose behavior is hard to analyze.
Consequently, 
though the asymptotic convergence of ETD$(\lambda, \beta)$ in prediction may be established similarly to \citet{yu2015convergence} for certain $\beta$,
so far no finite sample analysis is available for ETD($\lambda, \beta$) in prediction, much less control.
Nevertheless,
we believe the soft truncation and the hard truncation are two different directions for variance reduction. 
The soft truncation is analogous to computing the return $G_t$ with a discount factor different from $\gamma$
(see, e.g., \citealt{romoff2019separating});
the hard truncation is analogous to computing the return $G_t$ with a fixed horizon
(see, e.g., \citealt{de2019fixed}).
It is straightforward to combine the two techniques together.
For example,
we can consider $F_{t, \beta, n}$ defined as
\begin{align}
F_{t, \beta, n} \doteq
\begin{cases}
\sum_{j=0}^n \beta^j \rho_{t-j:t-1} i_{t-j} & t \geq n \\
F_{t, \beta} & t < n
\end{cases}.
\end{align}
This combination,
however,
deviates from the main purpose of this paper
and is saved for future work.

We now study the truncated trace $F_{t, n}$.
Similar to~\eqref{eq rich limit},
we study 
the limit of the conditional expectation of the truncated followon trace and define
\begin{align}
\label{eq my limit}
m_n(s) \doteq \lim_{t \to \infty} \E\left[F_{t,n} | S_t = s\right].
\end{align}
When $n=\infty$,
this $m_{\infty}$ is referred to as \emph{emphasis} in \citet{zhang2019provably}.
We,
therefore,
refer to $m_{n}$ as \emph{truncated emphasis} for a finite $n$.
\begin{lemma}
\label{lem emphasis expression}
Let Assumptions \ref{assu ergodic} and \ref{assu coverage} hold.
Then 
\begin{align}
\label{eq def mn}
m_n &= \sum_{j=0}^n \gamma^j D_\mu^{-1} (P_\pi^\top)^j D_\mu i, \\
\label{eq def m}
m &\doteq \lim_{n\to\infty} m_n = D_\mu^{-1}(I - \gamma P_\pi^\top)^{-1} D_\mu i.
\end{align}
\end{lemma}
The proof of Lemma \ref{lem emphasis expression} is provided in Section \ref{sec proof of empahsis expression lemma}.
By definition, the weighting vector $f$ in~\eqref{eq f} involved in $M$ of the ETD update \eqref{eq etd} satisfies $f = D_\mu m$.
Similarly, we define $f_n \doteq D_\mu m_n $.
\begin{lemma}
\label{lem emphasis bound}
Let Assumptions \ref{assu ergodic} and \ref{assu coverage} hold. Then
\begin{align}
\norm{m_n - m}_1 &\leq \gamma^{n+1} \frac{d_{\mu, max}}{d_{\mu, min}} \norm{m}_1, \\
\norm{f_n - f}_\infty & \leq \gamma^{n+1} \frac{d_{\mu, max}^2}{d_{\mu, min}} \norm{m}_1,
\end{align}
where $d_{\mu, max} \doteq \max_s d_\mu(s)$ and $d_{\mu, min} \doteq \min_s d_\mu(s)$.
\end{lemma}
The proof of Lemma \ref{lem emphasis bound} is provided in Section \ref{sec proof of emphasis bound lemma}.
The $M$ matrix of the ETD(0) update \eqref{eq etd} is $X^\top D_f (\gamma P_\pi - I)X$.
Similarly,
it can be shown that the $M$ matrix of Truncated Emphatic TD (Algorithm \ref{alg etd}) is $X^\top D_{f_n} (\gamma P_\pi - I)X$.
Lemma \ref{lem emphasis bound} asserts that $f_n$ approaches $f$ geometrically fast.
Consequently,
we can expect $X^\top D_{f_n}(\gamma P_\pi - I)X$ to be n.d.\ if $n$ is not too small. 
\begin{lemma}
\label{lem nd}
Under Assumptions \ref{assu ergodic}, \ref{assu coverage}, and \ref{assu full rank},  if 
\begin{align}
\label{eq condition of n nd}
\gamma^{n+1} < \frac{\lambda_{min} d_{\mu, min}}{d_{\mu, max}^2 \norm{\gamma P_\pi - I} \norm{m}_{1}},
\end{align}
then $X^\top D_{f_n} (\gamma P_\pi - I)X$ is n.d.,
where $\lambda_{min}$ is the minimum eigenvalue of
\begin{align}
\frac{1}{2} \left(D_f(I - \gamma P_\pi) + (I - \gamma P_\pi^\top) D_f \right).
\end{align}
\end{lemma}
\citet{sutton2016emphatic} prove that $\lambda_{min} > 0$.
The proof of Lemma \ref{lem nd} is provided in Section~\ref{sec proof of nd}.
Since the LHS of \eqref{eq condition of n nd} diminishes geometrically as $n$ increases,
we argue that in practice we do not need a very large $n$.
Recall that the motivation of using the followon trace $F_t$ is to ensure the limiting update matrix to be n.d.
Lemma \ref{lem nd} shows that to ensure this negative definiteness,
we do not need to use all history to compute $F_t$.
\emph{Earlier steps contribute little to this negative definiteness due to discounting but introduce large variance due to the products of importance sampling ratios.}
As suggested by~\eqref{eq condition of n nd},
the desired value of $n$ depends on the magnitude of the emphasis $m$,
which is determined together by the behavior policy $\mu$,
the target policy $\pi$, the structure of the MDP,
and the magnitude of the interest $i$.
In general,
when the magnitude of the emphasis increases,
the desired truncation length also increases.
In practice,
we propose to treat the truncation length $n$ as an additional hyperparameter,
as estimating the desired $n$ without access to the transition kernel $p$ can be very challenging,
which we leave for future work.

We can now show the asymptotic convergence of Truncated Emphatic TD using the standard ODE-based approach.
\begin{assumption}
  \label{assu lr}
  The learning rates $\qty{\alpha_t}$ are positive, nonincreasing, and satisfy
  \begin{align}
    \sum_t \alpha_t = \infty, \sum_t \alpha_t^2 < \infty.
  \end{align}
\end{assumption}
\begin{theorem}
\label{thm etd convergence}
Let the assumptions and conditions of Lemma \ref{lem nd} hold.
Let Assumption \ref{assu lr} hold.
Then the iterates $\qty{w_t}$ generated by Truncated Empathic TD (Algorithm \ref{alg etd}) satisfy
\begin{align}
\lim_{t\to\infty} w_t &= w_{*,n} \qq{a.s., where} \\
w_{*, n} &\doteq -A_n^{-1} b_n, \, A_n \doteq X^\top D_{f_n} (\gamma P_\pi - I)X, \, b_n \doteq X^\top D_{f_n} r_\pi.
\end{align}
\end{theorem}
The proof of Theorem \ref{thm etd convergence} is provided in Section \ref{sec proof of etd convergence},
which,
after the negative definiteness of $A_n$ is established with Lemma \ref{lem nd},
follows the same routine as the convergence proof of on-policy TD($\lambda$) in
Proposition 6.4 of \citet{bertsekas1996neuro}.

We now give a finite sample analysis of Projected Truncated Emphatic TD (Algorithm \ref{alg petd}).
Algorithm \ref{alg petd} is different from Algorithm \ref{alg etd} in that it adopts an additional projection $\Pi_R$ when updating the weight $w_t$. 
Here $\Pi_R$ denotes the projection onto the ball of a radius $R$ centered at the origin w.r.t.\ $\ell_2$ norm.
Introducing such a projection is common practice in finite sample analysis of TD methods \citep{bhandari2018finite,zou2019finite}.
This projection is mainly used to control the errors introduced by Markovian samples.
If i.i.d.\ samples are used instead,
such projection can indeed be eliminated \citep{bhandari2018finite,dalal2018finite}.
\begin{algorithm}
$S_0 \sim p_0(\cdot)$ \;
$t \gets 0$ \;
\While{True}{
  Sample $A_t \sim \mu(\cdot | S_t)$ \;
  Execute $A_t$, get $R_{t+1}, S_{t+1}$ \;
  $\rho_t \gets \frac{\pi(A_t | S_t)}{\mu(A_t | S_t)}$ \;
  $F_{t, n} \gets 0$ \;
  \For{$k = 0, \dots, n$}{
    $F_{t, n} \gets i_{t-n+k} + \gamma \rho_{t-n+k-1} F_{t, n}$
  }
  $w_{t+1} \gets \Pi_R \left(w_t + \alpha_t F_{t, n} \rho_t (R_{t+1} + \gamma x_{t+1}^\top w_t - x_t^\top w_t ) x_t \right)$ \;
  $t \gets t + 1$ \;
}
\caption{\label{alg petd} Projected Truncated Emphatic TD}
\end{algorithm}

\begin{theorem}
\label{thm finite sample}
Let the assumptions and conditions of Lemma \ref{lem nd} hold.
Let $R \geq \norm{w_{*, n}}$.
With proper learning rates $\qty{\alpha_t}$,
for sufficiently large $t$,
\begin{align}
\E\left[\norm{w_t - w_{*, n}}^2\right] = \fO\left(\frac{\ln^3 t}{t}\right).
\end{align}
\end{theorem}
The proof of Theorem \ref{thm finite sample} is omitted to avoid verbatim repetition since it is just a special case of a more general result in the control setting (Theorem \ref{thm finite sample sa}).
The conditions on learning rates and the constants hidden by $\fO\left(\cdot\right)$ are also similar to those of Theorem \ref{thm finite sample sa}.
We now analyze the performance of $w_{*, n}$.
\begin{lemma}
\label{lem contraction}
Let $\kappa \doteq \min_s \frac{d_\mu(s)i(s)}{f(s)}$.
Let Assumptions \ref{assu ergodic}, \ref{assu coverage}, and \ref{assu full rank} hold.
If
\begin{align}
\label{eq condition of n contraction}
\gamma^{n+1} < \frac{\kappa d_{\mu, min}\min_s i(s)d_\mu(s) }{d_{\mu, max}^2 \norm{I - \gamma P_\pi^\top}_\infty \norm{m}_1},
\end{align}
then $\Pi_{f_n} \bop_\pi$ is a $\sqrt{\gamma}$-contraction in $\norm{\cdot}_{f_{n}}$ and
\begin{align}
  \label{eq etd performance bound}
\norm{Xw_{*, n} - v_\pi}_{f_n} \leq \frac{1}{\sqrt{1 - \gamma}} \norm{\Pi_{{f, n}} v_\pi - v_\pi}_{f_n}.
\end{align}
\end{lemma}
Here $\Pi_{f_n}$ denotes the projection onto the column space of $X$ w.r.t.\ the norm induced by $f_n$, i.e.,
\begin{align}
  \Pi_{f_n} v \doteq X \arg\min_{w} \norm{Xw - v}^2_{f_n}.
\end{align}
The proof of Lemma~\ref{lem contraction} is similar to \citet{hallak2016generalized} and is provided in Section \ref{sec proof of contraction}.
Again, the LHS of \eqref{eq condition of n contraction} diminishes geometrically.
So in practice, $n$ might not need to be too large.
Lemma~\ref{lem contraction} characterizes the performance of the fixed points of Truncated ETD methods in prediction settings.
In the following we highlight two points regarding those fixed points from different truncation length.
\begin{enumerate}
  \item We argue that those fixed points are equally good.
  The $\norm{\Pi_{{f, n}} v_\pi - v_\pi}_{f_n}$ term in~\eqref{eq etd performance bound} is the representation error resulting from the limit of the capacity of the linear function approximator. 
With different truncation length, 
we use different norm (i.e., $\norm{\cdot}_{f_n}$) to measure the representation error.
The multiplicative factor $\frac{1}{\sqrt{1 - \gamma}}$,
however,
does not depend on $n$.
In other words,
as long as $n$ is sufficiently large in the sense of~\eqref{eq condition of n contraction},
the exact value of $n$, 
including $n=\infty$ (i.e., no truncation),
does not seem to affect the performance of the fixed point much.
The intuition is straightforward.
Comparing~\eqref{eq truncated trace} and~\eqref{eq etd trace expanded},
it is easy to see that by using the truncation,
we discard the term $\sum_{j=n+1}^{t} \gamma^j \rho_{t-j:t-1} i_{t-j}$ corresponding to earlier transitions from steps $0$ to $t-n-1$.
This term has a large, possibly infinite, variance because of the product of importance sampling ratios.
The expectation of this term is,
however,
negligible because the expectation of the importance sampling ratios are well bounded (see the proof of Lemma~\ref{lem emphasis bound}) and the multiplicative factor $\gamma^j$ is negligible.
It is the expectation, not the variance,
of the trace that determines the performance of the corresponding fixed point.
Consequently,
the truncation proposed in this work does not seem to yield a compromise in the performance of the fixed point.
The truncation proposed in this work is more like a free variance reduction instead of a bias-variance tradeoff.
\item We argue that those fixed points are better than the fixed points of gradient TD methods
minimizing $d_\mu$-induced mean squared projected Bellman errors
(e.g., GTD in \citealt{sutton2009convergent}, 
GTD2 and TDC in \citealt{sutton2009fast}, 
Gradient Tree Backup and Gradient Retrace in \citealt{touati2018convergent}).
This is because the performance of the fixed points of Truncated ETD methods can be well-bounded by the representation error, 
provided that the length of the truncation is sufficiently large.
By contrast,
the performance of the fixed points of gradient TD methods
can be arbitrarily worse,
no matter how small the representation error is \citep{kolter2011fixed}.
\end{enumerate}

%
\section{Control: Emphatic Approximate Value Iteration}
\label{sec control eavi}

The study of the canonical approximate value iteration \eqref{eq on-policy avi} is essential to the study of the on-policy control algorithm SARSA \citep{melo2008analysis,zou2019finite}.
Similarly, in this section, we study approximate value iteration from an off-policy perspective,
which prepares us for the off-policy control algorithm in the next section.
In the rest of this paper,
we write $f, m, f_n, m_n$, $\kappa$ (defined in Lemma~\ref{lem contraction}), and $\lambda_{min}$ (defined in Lemma~\ref{lem nd}) as $f_{\mu, \pi}, m_{\mu, \pi}, f_{n, \mu, \pi}, m_{n, \mu, \pi}, \kappa_{\mu, \pi}, \lambda_{min, \mu, \pi}$ to explicitly acknowledge their dependence on $\mu$ and $\pi$.

The canonical approximate value iteration operator in \eqref{eq on-policy avi} is in a sense on-policy in that the projection operator is defined w.r.t.\ a norm induced by the policy of the current iteration.
In this section, we study approximation value iteration from an off-policy perspective,
i.e.,
with a projection operator defined w.r.t.\ a different norm. 
Let $\pi_{w}$ and $\mu_{w}$ be target and behavior policies respectively.
They depend on $w$, 
the parameters used for estimating the value function,
through the value function estimate $v = Xw \in \R^\ns$,
e.g.,
they can be softmax policies (cf. \eqref{eq softmax policy v}) with different temperatures.
We consider the iterates $\qty{v_k \doteq X w_k}$ generated by
\begin{align}
v_{k+1} \doteq \pbop{\mu_{w_k}}{\pi_{w_k}} v_k.
\end{align}
We call this scheme \emph{emphatic approximate value iteration} as the projection operator is defined w.r.t.\ the norm induced by the (truncated) followon trace.
In the rest of this section, we show that emphatic approximate value iteration adopts at least one fixed point.

With $\Lambda_M$ denoting the closure of $\qty{\mu_w \mid w \in \R^K}$ and $\Lambda_\Pi$ denoting the closure of \\$\qty{\pi_w \mid w \in \R^K}$,
we make the following assumptions.
\begin{assumption}
\label{assu continuous}
Both $\pi_w$ and $\mu_w$ are continuous in $w$.
\end{assumption}
\begin{assumption}
\label{assu closure ergodic}
For any $\mu \in \Lambda_M$, the Markov chain induced by $\mu$ is ergodic and $\mu(a | s) > 0$ holds for all $(s, a)$.
\end{assumption} 
Assumption \ref{assu continuous} is standard in analyzing approximate value iteration \citep{de2000existence}.
If $\pi_w$ is not continuous in $w$,
even the canonical approximate value iteration can fail to have a fixed point \citep{de2000existence}. 
The ergodicity assumption of all the policies in the closure in Assumption \ref{assu closure ergodic} is also standard for analyzing control algorithms,
in both on-policy \citep{marbach2001simulation} and off-policy \citep{zhang2021breaking,DBLP:journals/corr/abs-2111-02997} settings.
One common strategy to ensure this ergodicity in closure is to mix a softmax policy with a uniformly random policy,
assuming the uniformly random policy always induces an ergodic chain.

We now define two helper functions to understand how $n$ should be selected in emphatic approximate value iteration.
\begin{align}
n_1(\mu, \pi) &\doteq \frac{\ln (\lambda_{min, \mu, \pi} d_{\mu, min}) - \ln (d_{\mu, max}^2 \norm{\gamma P_\pi - I} \norm{m_{\mu, \pi}}_1)}{\ln \gamma} - 1, \\
n_2(\mu, \pi) &\doteq \frac{\ln \left(\kappa_{\mu, \pi} d_{\mu, min} \min_s i(s) d_\mu(s) \right) - \ln \left({d_{\mu, max}^2 \norm{I - \gamma P_{\pi}^\top}_\infty \norm{m_{\mu, \pi}}_1}\right)}{\ln \gamma} - 1.
\end{align}
Here $n_1$ and $n_2$ correspond to the conditions of $n$ in Lemmas \ref{lem nd} and \ref{lem contraction} respectively.
Assumption~\ref{assu closure ergodic} ensures that $n_1$ and $n_2$ are well defined on $\Lambda_M \times \Lambda_\Pi$.
The invariant distribution $d_\mu$ is continuous in $\mu$ (see, e.g., Lemma 9 of \citealt{zhang2021breaking}),
the minimum eigenvalue $\lambda_{min, \mu, \pi}$ is continuous in the elements of the matrix 
(see, e.g., Corollary 8.6.2 of \citealt{DBLP:books/daglib/0086372}) and thus is also continuous in $\mu$ and $\pi$,
and both $\Lambda_M$ and $\Lambda_\Pi$ are compact.
Therefore, $\sup_{\mu \in \Lambda_M, \pi \in \Lambda_\Pi} \max \qty{n_1(\mu, \pi), n_2(\mu, \pi)} < \infty$ by the extreme value theorem.
This allows us to select $n$ as suggested by the following lemma.
\begin{lemma}
\label{thm fixed point}
Let Assumptions \ref{assu full rank}, \ref{assu continuous}, and \ref{assu closure ergodic} hold.
If 
\begin{align}
\label{eq condition of n avi}
n > \sup_{\mu \in \Lambda_M, \pi \in \Lambda_\Pi} \max \qty{n_1(\mu, \pi), n_2(\mu, \pi)},
\end{align}
then there exists at least one $w_*$ such that
\begin{align}
Xw_* = \pbop{\mu_{w_*}}{\pi_{w_*}} Xw_*.
\end{align}
\end{lemma}
The proof of Theorem \ref{thm fixed point}
is provided in \ref{sec proof of fixed point},
which follows the same steps of \citet{de2000existence} but generalizes their results from (on-policy) approximate value iteration to emphatic approximate value iteration.

\section{Control: Truncated Emphatic Expected SARSA}
\label{sec control sarsa}
We now present our control algorithm, Truncated Emphatic Expected SARSA.
Unlike planning methods such as approximate value iteration,
learning methods for control like SARSA usually work directly on action-value estimates.
To this end,
we overload notation for the ease of presentation.
In particular, we overload the feature function $x$ as $x: \fS \times \fA \to \R^K$ to denote a state-action feature function.
Correspondingly, 
the feature matrix $X$ is now a matrix in $\R^{\nsa \times K}$ whose $(s,a)$-th row is $x(s, a)^\top$.
The transition matrix $P_\pi$ is now a matrix in $\R^{\nsa \times \nsa}$ to denote the state-action pair transition, i.e.,
\begin{align}
P_\pi((s, a), (s', a')) \doteq p(s'|s, a)\pi(a' | s').
\end{align}
Consequently, the Bellman operator is overloaded as 
\begin{align}
\bop_{\pi} q \doteq r + \gamma P_\pi q.
\end{align}
The stationary distribution $d_\mu$ is now in $\R^\nsa$ and denotes the invariant state-action pair distribution under the policy $\mu$.
$D_\mu$ is then a diagonal matrix in $\R^{\nsa \times \nsa}$.
The interest function $i$ is now from $\fS \times \fA$ to $(0, +\infty)$ to denote user's preference for each state-action pair.
The followon trace $F_t$ is now defined as
\begin{align}
F_t \doteq i_t + \gamma \rho_t F_{t-1},
\end{align}
which is the same as the followon trace used in ELSTDQ($\lambda$) in \citet{white2017unifying}.
Correspondingly,
the truncated trace is defined as
\begin{align}
\label{eq sa trace}
F_{t, n} \doteq \begin{cases}
\sum_{j=0}^n \gamma^j \rho_{t-j+1:t} i_{t-j} & t \geq n \\
F_t & t < n
\end{cases}.
\end{align} 
The truncated emphasis $m_{n, \mu, \pi}$ is now in $\R^\nsa$ and defined as
\begin{align}
m_{n, \mu, \pi}(s, a) \doteq \lim_{t \to \infty} \E[F_{t, n} | S_t = s, A_t = a].
\end{align} 
Other notation is also overloaded accordingly, e.g., 
\begin{align}
m_{\mu, \pi} \doteq \lim_{n\to\infty} m_{n, \mu, \pi}, \, f_{n, \mu, \pi} \doteq D_{\mu} m_{n, \mu, \pi}, \, f_{\mu, \pi} \doteq D_\mu m_{\mu, \pi}.
\end{align}
Previous theoretical results also hold with the overloaded notation for state-action pairs.
In particular,
we have
\begin{lemma}
\label{thm fixed point sa}
Let Assumptions \ref{assu full rank}, \ref{assu continuous}, and \ref{assu closure ergodic} hold.
Define
\begin{align}
n_1(\mu, \pi) &\doteq \frac{\ln (\lambda_{min, \mu, \pi} d_{\mu, min}) - \ln (d_{\mu, max}^2 \norm{\gamma P_\pi - I} \norm{m_{\mu, \pi}}_1)}{\ln \gamma} - 1, \\
n_2(\mu, \pi) &\doteq \frac{\ln \left(\kappa_{\mu, \pi} d_{\mu, min} \min_{s,a} i(s, a) d_\mu(s, a) \right) - \ln \left({d_{\mu, max}^2 \norm{I - \gamma P_{\pi}^\top}_\infty \norm{m_{\mu, \pi}}_1}\right)}{\ln \gamma} -1,
\end{align}
where $\lambda_{min, \mu, \pi}$ is the minimum eigenvalue of
\begin{align}
\frac{1}{2}\left(D_{f_{\mu, \pi}}(I - \gamma P_\pi) + (I - \gamma P_\pi^\top) D_{f_{\mu, \pi}} \right),
\end{align}
$d_{\mu, min} \doteq \min_{s, a} d_\mu(s, a), d_{\mu, max} \doteq \max_{s, a}d_{\mu, max}(s, a), \kappa_{\mu, \pi} \doteq \min_{s, a} \frac{d_\mu(s, a)i(s, a)}{f(s, a)}$.
If 
\begin{align}
\label{eq condition of n avi sa}
n > \sup_{\mu \in \Lambda_M, \pi \in \Lambda_\Pi} \max \qty{n_1(\mu, \pi), n_2(\mu, \pi)}
\end{align}
holds, then
\begin{enumerate}[(i).]
  \item For any $\mu \in \Lambda_M, \pi \in \Lambda_\Pi$, $X^\top D_{f_{n, \mu, \pi}} (\gamma P_\pi - I)X$ is n.d.,
  \item For any $\mu \in \Lambda_M, \pi \in \Lambda_\Pi$, $\Pi_{f_{n, \mu, \pi}}\bop_\pi$ is a $\sqrt{\gamma}$ contraction in $\norm{\cdot}_{f_{n, \mu, \pi}}$,
  \item There exists at least one $w_*$ such that
\begin{align}
\label{eq projected bellman eq sa}
Xw_* = \pbop{\mu_{w_*}}{\pi_{w_*}} Xw_*.
\end{align}
We use $\fW_*$ to denote the set of all such $w_*$.
\end{enumerate}
\end{lemma}
The proof of Lemma \ref{thm fixed point sa} is omitted 
since it is a verbatim repetition of the proofs of Lemmas \ref{lem nd}, \ref{lem contraction}, and \ref{thm fixed point}.

The iterative update scheme \eqref{eq projected bellman eq sa} is emphatic approximate value iteration applied to action-value estimation.
To implement this scheme incrementally in a learning sense,
we propose  
Truncated Emphatic Off-Policy Expected SARSA (Algorithm \ref{alg esarsa}).
When computing $F_{t,n}$,
we require that the previous importance sampling ratios be recomputed with the current weight $w_t$. 
This requirement is mainly for the ease of asymptotic analysis and is eliminated in Projected Truncated Emphatic Expected SARSA,
for which we provide a finite sample analysis.
\begin{algorithm}
$S_0 \sim p_0(\cdot)$ \;
$A_0 \sim \mu_{w_0}(\cdot | S_0)$ \;
$t \gets 0$ \;
\While{True}{
  Execute $A_t$, get $R_{t+1}, S_{t+1}$ \;
  $A_{t+1} \sim \mu_{w_t}(\cdot | S_{t+1})$ \;
  $\rho_t \gets \frac{\pi_{w_t}(A_t | S_t)}{\mu_{w_t}(A_t | S_t)}$ \;
  $F_{t, n} \gets 0$ \;
  \For{$k = 0, \dots, n$}{
    $F_{t, n} \gets i_{t-n+k} + \gamma \frac{\pi_{w_t}(A_{t-n+k} | S_{t-n+k})}{\mu_{w_t}(A_{t-n+k}|S_{t-n+k})} F_{t, n}$
  }
  $w_{t+1} \gets w_t + \alpha_t F_{t, n} (R_{t+1} + \gamma \sum_a {\pi_{w_t}(a | S_{t+1})} x(S_{t+1}, a)^\top w_t - x_t^\top w_t ) x_t $ \;
  $t \gets t+ 1$ \;
}
\caption{\label{alg esarsa}Truncated Emphatic Expected SARSA}
\end{algorithm}

We can now present our asymptotic convergence analysis of Algorithm \ref{alg esarsa}.
We first study the properties of the possible fixed points.
We can rewrite \eqref{eq projected bellman eq sa} as
\begin{align}
A_{w_*} w_* + b_{w_*} = 0,
\end{align}
where 
\begin{align}
A_w &\doteq X^\top D_{f_{n, \mu_{w}, \pi_w}}(\gamma P_{\pi_w} - I)X, \\
b_w &\doteq X^\top D_{f_{n, \mu_{w}, \pi_w}} r.
\end{align}
Consequently,
\begin{align}
w_* = A_{w_*}^{-1}b_{w_*}.
\end{align}
Since $\Lambda_M$ and $\Lambda_\Pi$ are compact,
both $\pi_w$ and $\mu_w$ are continuous in $w$,
the RHS of the above equation is bounded from above by the extreme value theorem.
Consequently, 
there exists a constant $R < \infty$ such that
\begin{align}
\sup_{w_* \in \fW_*} \norm{w_*} \leq R.
\end{align} 
We then make several regularization conditions on the policies $\pi_w$ and $\mu_w$.
For the analysis of on-policy SARSA \eqref{eq on-policy sarsa},
it is commonly assumed that the policy $\pi_w$ is Lipschitz continuous in $w$
and the Lipschitz constant is not too large \citep{perkins2003convergent,zou2019finite}.
This technical assumption is mainly used to ensure that a small change in the value estimate does not result in a big difference in the policy thus enforces certain smoothness of the overall learning process.
Without such assumptions,
even on-policy SARSA can chatter and fail to converge \citep{gordon1996chattering,gordon2001reinforcement}. 
In this paper, we adopt similar assumptions in our off-policy setting.
\begin{assumption}
\label{assu lipschitz continuous}
Both $\mu_w$ and $\pi_w$ are Lipschitz continuous in $w$, i.e.,
there exist constants $L_\mu$ and $L_\pi$ such that for any $s \in \fS, a \in \fA$, 
\begin{align}
|\pi_w(a|s) - \pi_{w'}(a|s)| &\leq L_\mu \norm{w - w'}, \\
|\mu_w(a|s) - \mu_{w'}(a|s)| &\leq L_\mu \norm{w - w'}.
\end{align}
\end{assumption}
The Lipschitz continuity of the policies immediately implies the Lipschitz continuity of $A_w$ and $b_w$ governing the expected updates of Algorithm \ref{alg esarsa}.
\begin{lemma}
  \label{lem lipschitz}
Let Assumptions \ref{assu closure ergodic} and \ref{assu lipschitz continuous} hold.
There exist positive constants $C_1, C_2, C_3,$ and $C_4$ such that for any $w, w'$
\begin{align}
\norm{A_{w} - A_{w'}} &\leq (C_1 L_\mu + C_2 L_\pi) \norm{w - w'}, \\
\norm{b_{w} - b_{w'}} &\leq (C_3 L_\mu + C_4 L_\pi) \norm{w - w'}.
\end{align}
\end{lemma}
The proof Lemma~\ref{lem lipschitz} is provided in Section~\ref{sec proof of lipschitz lemma}.
Under the conditions of Lemma~\ref{thm fixed point sa}, 
for any $w$, the matrix
\begin{align}
M(w) \doteq \frac{1}{2} \left(X^\top D_{f_{n, \mu_{w}, \pi_{w}}} (I - \gamma P_{\pi_{w}}) X + X^\top (I - \gamma P_{\pi_{w}}^\top) D_{f_{n, \mu_{w}, \pi_{w}}} X \right)
\end{align}
is p.d.
For a symmetric positive definite matrix $M$, 
let $\lambda(M)$ denote the smallest eigenvalue of $M$.
For any $w \in \R^K$, we have $\lambda(M(w)) > 0$.
By the continuity of eigenvalues in the elements of the matrix,
the compactness of $\Lambda_M$ and $\Lambda_\Pi$,
and the extreme value theorem,
we have
\begin{align}
\inf_{w \in \R^K} \lambda(M(w)) > 0.
\end{align} 
This allows us to make the following assumptions about the Lipschitz constants $L_\mu$ and $L_\pi$,
akin to \citet{perkins2003convergent,zou2019finite}.
\begin{assumption}
\label{assu smooth enough asymptotic}
$L_\mu$ and $L_\pi$ are small enough such that
\begin{align}
\lambda_{min}'  \doteq \inf_{w \in \fW} \lambda(M(w)) - \left((C_1 L_\mu + C_2 L_\pi)R + C_3 L_\mu + C_4 L_\pi\right) > 0.
\end{align}
\end{assumption}

With these regularizations on $\pi_w$ and $\mu_w$,
we can now present a high probability asymptotic convergence analysis for Algorithm \ref{alg esarsa}.
\begin{theorem}
\label{thm esarsa convergence}
Let the assumptions and conditions in Lemma~\ref{thm fixed point sa} hold.
Let Assumptions \ref{assu lr}, \ref{assu lipschitz continuous}, and \ref{assu smooth enough asymptotic} hold. 
Then for any compact set $\fW \subset \R^K$ and any $w \in \fW$,
there exists a constant $C_\fW$ such that for any $w_* \in \fW_*$,
the iterates $\qty{w_t}$ generated by Algorithm \ref{alg esarsa} satisfy
\begin{align}
\label{eq high prob bound}
\Pr(\lim_{t\to\infty} w_t = w_* \mid w_0 = w) \geq 1 - C_\fW \sum_{t=0}^\infty \alpha_t^2.
\end{align} 
This immediately implies that $\fW_*$ contains only one element (under the conditions of this theorem).
\end{theorem}
In \eqref{eq high prob bound},
$C_\fW$ depends on the compact set $\fW$ from which the weight $w_0$ is selected.
For \eqref{eq high prob bound} to be nontrivial, 
the learning rates have to be small enough,
depending on the choice of initial weights.
The proof of Theorem \ref{thm esarsa convergence} is provided in Section \ref{sec proof esarsa convergence} and depends on Theorem 13 of \citet{DBLP:books/sp/BenvenisteMP90}.\footnote{It might be possible to obtain an almost sure convergence of Algorithm \ref{alg esarsa} like Theorem \ref{thm etd convergence} by invoking Theorem 17 of \citet{DBLP:books/sp/BenvenisteMP90}.
Doing so
requires verifying (1.9.5) of \citet{DBLP:books/sp/BenvenisteMP90}.
If how \citet{melo2008analysis} verify (1.9.5) was documented in the context of on-policy SARSA with linear function approximation, 
it is expected that (1.9.5) can also be similarly verified in the context of Algorithm \ref{alg esarsa}.
}

\begin{algorithm}
Initialize $w_0$ such that $\norm{w_0} \leq R$ \; 
$S_0 \sim p_0(\cdot)$ \;
$A_0 \sim \mu_{w_0}(\cdot | S_0)$ \;
$t \gets 0$ \;
\While{True}{
  Execute $A_t$, get $R_{t+1}, S_{t+1}$ \;
  $A_{t+1} \sim \mu_{w_t}(\cdot | S_{t+1})$ \;
  $\rho_t \gets \frac{\pi_{w_t}(A_t | S_t)}{\mu_{w_t}(A_t | S_t)}$ \;
  $F_{t, n} \gets 0$ \;
  \For{$k = 0, \dots, n$}{
    $F_{t, n} \gets i_{t-n+k} + \gamma \rho_{t-n+k} F_{t, n}$
  }
  $w_{t+1} \gets \Pi_R\left(w_t + \alpha_t F_{t, n} (R_{t+1} + \gamma \sum_a {\pi_{w_t}(a | S_{t+1})} x(S_{t+1}, a)^\top w_t - x_t^\top w_t ) x_t \right)$ \;
  $t \gets t+ 1$ \;
}
\caption{\label{alg pesarsa}Projected Truncated Emphatic Expected SARSA}
\end{algorithm}
We now analyze the convergence rate of Projected Truncated Emphatic Expected SARSA (Algorithm \ref{alg pesarsa}).
Unlike Algorithm \ref{alg esarsa},
when computing $F_{t,n}$ in Algorithm \ref{alg pesarsa},
we do \emph{not} need to recompute previous importance sampling ratios.
Similar to Assumption \ref{assu smooth enough asymptotic},
we make the following assumption about the Lipschitz constants $L_\mu$ and $L_\pi$ for analyzing Algorithm~\ref{alg pesarsa}.
\begin{assumption}
\label{assu smooth enough}
$L_\mu$ and $L_\pi$ are not too large such that
\begin{align}
\lambda_{min}''  \doteq \inf_{w_* \in \fW_*} \lambda(M(w_*)) - \left((C_1 L_\mu + C_2 L_\pi)R + C_3 L_\mu + C_4 L_\pi\right) > 0.
\end{align}
\end{assumption}
When defining $\lambda_{min}'$ in Assumption \ref{assu smooth enough asymptotic},
the infimum is taken over all possible $w$.
When defining $\lambda_{min}''$ in Assumption \ref{assu smooth enough},
the infimum is taken over only $\fW_*$.
This improvement is made possible by the introduction of the projection $\Pi_R$.  

To analyze the convergence rate of Algorithm~\ref{alg pesarsa},
it is crucial to know how fast the induced chain mixes,
which is provided by the following lemma.

\begin{lemma}
\label{assu uniform ergodicity}
(Lemma 1 of \citet{DBLP:journals/corr/abs-2111-02997})
Let Assumption~\ref{assu closure ergodic} hold.
Then there are constants $C_0 > 0$ and $\kappa \in (0, 1)$ such that for any $w \in \R^K$ and $t \geq 0$,
the chain $\qty{S_t}_{t=0, 1, \dots}$ induced by the policy $\mu_w$ satisfies
\begin{align}
\sum_{s \in \fS} \left|\Pr(S_t = s) - \bar d_{\mu_w}(s) \right| \leq C_0 \kappa^t,
\end{align}
where $\bar d_{\mu_w}$ is the invariant state distribution induced by the policy $\mu_w$.
\end{lemma}
Lemma~\ref{assu uniform ergodicity} is usually referred to as \emph{uniform mixing}
since the mixing rate $\kappa$ does not depend on $w$.
This uniform mixing appears to be a technical assumption in \citet{zou2019finite}.
We are now ready to present our finite sample analysis.
\begin{theorem}
\label{thm finite sample sa}
Let the assumptions and conditions in Lemma \ref{thm fixed point sa} hold.
Let Assumptions \ref{assu lipschitz continuous} and \ref{assu smooth enough} hold.
Set the learning rate $\qty{\alpha_t}$ in Algorithm \ref{alg esarsa} to
\begin{align}
  \label{eq learning rate}
\alpha_t \doteq \frac{1}{2 \alpha_\lambda (t+1)},
\end{align}
where $\alpha_\lambda \in (0, \lambda_{min}'')$ is some constant.
Then for any $w_* \in \fW_*$, for sufficiently large $t$ (in the sense that $t - \fO(\ln t) > n$), 
the weight vector $w_t$ generated by Algorithm \ref{alg pesarsa} satisfies
\begin{align}
\label{eq esarsa convergence}
\E\left[ \norm{w_t - w_*}^2\right] = \fO\left( \frac{\ln^3t}{t}\right).
\end{align}
This immediately implies that $\fW_*$ contains only one element (under the conditions of this theorem).
\end{theorem}
The proof of Theorem \ref{thm finite sample sa} and the constants hidden by $\fO(\cdot)$ are detailed in Section \ref{sec proof finite sample sa}.
The proof follows the same steps as \citet{zou2019finite}
but generalizes the analysis of the on-policy SARSA in \citet{zou2019finite} to the off-policy setting and includes backward traces, which are not included in \citet{zou2019finite}.

In this section,
we present (Projected) Truncated Emphatic Expected SARSA as a convergent off-policy control algorithm with linear function approximation.
Importantly,
in Algorithms \ref{alg esarsa} and \ref{alg pesarsa},
the behavior policy is a function of the current action-value estimates and thus changes every time step and can be very different from the target policy.
These two features are common in practice (see, e.g. \citet{mnih2015human}) but rarely appreciated in existing literature.
For example,
in Greedy-GQ \citep{maei2010toward,wang2020finite},
a control algorithm in the family of the gradient TD methods,
the behavior policy is assumed to be fixed.
In the convergent analysis of linear $Q$-learning \citep{melo2008analysis,lee2019unified},
the behavior policy is assumed to be sufficiently close to the policy that linear $Q$-learning is expected to converge to.

\section{Related Work}

Besides gradient TD and emphatic TD methods,
there are also other methods for addressing the deadly triad,
including density-ratio-based methods \citep{hallak2017consistent,liu2018breaking,gelada2019off,nachum2019dualdice,zhang2020gradientdice} and target-network-based methods \citep{zhang2021breaking}.
Density-ratio-based methods rely on learning the density ratio $\frac{d_\pi(s)}{d_\mu(s)}$ directly via function approximation.
This ratio can then be used to reweight the off-policy TD update \eqref{eq off-policy td} \citep{hallak2017consistent} if the goal is to learn the value function of the target policy or 
reweight rewards directly when computing the empirical average of rewards \citep{liu2018breaking} if the goal is to obtain a scalar performance metric of the target policy.
Density ratio learning suffers from the same critical problem as emphasis learning and gradient TD methods.
Under general conditions for linear function approximation, 
it is hard to bound the distance between the learned density ratio and the ground truth.
Consequently,
the downstream updates relying on the learned density ratio rarely have  rigorous performance guarantees.
By contrast,
the performance of the value function estimation resulting from emphatic TD methods is well bounded (\citet{hallak2016generalized} and Lemma \ref{lem contraction}).

Target-network-based methods rely on the use of a target network \citep{mnih2015human} for bootstrapping. 
Though provably convergent,
the methods of \citet{zhang2021breaking} require a sufficiently large ridge regularization,
introducing extra bias.
By contrast,
emphatic TD methods do not require any extra regularization.

Variance reduction is an active research area in RL (e.g., \citealt{du2017stochastic,papini2018stochastic}),
which is usually achieved by designing proper control variates, see, e.g., \citet{johnson2013accelerating}.
By contrast, we reduce variance by truncating the followon trace directly.
This technique is specifically designed for emphatic TD methods and we leave the possible combination of truncated emphatic TD methods and standard variance reduction techniques for future work. 

\section{Experiments}
In this section, 
we empirically investigate the proposed truncated emphatic TD methods,
focusing on the effect of $n$.
The implementation is made publicly available to facilitate future research.\footnote{\url{https://github.com/ShangtongZhang/DeepRL}}

\begin{figure}[h]
  \centering
  \includegraphics[width=0.5\textwidth]{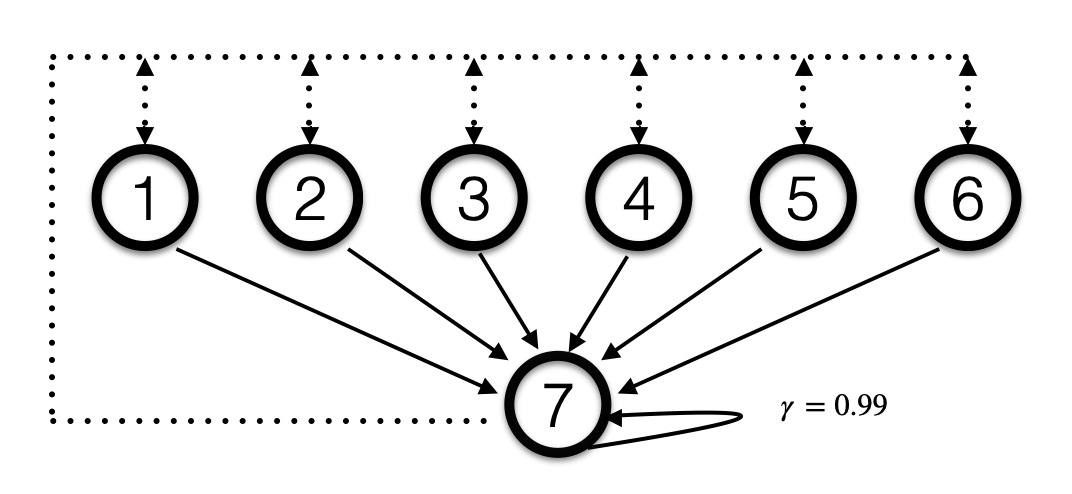}
  \caption{\label{fig MDP} Baird's counterexample from Chapter 11.2 of \citet{sutton2018reinforcement}.
  The figure is taken from \citet{zhang2019provably}. 
  There are two actions available at each state, \texttt{dashed} and \texttt{solid}.
  The \texttt{solid} action always leads to state $7$.
  The \texttt{dashed} action leads to one of states 1 - 6,
  with equal probability.
  The discount factor is $\gamma = 0.99$.
  The reward is always 0.
  The initial state is sampled uniformly from all the seven states.}
\end{figure}

\begin{figure}[t]
  \centering
  \includegraphics[width=\textwidth]{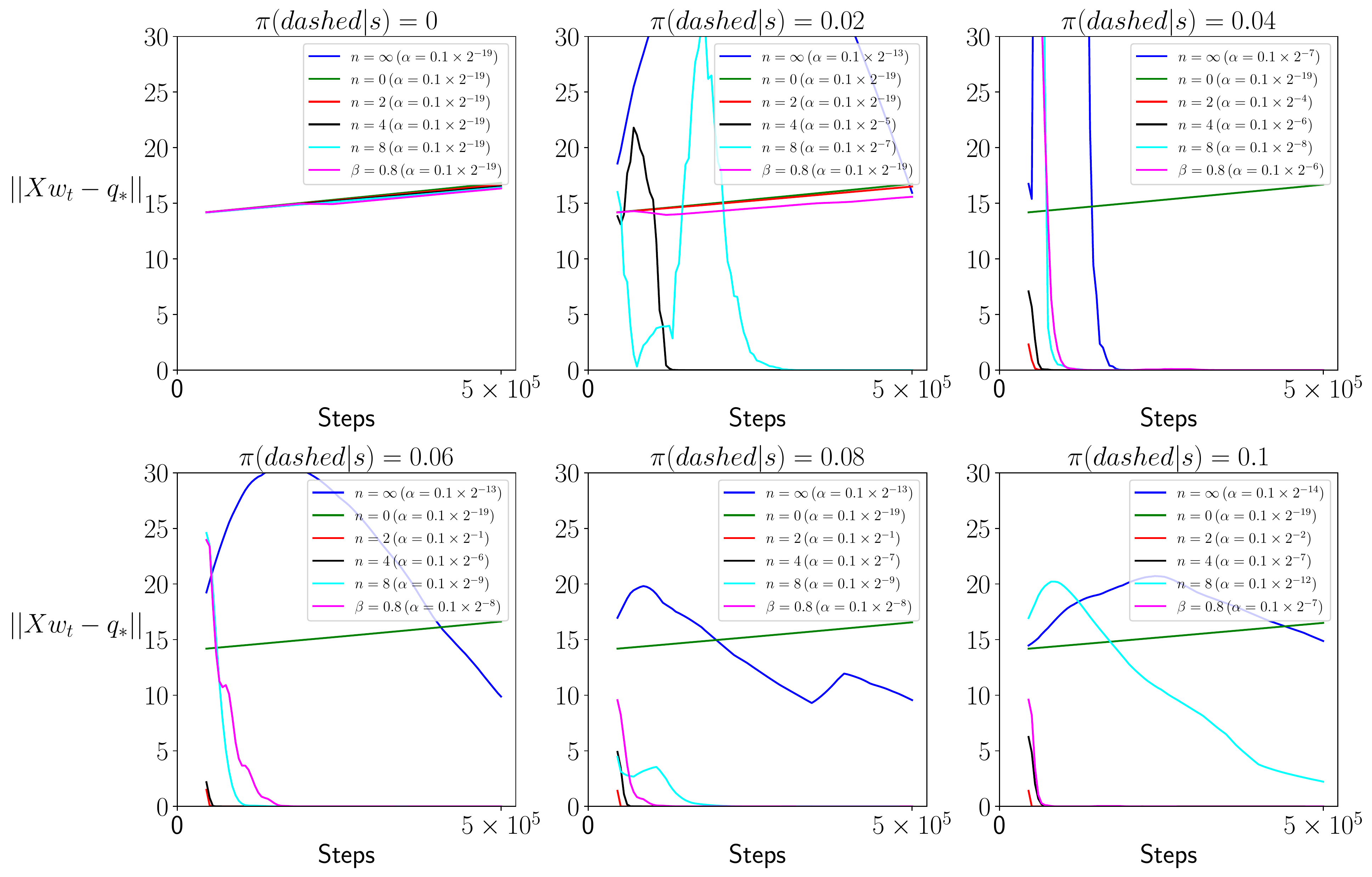}
  \caption{\label{fig prediction} Truncated Emphatic TD and ETD(0, $\beta$) in the prediction setting. 
  To improve readability,
  this figure contains only \emph{one representative run} and uses a sliding window of size 10 for smoothing.
  A more informative but harder to read version including 30 independent runs without smoothing is provided in Figure~\ref{fig prediction full} in Section~\ref{sec comp plots}.
  The curves in the two figures share similar trends and all the discussion in the paper is based on the comprehensive results in Figure~\ref{fig prediction full}.}
\end{figure}

\begin{table}
  \centering
\begin{tabular}{c|cccccc}
  \toprule
  {} &         $n=\infty$ & $n=0$ &                $n=2$ &              $n=4$ &      $n=8$ & $\beta=0.8$ \\
  \midrule
  $\pi(dashed=0|s)$    &          - &     - &                    - &                  - &          - &           - \\
  $\pi(dashed=0.02|s)$ &          - &     - &                    - &  \textbf{$10^{4}$} &  $10^{14}$ &           - \\
  $\pi(dashed=0.04|s)$ &   $10^{7}$ &     - &    \textbf{$10^{1}$} &           $10^{1}$ &   $10^{9}$ &    $10^{9}$ \\
  $\pi(dashed=0.06|s)$ &          - &     - &             $10^{2}$ &  \textbf{$10^{0}$} &   $10^{4}$ &    $10^{4}$ \\
  $\pi(dashed=0.08|s)$ &          - &     - &   \textbf{$10^{-1}$} &           $10^{0}$ &   $10^{7}$ &    $10^{7}$ \\
  $\pi(dashed=0.1|s)$  &          - &     - &  \textbf{$10^{-11}$} &           $10^{0}$ &   $10^{2}$ &    $10^{4}$ \\
  \bottomrule
\end{tabular}
\caption{\label{tab prediction} 
Average variance of curves in Figure~\ref{fig prediction full}.
Each curve in Figure~\ref{fig prediction full} consists of 100 data points. 
The average variance of those data points is reported in this table.
Here we consider only successful configurations whose averaged prediction error at the end of training is smaller than 5.
The average variance of other curves are not included and denoted as ``-''.
Other tables in this section also follow this reporting protocol.
}
\end{table}

We first use Baird's counterexample as the benchmark,
which is illustrated in Figure \ref{fig MDP}.
We consider three different settings:
prediction, control with a fixed behavior policy,
and control with a changing behavior policy.
For the prediction setting,
we consider a behavior policy $\mu(\texttt{solid}|s) = \frac{1}{7}$ and $\mu(\texttt{dashed} | s) = \frac{6}{7}$,
which is the same as the behavior policy used in \citet{sutton2018reinforcement}.
We consider different target policies 
from $\pi(\texttt{dashed}|s)=0$ to $\pi(\texttt{dashed}|s)=0.1$.
We consider linear function approximation,
where the features and the initialization of the weight vector are the same as Section D.2 of \citet{zhang2021breaking}.
We benchmark Algorithm \ref{alg etd} with different selection of $n$.
When $n=\infty$,
Algorithm~\ref{alg etd} reduces to the original ETD(0).
When $n=0$,
Algorithm~\ref{alg etd} reduces to the naive off-policy TD. 
We use a fixed learning rate $\alpha$,
which is tuned from 
$\Lambda_\alpha \doteq \qty{0.1 \times 2^{0}, 0.1 \times 2^{-1}, \dots, 0.1 \times 2^{-19}}$ for each $n$,
with 30 independent runs.
We report learning curves with the learning rate minimizing the value prediction error at the end of training.
Additionally,
we also benchmark ETD$(0, \beta)$,
where we replace the $F_{t, n}$ in Algorithm \ref{alg etd} with the trace $F_{t, \beta}$ computed via \eqref{eq hallak trace}.
We tune $\beta$ in $\qty{0.1, 0.2, 0.4, 0.8}$.
For each $\beta$,
we tune the learning rate $\alpha$ in $\Lambda_\alpha$ as before.
The interest is 1 for all states (i.e., $i(s) \equiv 1 \, \forall s$). 
We report the learning curves with the best $\beta$.
All curves are averaged over 30 independent runs with shaded regions indicating \emph{standard errors},
unless otherwise specified.
This experimental and reporting protocol is also used in all the remaining experiments in this paper.

As shown by Figures \ref{fig prediction} and \ref{fig prediction full} with $n=0$,
the naive off-policy TD makes no progress in this prediction setting.
The curve is almost flat because the best learning rate is $0.1 \times 2^{-19}$;
using any larger learning rate simply accelerates divergence.
As shown by the curves with $n=\infty$,
naive ETD(0) does make some progresses when $\pi(\texttt{dashed}|s) > 0$,
though the final prediction errors at the end of training are usually large.
By contrast,
using $n=4$ leads to quick convergence in all the tasks with $\pi(\texttt{dashed}|s) > 0$.
Reducing $n$ from 4 to 2 also works when $\pi(\texttt{dashed}|s) \geq 0.04$ and
increasing $n$ from 4 to 8 significantly increases the variance.
Obviously increasing $n$ leads to a larger variance,
so in practice we want to find the smallest $n$.
Moreover, though ETD($0, \beta$) converges when $\pi(\texttt{dashed}|s) \geq 0.04$,
it usually exhibits larger variance than our Truncated ETD with $n=2$ or $n=4$ (Table~\ref{tab prediction}).
We conjecture that this is because the trace \eqref{eq hallak trace} still relies on all the history. 
Consider, e.g., $\pi(\texttt{dashed}|s) = 0.02$: the maximum importance ratio is $\rho_{max} = 0.98 \times 7 = 6.86$.
If $\beta \rho_{max} > 1$,
there is still a chance that the trace in \eqref{eq hallak trace} goes to infinity since it depends on all the history.
However, requiring $\beta \rho_{max} < 1$ would require using a small $\beta$,
which itself could also lead to instability.
By contrast,
with truncation, $F_{t,n}$ is always guaranteed to be bounded.
The results suggest that our hard truncation also has empirical advantages over the soft truncation in \citet{hallak2016generalized},
besides the theoretical advantages of enabling finite sample analysis for both prediction and control settings. 
It can be analytically computed that 
for all $\pi(\texttt{dashed} | s) \in \qty{0, 0.02, 0.04, 0.06, 0.08, 0.1}$,
the desired $n$ as suggested by Lemma~\ref{lem nd} is around 700.
The $n$ we use in the experiments is much smaller than the suggested one.
This is because Lemma~\ref{lem nd} has to be conservative enough to cope well with all possible MDPs.
In this work,
we focus on establishing the existence of such an $n$ and giving an initial but possibly loose bound. 
We leave the improvement of Lemma~\ref{lem nd} for future work.
For computational experiments,
we recommend to treat $n$ as an additional hyperparameter.

When $\pi(\texttt{dashed}|s) = 0$,
which is used in the original Baird's counterexample,
no selection of $n$ or $\beta$ is able to make any progress.
The failure of ETD(0) with this target policy is also observed by \citet{sutton2018reinforcement}.
This target policy is particularly challenging because the its off-policyness is the largest in all the tested target policies,
making it hard to observe progresses in computational experiments.
Though truncation is not guaranteed to always reduce the variance to desired levels while maintaining convergence,
our experiments in the prediction setting do suggest it is a promising approach. 
We leave a more in-depth investigation with this target policy for future work.

\begin{figure}[t]
  \centering
  \includegraphics[width=\textwidth]{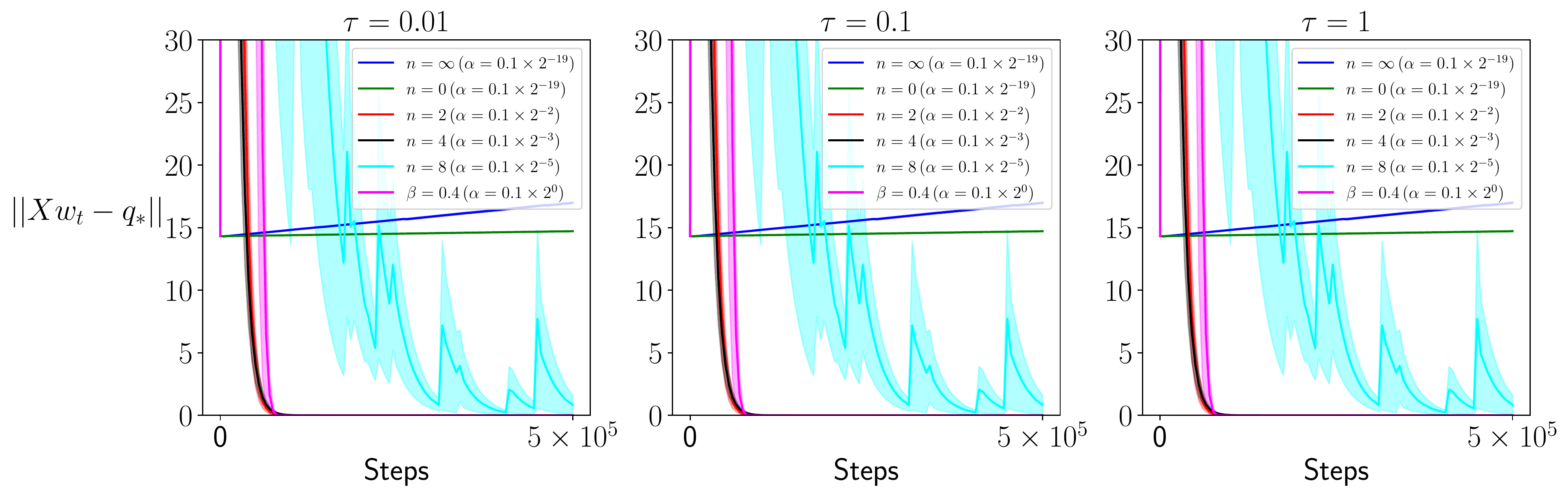}
  \caption{\label{fig control1} Truncated Emphatic Expected SARSA and its $\beta$-variant in the control setting with a fixed behavior policy. 
  The shaded regions are invisible for some curves because their standard errors are too small.}
\end{figure}

\begin{table}
  \centering
  \begin{tabular}{c|cccccc}
    \toprule
    {} & $n=\infty$ & $n=0$ &     $n=2$ &              $n=4$ &     $n=8$ & $\beta=0.8$ \\
    \midrule
    $\tau=0$    &          - &     - &  $10^{4}$ &  \textbf{$10^{3}$} &  $10^{6}$ &   $10^{11}$ \\
    $\tau=0.01$ &          - &     - &  $10^{4}$ &  \textbf{$10^{3}$} &  $10^{6}$ &   $10^{11}$ \\
    $\tau=0.1$  &          - &     - &  $10^{4}$ &  \textbf{$10^{3}$} &  $10^{6}$ &   $10^{11}$ \\
    \bottomrule
  \end{tabular}
  \caption{\label{tab control1} Average variance of curves in Figure~\ref{fig control1}. Here $n=4$ has smaller variance than $n=2$ because the former converges slightly faster.}
\end{table}

In the control setting with a fixed behavior policy,
we benchmark Algorithm~\ref{alg pesarsa} with different selection of $n$,
as well as its $\beta$-variant (cf.\ \eqref{eq hallak trace}).
In particular, 
we set the radius of the ball for projection to be infinity (i.e., the projection is now an identity mapping).
Consequently,
when $n=\infty$,
our implementation of Algorithm \ref{alg pesarsa} becomes a straightforward extension of ETD(0) to the control setting. 
We use the same behavior policy as the prediction setting.
The target policy is a softmax policy with a temperature $\tau$: 
\begin{align}
  \pi(\texttt{dashed}|s) \doteq \frac{\exp\left(q(s, \texttt{dashed})/\tau\right)}{\exp\left(q(s, \texttt{dashed})/\tau\right) + \exp\left(q(s, \texttt{solid})/\tau\right)}.
\end{align}
We test three different temperatures $\tau \in \qty{0.01, 0.1, 1}$.
When $\tau$ approaches $0$,
the target policies become more and more greedy.
Consequently, Algorithm \ref{alg pesarsa} approaches $Q$-learning. 
As shown in Figure~\ref{fig control1},
neither the naive off-policy expected SARSA (i.e., $n=0$) nor the naive extension of ETD(0) (i.e., $n=\infty$) makes any progress in this setting.
By contrast, 
our Truncated Empathic Expected SARSA consistently converges,
with lower variance than its $\beta$-variant (Table~\ref{tab control1}).

\begin{figure}[t]
  \centering
  \includegraphics[width=\textwidth]{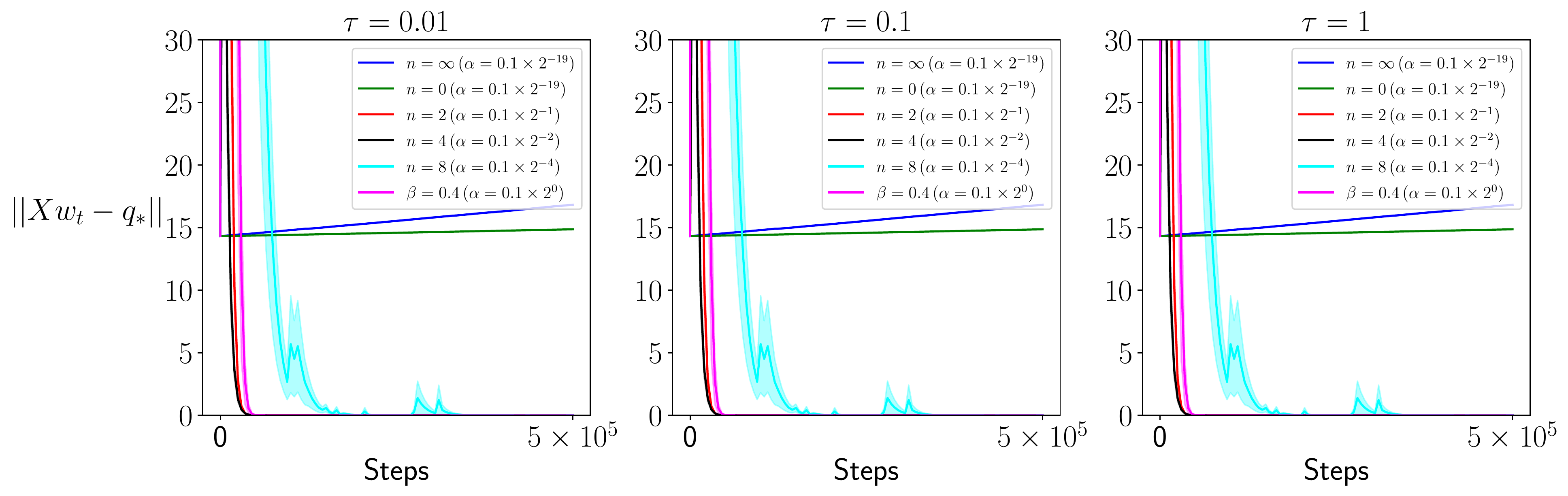}
  \caption{\label{fig control2} Truncated Emphatic Expected SARSA and its $\beta$-variant in the control setting with a changing behavior policy.
  The shaded regions are invisible for some curves because their standard errors are too small.}
\end{figure}

\begin{table}
  \centering
  \begin{tabular}{c|cccccc}
    \toprule
    {} & $n=\infty$ & $n=0$ &     $n=2$ &              $n=4$ &     $n=8$ & $\beta=0.8$ \\
    \midrule
    $\tau=0$    &          - &     - &  $10^{3}$ &  \textbf{$10^{2}$} &  $10^{6}$ &    $10^{6}$ \\
    $\tau=0.01$ &          - &     - &  $10^{3}$ &  \textbf{$10^{2}$} &  $10^{6}$ &    $10^{6}$ \\
    $\tau=0.1$  &          - &     - &  $10^{3}$ &  \textbf{$10^{2}$} &  $10^{6}$ &    $10^{6}$ \\
    \bottomrule
  \end{tabular}
  \caption{\label{tab control2} Average variance of curves in Figure~\ref{fig control2}. Here $n=4$ has smaller variance than $n=2$ because the former converges slightly faster.}
\end{table}

In the control setting with a changing behavior policy,
we still benchmark Algorithm~\ref{alg pesarsa} with a different selection of $n$ and its $\beta$-variant.
The target policy is still the softmax policy with a temperature $\tau$.
The behavior policy is now a mixture policy same as the one used in \citet{zhang2021breaking}.
At each time step,
with probability 0.9,
the behavior policy is the same as the behavior policy used in the prediction setting;
with probability 0.1,
the behavior policy is a softmax policy with temperature 1.
As shown by Figure~\ref{fig control2} and Table~\ref{tab control2},
the results in this setting are similar to the previous setting with a fixed behavior policy but the variance with $n \in \qty {2, 4}$ is reduced.
This is because the behavior policy is now related to the target policy,
i.e.,
the off-policyness is reduced.

\begin{figure}[h]
  \centering
  \includegraphics[width=0.5\textwidth]{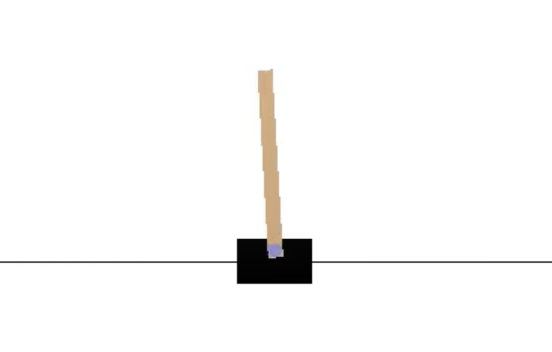}
  \caption{\label{fig cartpole} CartPole. 
  At each time step,
  we observe the \texttt{velocity}, \texttt{acceleration}, \texttt{angular velocity}, and \texttt{angular acceleration} of the pole and move the car \texttt{left} or \texttt{right} to keep the pole balanced.
  The reward is +1 every time step.
  An episode ends if a maximum of 1000 steps is reached or the pole falls.}
\end{figure}

We further evaluate Truncated Emphatic TD methods in the CartPole domain (Figure~\ref{fig cartpole}),
which is a classical nonsynthetic control problem.
We use tile coding \citep{sutton1995generalization} to map the four-dimensional observation 
(velocity, acceleration, angular velocity, angular acceleration) 
to a binary vector in $\R^{1024}$ and then apply linear function approximation.
In particular,
we use the tile coding software recommended in Chapter 10.1 of \citet{sutton2018reinforcement}.
We benchmark Algorithm~\ref{alg pesarsa} and its $\beta$-variant (cf. \eqref{eq hallak trace}),
following the same hyperparameter tuning protocol as in Baird's counterexample.  
We use $\gamma = 0.99$ and $i(s) = 1$.
The target policy is a softmax policy with temperature $\tau = 0.01$.
The behavior policy is a $\epsilon$-softmax policy with $\epsilon = 0.95$ and $\tau = 1$.
In other words,
at each time step,
with probability 0.95,
the agent selects an action according to a uniformly random policy;
with probability 0.05,
the agent selects an action according to a softmax policy with temperature $\tau = 1$.
We grant large randomness to the behavior policy to enlarge the off-policyness of the problem,
making it more challenging.
We evaluate the agent every $5 \times 10^3$ steps during the training process for 10 episodes and report the averaged undiscounted episodic return.
Figure~\ref{fig cartpole expt} (Left) investigates the effect of different truncation length.
We recall that the learning rate $\alpha$ is tuned from $\Lambda_\alpha$ maximizing the evaluation return at the end of the training.
With $n=\infty$ (i.e., no truncation),
the agent barely learns anything.
With $n=0$ (i.e., naive off-policy expected SARSA without followon trace),
the agent reaches a reasonable performance level but using $n=4$ performs better.
Using $n=2$ performs better than using $n=4$
in the middle of the training but the performance drops near the end of the training.
We conjecture that this may suggest that a truncation length of 2 is not enough to stabilize the off-policy training in the tested problem.
We note that being able to achieve a reasonable performance with $n=0$ does not mean there is no stability issue with $n=0$,
since divergence to infinity and failing to learn at all is not the only consequence of instability.
For example,
the iterates can also chatter in a bounded region \citep{gordon1996chattering,gordon2001reinforcement},
which might be accountable for the early plateau of the curve with $n=0$.
Increasing $n$ improves stability and might help escape from the early plateau.
Figure~\ref{fig cartpole expt}~(Right) further investigates the soft truncation using~\eqref{eq hallak trace}.
We recall that $\beta$ is tuned from $\qty{0.1, 0.2, 0.4, 0.8}$.
Using the soft truncation with $\beta = 0.2$ performs similar to using the hard truncation with $n=4$.
It can, however, 
be computed that the data points of the curve with $\beta=0.2$ has an average variance around $1.8 \times 10^4$ while that of $n=4$ is around $7 \times 10^3$.
This suggests that our proposed hard truncation might be a better option for variance reduction than the existing soft truncation for the tested problem.

\begin{figure}[h]
  \centering
  \includegraphics[width=0.4\textwidth]{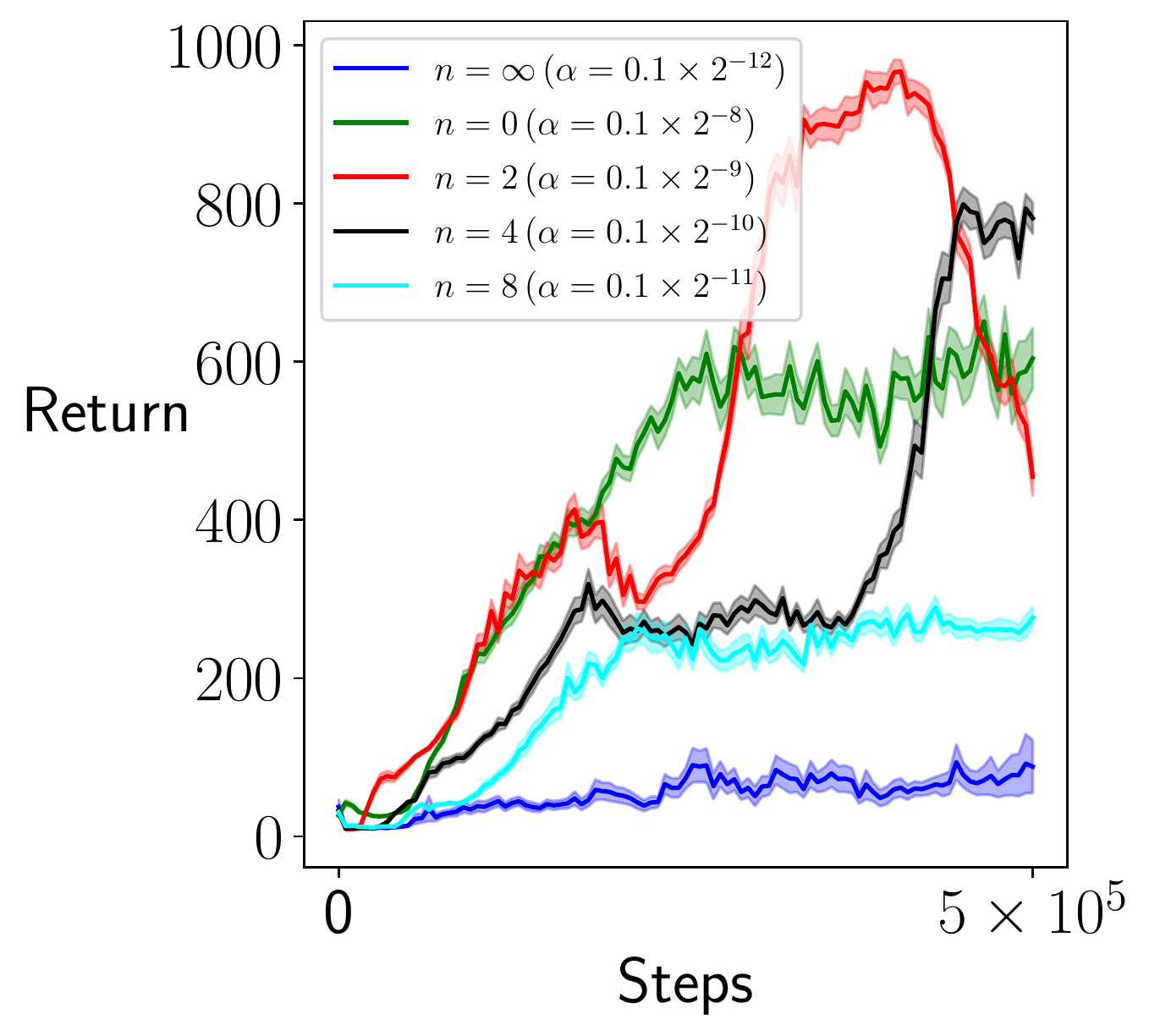}
  \includegraphics[width=0.4\textwidth]{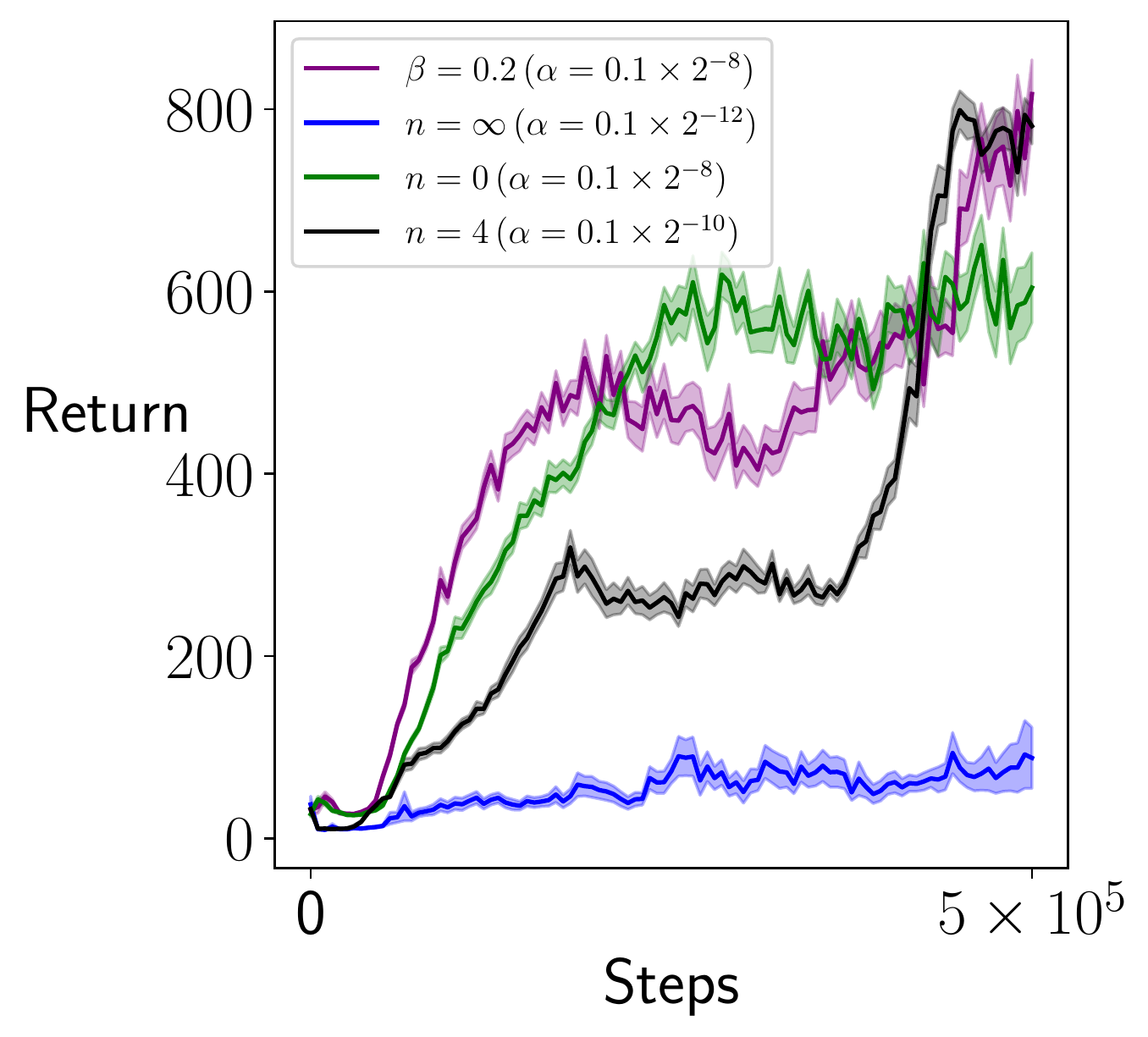}
  \caption{\label{fig cartpole expt} Truncated Emphatic Expected SARSA and its $\beta$-variant in the CartPole domain.}
\end{figure}

\section{Conclusion}
In this paper,
we addressed the two open problems in emphatic TD methods simultaneously by using truncated followon traces.
Our analysis is limited to ETD(0) but the extension to ETD($\lambda$) is straightforward,
which we leave for future work.
The idea of using truncated traces as a variance reduction technique can also be applied to other trace-based off-policy RL algorithms,
e.g., GTD($\lambda$) \citep{maei2011gradient},
and other variants of followon traces, e.g., ETD($\lambda, \beta$) and NETD,
which we also leave for future work.
In this paper,
we mainly focused on value-based methods.
A possibility for future work is to equip followon-trace-based actor-critic algorithms (e.g., \citealt{imani2018off,zhang2019provably}) with the truncated followon trace.
Further,
similar to the canonical approximate value iteration, 
bounding the performance of the fixed points of emphatic approximate value iteration also remains an open problem. 
In this paper,
we restricted our empirical study to
linear function approximation.
Empirically investigating truncated emphatic TD methods with large neural networks like \citet{jiang2021} is also a possibility for future work.   

\acks{The authors thank the action editor and the anonymous reviewers for their insightful feedback. 
SZ is generously funded by the Engineering and Physical Sciences Research Council (EPSRC).
This project has received funding from the European Research Council under the European Union’s Horizon 2020 research and innovation programme (grant agreement number 637713). The experiments were made possible by a generous equipment grant from NVIDIA.}

\appendix

\section{ODE-Based Convergent Results}

\subsection{Proposition 4.8 of \citet{bertsekas1996neuro}}
\label{sec ndp}
Consider the iterates $\qty{w_t}$ evolving in $\R^K$ defined as
\begin{align}
  w_{t+1} \doteq w_t + \alpha_t \left(A(O_t)w_t + b(O_t)\right),
\end{align}
where $\qty{O_t}$ denote a Markov chain in a space $\fO$, 
$\qty{\alpha_t}$ is a sequence of learning rates, 
$A$ and $b$ are functions from $\fO$ to $\R^{K \times K}$ and $\R^K$ respectively.

\begin{assumption}
  \label{assu ndp lr}
  $\qty{\alpha_t}$ is a deterministic, positive, nonincreasing sequence such that 
  \begin{align}
    \sum_t \alpha_t = \infty, \sum_t \alpha_t^2 < \infty.
  \end{align}
\end{assumption}
\begin{assumption}
\label{assu ndp ergodic}
  The chain $\qty{O_t}$ adopts a unique invariant distribution denoted by $d_\fO$.
\end{assumption}
\begin{assumption}
  \label{assu ndp nd}
  The matrix $\bar A \doteq \E_{O_t \sim d_\fO}[A(O_t)]$ is n.d..
\end{assumption}
\begin{assumption}
  \label{assu ndp bounded}
  $\sup_{o \in \fO} \norm{A(o)} < \infty, \sup_{o \in \fO} \norm{b(o)} < \infty$
\end{assumption}
\begin{assumption}
  \label{assu ndp mixing}
  There exist constants $C_0 > 0$ and $\kappa \in (0, 1)$ such that
  \begin{align}
    \norm{\E \left[ A(O_t)\right] - \bar A}  \leq C_0 \kappa^t, \\
    \norm{\E \left[ b(O_t)\right] - \bar b}  \leq C_0 \kappa^t,
  \end{align}
  where $\bar b \doteq \E_{O_t \sim d_\fO}[b(O_t)]$.
\end{assumption}

\begin{theorem}
  \label{thm ndp}
  Let Assumptions \ref{assu ndp lr} - \ref{assu ndp mixing} hold. Then
  \begin{align}
    \lim_{t\to \infty} w_t = \bar A^{-1}\bar b \qq{a.s..}
  \end{align}
\end{theorem}

\subsection{Theorem 13 of \citet{DBLP:books/sp/BenvenisteMP90}}
Consider the iterates $\qty{w_t}$ evolving in $\R^K$ defined as 
\begin{align}
\label{eq aasa iterates}
w_{t+1} \doteq w_t + \alpha_t H(w_t, O_{t+1}),
\end{align}
where $\qty{O_t \in \R^L}$ are random variables, $\qty{\alpha_t}$ is a sequence of learning rates, $H$ is a function from $\R^K \times \R^L$ to $\R^K$.
We use $\fF_t$ to denote the $\sigma$-field generated by $\qty{w_0, O_0, O_1, \dots, O_t}$ and make the following assumptions:
\begin{assumption}
\label{assu aasa lr}
$\qty{\alpha_t}$ is a deterministic, positive, nonincreasing sequence such that 
\begin{align}
\sum_t \alpha_t = \infty, \sum_t \alpha_t^2 < \infty.
\end{align}
\end{assumption}
\begin{assumption}
\label{assu aasa markov}
There exists a family $\qty{P_w \mid w \in \R^L}$ of parameterized transition probabilities $P_w$ on $\R^L$ such that for any $B \in \fB(\R^K)$,
\begin{align}
\Pr(O_{t+1} \in B | \fF_t) = P_{w_t}(O_t, B).
\end{align}
Additionally, for any function $f$ defined on $\R^L$,
we define $(P_w f)(o) \doteq \int f(y)P_w(o, dy)$.
Here $\fB(\cdot)$ denotes the Borel sets.
\end{assumption}
\begin{assumption}
\label{assu aasa bounds}
Let $D$ be an open subset of $\R^K$.
For any compact subset $Q$ of $D$,
there exists constants $C_1, q_1$ (depending on $Q$), 
such that for any $w \in Q$ and any $o$, we have
\begin{align}
\norm{H(w, o)} \leq C_1 (1 + \norm{o}^{q_1}).
\end{align}
\end{assumption}
\begin{assumption}
\label{assu aasa possion}
There exists a function $h: D \to \R^K$,
and for each $w \in D$ a function $\nu_w: \R^L \to \R^K$,
such that
\begin{enumerate}[(i)]
  \item $h$ is locally Lipschitz continuous on $D$
  \item $\nu_w(o) - (P_w \nu_w)(o) = H(w, o) - h(w)$ holds for all $w \in D, o \in \R^L$
  \item for all compact subsets $Q$ of D, there exist constants $C_2, C_3, q_2, q_3$ (depending on $Q$), such that for all $w, w' \in Q, z \in \R^L$,
  \begin{align}
       \norm{\nu_w(o)} &< C_2 (1 + \norm{o}^{q_2}), \\
       \norm{(P_w \nu_w)(o) - (P_{w'}\nu_{w'})(o)} &\leq C_3 \norm{w - w'}(1 + \norm{o}^{q_3}).
  \end{align}
\end{enumerate}
\end{assumption}
\begin{assumption}
\label{assu aasa bounds expectation}
For any compact subset $Q$ of $D$ and any $q > 0$, 
there exists constant $C_4$ (depending on $Q, q$) such that for all $t$, $o \in \R^L, w \in \R^K$,
\begin{align}
\E\left[\mathbb{I}(\qty{w_k \in Q, k \leq t})(1 + \norm{O_{t+1}}^q) \mid O_0 = o, w_0 = w\right] \leq  C_4 (1 + \norm{o}^q),
\end{align}
where $\mathbb{I}$ is the indicator function.
\end{assumption}

\begin{assumption}
\label{assu aasa lyapunov}
There exist a function $U \in \mathcal{C}^2(\R^K)$ and $w_* \in D$ such that
\begin{enumerate}[(i)]
  \item $U(w) \to C \leq +\infty$ if $w \to \partial D$ or $\norm{w} \to \infty$ 
  \item $U(w) < C$ for all $w \in D$
  \item $U(w) \geq 0$, where the equality holds i.f.f. $w = w_*$
  \item $\indot{\dv{U(w)}{w}}{h(w)} \leq 0$ for all $w \in D$, where the equality holds i.f.f. $w = w_*$.
\end{enumerate}
\end{assumption}
\begin{theorem}
\label{thm aasa}
(Theorem 13 of \citet{DBLP:books/sp/BenvenisteMP90} (p. 236))
Let Assumptions \ref{assu aasa lr}~-~\ref{assu aasa lyapunov} hold.
For any compact $Q \subset D$, there exist constants $C_0, q_0$ such that for all $w \in Q, o \in \R^L$,
the iterates $\qty{w_t}$ generated by \eqref{eq aasa iterates} satisfy
\begin{align}
\Pr( \lim_{t \to \infty} w_{t} = w_* \mid O_0 = o, w_0 = w) \geq 1 - C_0(1 + \norm{o}^{q_0}) \sum_{t=0}^\infty \alpha_t^2.
\end{align}
\end{theorem}

\subsection{Theorem 13 of \citet{DBLP:books/sp/BenvenisteMP90} with a Finite Chain}
Consider the iterates $\qty{w_t}$ evolving in $\R^K$ defined as 
\begin{align}
\label{eq aasa finite iterates}
w_{t+1} \doteq w_t + \alpha_t \bar H(w_t, O_{t+1}),
\end{align}
where $\qty{O_t}$ are random variables evolving in a \emph{finite} space $\fO$, $\qty{\alpha_t}$ is a sequence of learning rates, $\bar H$ is a function from $\R^K \times \fO$ to $\R^K$.
Without loss of generality, let $\fO \doteq \qty{1, 2, \dots, N} \subset \R$.
We make the following assumptions.
\begin{assumption}
\label{assu aasa finite lr}
$\qty{\alpha_t}$ is a deterministic, positive, nonincreasing sequence such that 
\begin{align}
\sum_t \alpha_t = \infty, \sum_t \alpha_t^2 < \infty.
\end{align}
\end{assumption}
\begin{assumption}
\label{assu aasa finite markov}
There exists a family $\qty{\bar P_w \in \R^{N \times N} \mid w \in \R^K}$ of parameterized transition matrices such that 
the random variables $\qty{O_t}$ evolve according to
\begin{align}
O_{t+1} \sim \bar P_{w_t}(O_t, \cdot)
\end{align}
Let $\Lambda_w$ be the closure of $\qty{\bar P_w \mid w \in \R^K}$, for any $P \in \Lambda_w$, the Markov chain in $\fO$ induced by the transition matrix $P$ is ergodic.
We use $d_P$ to denote the invariant distribution of the chain induced by $P$.
In particular, $d_w$ denotes the invariant distribution of the chain induced by $\bar P_w$.
We define
\begin{align}
h(w) \doteq \sum_{o \in \fO} d_w(o) \bar H(w, o).
\end{align}
\end{assumption}
\begin{assumption}
\label{assu aasa finite continuity}
$\bar P_w$ is Lipschitz continuous in $w$. 
For any compact $Q \subset \R^K$ and any $o \in \fO$, $\bar{H}(w, o)$ is Lipschitz continuous in $w$ on $Q$.
\end{assumption}
\begin{assumption}
\label{assu aasa finite lyapunov}
There exist function $U \in \mathcal{C}^2(\R^K)$ and $w_* \in \R^K$ such that
\begin{enumerate}[(i)]
  \item $U(w) \to \infty$ when $\norm{w} \to \infty$ 
  \item $U(w) < \infty$ for all $w \in \R^K$
  \item $U(w) \geq 0$, where the equality holds i.f.f. $w = w_*$
  \item $\indot{\dv{U(w)}{w}}{h(w)} \leq 0$ for all $w \in \R^K$, where the equality holds i.f.f. $w = w_*$.
\end{enumerate}
\end{assumption}
\begin{corollary}
\label{cor aasa finite}
Under Assumptions \ref{assu aasa finite lr} - \ref{assu aasa finite lyapunov},
for any compact set $Q \subset \R^K$,
there exists constants $C_0$ (depending on $Q$) such that for all $w \in Q, o \in \fO$,
the iterates $\qty{w_t}$ generated by \eqref{eq aasa finite iterates} satisfy
\begin{align}
\Pr(\lim_{t\to\infty} w_t = w_* \mid O_0 = o, w_0 = w) \geq 1 - C_0 \sum_{t=0}^\infty \alpha_t^2.
\end{align}
\end{corollary}

\begin{proof}
We proceed by expressing \eqref{eq aasa finite iterates} in the form of \eqref{eq aasa iterates} and invoking Theorem \ref{thm aasa}.
Let
\begin{align}
H(w, o) &\doteq \begin{cases}
\bar H(w, o) & o \in \fO \\
h(w) & o \notin \fO
\end{cases}.
\end{align}
Then \eqref{eq aasa finite iterates} can be rewritten as
\begin{align}
w_{t+1} \doteq w_t + \alpha_t H(w_t, O_{t+1}),
\end{align}
which has the same form as \eqref{eq aasa iterates}.
Here the $L$ in $\R^L$ is 1 and we consider $D$ to be $\R^K$.

Assumption~\ref{assu aasa lr} is identical to Assumption~\ref{assu aasa finite lr}.
Assumption~\ref{assu aasa lyapunov} is implied by Assumption~\ref{assu aasa finite lyapunov} via considering $C = \infty$.

To verify Assumption \ref{assu aasa markov},
let
\begin{align}
P_w(o, B) &\doteq \begin{cases}
\sum_{o'} \delta_{o'}(B) \bar P_w(o, o') & o \in \fO \\
\mathcal{N}(B) &  o \notin \fO
\end{cases},
\end{align}
where $\delta_{o'}(B)$ is the Dirac measure,
$\mathcal{N}(\cdot)$ denotes the normal distribution (we can use any well-defined distribution on $\R$ here).
Then Assumption \ref{assu aasa markov} follows from Assumption \ref{assu aasa finite markov}.

We now verify Assumption \ref{assu aasa bounds}.
From Assumption \ref{assu aasa finite continuity} 
and the finiteness of $\fO$,
for any compact $Q$, $\bar H(w, o)$ is bounded on $Q$.
So $h(w)$ is also bounded on $Q$.
Then the boundedness of $H(w, o)$ on $Q$ follows immediately.

We now verify Assumption \ref{assu aasa possion}(i).
First, for any compact $Q \subset \R^K$,
$\bar H(w, o)$ is Lipschitz continuous in $w$ and bounded on $Q$.
$d_w(o)$ is apparently bounded.
By Assumption \ref{assu aasa finite markov}, 
for any $P \in \Lambda_w$,
the chain induced by $P$ is ergodic.
It can be easily proved (see, e.g., Lemma 9 of \citet{zhang2021breaking}) that $d_w$ is also Lipschitz continuous in $w$.
The Lipschitz continuity of $h(w)$ on $Q$ then follows immediately
from the fact that the product of two bounded Lipschitz functions are still bounded and Lipschitz.
Since we are free to choose any $Q$,
$h(w)$ is locally Lipschitz continuous in $\R^K$.

We verify Assumption \ref{assu aasa possion}(ii) by constructing auxiliary Markov Reward Processes (MRPs) and using standard properties of MRPs.  
To construct the $i$-th MRP $(i=1, \dots, K)$,
let $H_{w, i}$ denote a vector in $\R^N$ whose $i$-th element is $H_i(w, o)$, the $i$-th element of $H(w, o)$.
For any $w \in \R^K, o \in \fO$,
we define a vector $\bar \nu_w(o)$ in $\R^K$ by defining its $i$-th element $\bar \nu_{w, i}(o)$ as
\begin{align}
\bar \nu_{w, i}(o) \doteq \E\left[\sum_{k = 0}^\infty [H_{w, i}(O_k) - h_i(w)] \mid O_0 = o, O_{k+1} \sim P_w(O_{k}, \cdot)\right],
\end{align}
where $h_i(w)$ is the $i$-the element of $h(w)$.
By definition,
$\bar \nu_{w, i}$ is the bias vector of the MRP induced by $\bar P_w$ in $\fO$ with the reward vector being $H_{w, i}$.
Since $\bar P_w$ induces an ergodic chain under Assumption \ref{assu aasa finite markov},
$\bar \nu_w$ is always well defined.
Moreover, $h_i(w)$ is the gain of this MRP.
It follows from Chapter 8.2.1 of \citet{puterman2014markov} that for any $w \in \R^K$ and $o \in \fO$,
\begin{align}
\label{eq proof tmp eq1}
\bar \nu_{w, i}(o) &= H_{w, i}(o) - h_i(w) + \sum_{o'} \bar \nu_{w, i}(o') \bar P_w(o, o') \\
\bar \nu_{w, i} &= H_{\bar P_w} H_{w, i},
\end{align}
where $H_P \doteq (I - P + 1 d_P^\top)^{-1}(I - 1 d_P^\top)$ is the fundamental matrix of the chain induced by a transition matrix $P$.
Define 
\begin{align}
\nu_w(o) &\doteq \begin{cases}
\bar \nu_w(o) & o \in \fO \\
0 & o \notin \fO
\end{cases}.
\end{align}
It is then easy to verify that 
for $o \in \fO$,
\begin{align}
(P_w \nu_w)(o) = \int \nu_w(y) P_w(o, dy) = 
\int \nu_w(y) \sum_{o'} \delta_{o'}(dy) \bar P_w(o, o') = \sum_{o'} \nu_w(o') \bar P_w(o, o').
\end{align}
For $o \notin \fO$,
$(P_w \nu_w)(o) = 0$.
For $o \in \fO$, 
Assumption \ref{assu aasa possion}(ii) holds since it is just \eqref{eq proof tmp eq1}.
For $o \notin \fO$, Assumption \ref{assu aasa possion}(ii) holds as well since since both LHS and RHS are 0.

We now verify Assumption \ref{assu aasa possion}(iii).
Since $d_P$ is Lipschitz continuous in $P$ for all $P \in \Lambda_w$
and $\Lambda_w$ is compact,
we have $\sup_{P \in \Lambda_w} \norm{(I - P + 1 d_P^\top)^{-1}} < \infty$ by the extreme value theorem.
Using the ergodicity of the chain induced by $P$ and
\begin{align}
  \norm{X^{-1} - Y^{-1}} = \norm{X^{-1}YY^{-1} - X^{-1}XY^{-1}} \leq \norm{X^{-1}}\norm{X-Y}\norm{Y^{-1}},
\end{align}
it is then easy to see that $H_{\bar P_w}$ is bounded and Lipschitz continuous in $w$.
Consequently,
for any compact $Q \subset \R^K, w \in Q, w' \in Q$,
$\nu_w$ is bounded on $Q$ and
\begin{align}
\norm{\bar \nu_{w, i} - \bar \nu_{w', i}} &\leq \norm{H_{\bar P_w} - H_{\bar P_{w'}}}\norm{H_{w, i}} + \norm{H_{\bar P_{w'}}} \norm{H_{w, i} - H_{w', i}} \\
&\leq C_1 \norm{w - w'} + C_2 \norm{w - w'},
\end{align}
where the constant $C_1$ results from the Lipschitz continuity of $H_{\bar P_w}$ and the boundedness of $H(w, o)$ on $Q$,
the constant $C_2$ results from the Lipschitz continuity of $H(w, o)$ on $Q$ and the boundedness of $H_{\bar P_w}$.
Since by Assumption \ref{assu aasa finite continuity},
$\bar P_w$ in Lipschitz continuous in $w$,
the Lipschitz continuity of $P_w \nu_w$ in $Q$ follows immediately,
which completes the verification of
Assumption \ref{assu aasa possion}(iii).

Assumption \ref{assu aasa bounds expectation} is trivial since $\fO$ is finite, which completes the proof.
\end{proof}
\newpage

\section{Proofs}
\subsection{Proofs of Lemmas \ref{lem emphasis expression}}
\label{sec proof of empahsis expression lemma}

\begin{proof}
Let $\tau_j \doteq (s_j, a_j, s_{j-1}, a_{j-1}, \dots s_1, a_1)$, $\Gamma_j \doteq (S_{t-j}, A_{t-j}, \dots, S_{t-1}, A_{t-1})$,
\begin{align}
\label{eq proof tmp mtn}
m_{t, n}(s) &\doteq \E[ F_t^n | S_t = s]  \\
&= \sum_{j=0}^n \gamma^j \E \left[ \rho_{t-j:t-1} i_{t-j} \right | S_t = s] \\
&= \sum_{j=0}^n \gamma^j \sum_{\tau_j \in (\fS \times \fA)^j} \Pr(\Gamma_j = \tau_j | S_t = s) \E[\rho_{t-j:t-1} i_{t-j} | \Gamma_j = \tau_j, S_t = s] \\
\intertext{\hfill (Law of total expectation)}
&= \sum_{j=0}^n \gamma^j \sum_{\tau_j \in (\fS \times \fA)^j} \frac{\Pr(\Gamma_j = \tau_j, S_t = s)}{\Pr(S_t = s)} \E[\rho_{t-j:t-1} i_{t-j} | \Gamma_j = \tau_j, S_t = s] \\
\intertext{\hfill (Bayes' rule)}
&= \sum_{j=0}^n \gamma^j \sum_{\tau_j \in (\fS \times \fA)^j} \frac{\Pr(\Gamma_j = \tau_j, S_t = s)}{\Pr(S_t = s)} i(s_j) \rho(s_{j}, a_{j}) \cdots \rho(s_{1}, a_{1}) \\
&= \sum_{j=0}^n \gamma^j \sum_{\tau_j \in (\fS \times \fA)^j} \frac{\Pr(S_{t-j} = s_j) P_\pi(s_{j}, s_{j-1}) \cdots P_\pi(s_{2}, s_{1}) P_\pi(s_{1}, s)}{\Pr(S_t = s)} i(s_j) \\
&= \sum_{j=0}^n \gamma^j \sum_{s_j} \frac{\Pr(S_{t-j} = s_j) P_\pi^j(s_j, s)}{\Pr(S_t = s)} i(s_j)
\end{align}
Assumption \ref{assu ergodic} implies 
\begin{align}
\lim_{t\to\infty} \Pr(S_t = s) = d_\mu(s).
\end{align}
Consequently,
\begin{align}
m_n(s) &= \lim_{t\to \infty} m_{t, n}(s) \\
&= \sum_{j=0}^n  \gamma^j \sum_{s_j \in \fS} \frac{d_\mu(s_j) P_\pi^j (s_j, s)}{d_\mu(s)} i(s_j).
\end{align}
In a matrix form,
\begin{align}
m_n &= \sum_{j=0}^n \gamma^j D_\mu^{-1}(P_\pi^\top)^j D_\mu i, \\
m &= \lim_{n \to \infty} m_n = D_\mu^{-1}(I - \gamma P_\pi^\top)^{-1} D_\mu i.
\end{align}
\end{proof}

\subsection{Proof of Lemma \ref{lem emphasis bound}}
\label{sec proof of emphasis bound lemma}
\begin{proof}
\begin{align}
m_{n} - m &= \sum_{j=n+1}^\infty \gamma^j D_\mu^{-1}(P_\pi^\top)^j D_\mu i \\
&= D_\mu^{-1} \left(\sum_{j=n+1}^\infty \gamma^j (P_\pi^\top)^j \right) D_\mu i \\
&= D_\mu^{-1} \gamma^{n+1} (P_\pi^\top)^{n+1} (I - \gamma P_\pi^\top)^{-1} D_\mu i \\
&=  \gamma^{n+1} D_\mu^{-1} (P_\pi^\top)^{n+1} D_\mu m ,
\end{align}
implying
\begin{align}
\norm{m_n - m}_1 &\leq \gamma^{n+1} \norm{D_\mu^{-1}}_1 \norm{D_\mu}_1 \norm{m}_1 \qq{(Using $\norm{P_\pi^\top}_1 = \norm{P_\pi}_\infty = 1$)}\\
&= \gamma^{n+1} \frac{d_{\mu, max}}{d_{\mu, min}} \norm{m}_1, \\
\norm{f_n - f}_\infty &\leq \norm{D_\mu}_\infty \norm{m_n - m}_\infty  \leq d_{\mu, max} \norm{m_n - m}_1 \leq d_{\mu, max} \norm{m_n - m}_\infty \\
&= \gamma^{n+1} \frac{d_{\mu, max}^2}{d_{\mu, min}} \norm{m}_1
\end{align}
 
\end{proof}

\subsection{Proof of Lemma \ref{lem nd}}
\label{sec proof of nd}
\begin{proof}
Let $D_{f_n}$ be a diagonal matrix whose diagonal entry is $f_n$
and $D_{f_n - f}$ be a diagonal matrix whose diagonal entry is $f_n - f$.
We have 
\begin{align}
y^\top D_{f_n}(\gamma P_\pi - I) y &= y^\top D_{f_n - f}(\gamma P_\pi - I)y + y^\top D_f (\gamma P_\pi - I) y \\
&\leq \norm{y}^2 \norm{D_{f_n - f}} \norm{\gamma P_\pi - I} +  y^\top D_f (\gamma P_\pi - I) y \\
&\leq \norm{y}^2 \norm{f_n - f}_\infty \norm{\gamma P_\pi - I} +  y^\top D_f (\gamma P_\pi - I) y \\
\end{align}
Similarly,
\begin{align}
y^\top (\gamma P_\pi^\top - I) D_{f_n} y \leq \norm{y}^2 \norm{f_n - f}_\infty \norm{\gamma P_\pi^\top - I} +  y^\top (\gamma P_\pi^\top - I) D_f y.
\end{align}
Combining the above two inequalities together,
we get
\begin{align}
&\frac{1}{2} y^\top \left( D_{f_n}(\gamma P_\pi - I)+ (\gamma P_\pi^\top - I) D_{f_n} \right) y \\
\leq& \norm{f_n - f}_\infty \norm{\gamma P_\pi - I} \norm{y}^2 + \frac{1}{2} y^\top \left( D_f (\gamma P_\pi - I) + (\gamma P_\pi^\top - I) D_f \right)y \\
\intertext{\hfill (Invariance of $\ell_2$ norm under transpose)}
\leq&(\norm{f_n - f}_\infty \norm{\gamma P_\pi - I} -\lambda_{min}) \norm{y}^2 \\
\intertext{\hfill (Eigendecomposition of real symmetric matrices)}
\leq&(\gamma^{n+1} \frac{d_{\mu, max}^2}{d_{\mu, min}} \norm{m}_1 \norm{\gamma P_\pi - I} -\lambda_{min}) \norm{y}^2
\qq{(Lemma \ref{lem emphasis bound})}
\end{align}
As long as the condition \eqref{eq condition of n nd} holds,
the above inequality asserts that $$D_{f_n} (\gamma P_\pi - I) + (\gamma P_\pi^\top - I)D_{f_n}$$ is n.d.,
implying $D_{f_n}(\gamma P_\pi - I)$ is n.d..
This together with Assumption \ref{assu full rank} completes the proof.
\end{proof}

\subsection{Proof of Theorem \ref{thm etd convergence}}
\label{sec proof of etd convergence}
\begin{proof}
Let $O_t \doteq (S_{t-n}, A_{t-n}, \dots, S_t, A_t, S_{t+1})$ be a sequence of random variables generated by Algorithm~\ref{alg etd}.
Let $o_t \doteq (s_{t-n}, a_{t-n}, \dots, s_t, a_t, s_{t+1})$ and define
functions
\begin{align}
A(o_t) &\doteq \left(\sum_{j=0}^n \gamma^j \left(\prod_{k=t-j}^{t-1} \frac{\pi(a_k | s_k)}{\mu(a_k | s_k)} \right) i(s_{t-j}) \right) \frac{\pi(a_t | s_t)}{\mu(a_t | s_t)} x(s_t) \left(\gamma x(s_{t+1})^\top - x(s_t)^\top \right), \\
b(o_t) &\doteq \left(\sum_{j=0}^n \gamma^j \left(\prod_{k=t-j}^{t-1} \frac{\pi(a_k | s_k)}{\mu(a_k | s_k)} \right) i(s_{t-j}) \right) \frac{\pi(a_t | s_t)}{\mu(a_t | s_t)} x(s_t) r(s_t, a_t).
\end{align}
Here $o_t$ is just placeholder for defining $A(o_t)$ and $b(o_t)$.
Then the update for $\qty{w_t}$ in Algorithm \ref{alg etd} can be expressed as
\begin{align}
w_{t+1} = w_t + \alpha_t \left(A(O_t)w_t + b(O_t)\right).
\end{align}
We now proceed to confirming its convergence via verifying 
Assumptions \ref{assu ndp lr} - \ref{assu ndp mixing} thus invoking Theorem \ref{thm ndp}.

Assumption \ref{assu ndp lr} is identical to Assumption \ref{assu lr}.
Assumption \ref{assu ndp ergodic} follows directly from Assumption \ref{assu ergodic}.
And it is easy to see the invariant distribution of $\qty{O_t}$ is 
\begin{align}
\label{eq proof tmp doot}
d_{\fO}(o_t) = d_\mu(s_{t-n}) \mu(a_{t-n} | s_{t-n}) p(s_{t-n+1}|s_{t-n}, a_{t-n}) \cdots p(s_{t+1}|s_t, a_t).
\end{align}
Moreover,
\begin{align}
\bar A \doteq &\E_{O_t \sim d_\fO}\left[A(O_t)\right] \\
=& \E_{O_t \sim d_\fO}\left[F_{t, n} \rho_t x_t (\gamma x_{t+1}^\top - x_t^\top) \right] \\
=& \sum_{s, a, s'} d_\mu(s) \mu(a|s) p(s'|s, a) \E\left[F_{t,n}\rho_t x_t (\gamma x_{t+1}^\top - x_t^\top)\right | S_t = s, A_t = a, S_{t+1} = s'] \\
\intertext{\hfill (Law of total expectation)}
=& \sum_{s, a, s'} d_\mu(s) \pi(a|s) p(s' | s, a) \E\left[F_{t, n} | S_t = s \right] x(s) (\gamma x(s')^\top - x(s)^\top) \\
\intertext{\hfill (Conditional independence and Markov property)}
=& \sum_{s, a, s'} d_\mu(s) \pi(a|s) p(s' | s, a) m_n(s) x(s) (\gamma x(s')^\top - x(s)^\top) 
\qq{\hfill (Using \eqref{eq proof tmp mtn} and \eqref{eq proof tmp doot})} \\
=& \sum_{s, a, s'} f_n(s) \pi(a|s) p(s' | s, a) x(s) (\gamma x(s')^\top - x(s)^\top) \qq{(Definition of $f_n$)} \\
=& X^\top D_{f_n} (\gamma P_\pi - I) X.
\end{align}
In the above equation,
we have abused the notation slightly to use $O_t$ to denote random variables sampled from $d_\fO$.
Similarly, it can be shown that 
\begin{align}
\bar b \doteq \E_{O_t \sim d_\fO}\left[b(O_t)\right] = X^\top D_{f_n} r_\pi.
\end{align}
Lemma \ref{lem nd} confirms that $\bar A$ is n.d., 
verifying Assumption \ref{assu ndp nd}.
Assumption \ref{assu ndp bounded} is obvious since $\ns, \na, n$ are all finite.
Assumption \ref{assu ndp mixing} follows immediately from the geometrically mixing rate of ergodic Markov chain.
For example,
\begin{align}
  \norm{\E[A(O_t)] - \bar A} &= \norm{\sum_{o \in \fO} \left(\Pr(O_t = o) - d_\fO(o) \right)A(o)} \\
  &\leq \max_o \norm{A(o)} \sum_{o} \left|\Pr(O_t = o) - d_\fO(o) \right| \\
  &\leq C_0 \kappa^t,
\end{align}
for some $C_0 > 0$ and $\kappa \in (0, 1)$.
Here the last inequality is a standard result, see,
e.g., Theorem 4.9 of \citet{levin2017markov}.

Note this procedure cannot be used to verify the convergence of the original ETD(0),
where we would need to consider $O_t = (F_t, S_t, A_t)$.
Since $F_t$ involves in $\R$, Assumption \ref{assu ndp bounded} cannot be verified.
\end{proof}

\subsection{Proof of Lemma \ref{lem contraction}}
\label{sec proof of contraction}
\begin{proof}
Since $i(s) > 0$ holds for any $s$
and $P_\pi$ is nonnegative,
from Lemma \ref{lem emphasis expression} it is easy to see for any $n_1 > n_2$,
\begin{align}
  m_{n_1}(s) > m_{n_2}(s)
\end{align}
always holds.
Then by the definition of $f_n$,
\begin{align}
  f_{n_1}(s) > f_{n_2}(s)
\end{align}
holds as well.
In particular, for any $n \geq 1$, 
\begin{align}
\label{eq monotonic f}
f(s) > f_n(s) > f_0(s) = d_\mu(s)i(s) > 0.
\end{align}
For any $v$, we have
\begin{align}
\gamma \norm{P_\pi v}_{{f_n}}^2 &= \gamma \sum_s f_n(s) \left( \sum_{s'} P_\pi(s, s') v(s') \right)^2 \\
&\leq \gamma\sum_s f_n(s) \sum_{s'} P_\pi(s, s') v^2(s') \qq{(Jensen's inequality)} \\
&=  \gamma\sum_{s'} v^2(s') \sum_s f_n(s) P_\pi(s, s') \\
&= v^\top diag(\gamma P_\pi^\top f_n) v \\
&= v^\top diag\left(f_n - (I - \gamma P_\pi^\top) f_n\right) v \\
&= v^\top diag\left(f_n - (I - \gamma P_\pi^\top) f + (I - \gamma P_\pi^\top)(f - f_n)\right) v \\
&= \norm{v}_{f_n}^2 - v^\top diag\left((I - \gamma P_\pi^\top)f\right) v + v^\top diag\left((I - \gamma P_\pi^\top)(f - f_n)\right) v \\
&= \norm{v}_{f_n}^2 - \norm{v}^2_{f_0} + v^\top diag\left((I - \gamma P_\pi^\top)(f - f_n)\right) v \\
\intertext{\hfill (Using $(I - \gamma P_\pi^\top)f = (I - \gamma P_\pi^\top) (I - \gamma P_\pi^\top)^{-1} D_\mu i = f_0$)}
&\leq \norm{v}_{f_n}^2 - \norm{v}^2_{f_0} + \norm{(I - \gamma P_\pi^\top)(f - f_n)}_\infty \norm{v}^2 \\
\intertext{\hfill (Property of $\ell_2$ norm of a diagonal matrix)}
&\leq \norm{v}_{f_n}^2 - \norm{v}^2_{f_0} + \norm{(I - \gamma P_\pi^\top)}_\infty \gamma^{n+1} \frac{d_{\mu, max}^2}{d_{\mu, min}} \norm{m}_1 \norm{v}^2 \qq{(Lemma \ref{lem emphasis bound})} \\
&\leq \norm{v}_{f_n}^2 - \norm{v}^2_{f_0} + \kappa \min_s i(s) d_\mu(s) \norm{v}^2 \qq{(Using \eqref{eq condition of n contraction})} \\
&\leq \norm{v}_{f_n}^2 - \norm{v}^2_{f_0} + \kappa \norm{v}^2_{f_n} \qq{(Using $\min_{s'} i(s')d_\mu(s') \leq f_0(s) < f_n(s)$ )}\\
&= (1 + \kappa) \norm{v}_{f_n}^2 - \norm{v}^2_{f_0}  \\
&= (1 + \kappa) \norm{v}_{f_n}^2 - \sum_s v(s)^2 d_\mu(s)i(s) \\
&= (1 + \kappa) \norm{v}_{f_n}^2 - \sum_s v(s)^2 f(s) \frac{d_\mu(s)i(s)}{f(s)} \\
&\leq (1 + \kappa) \norm{v}_{f_n}^2 - \kappa \sum_s v(s)^2 f(s) \qq{(Definition of $\kappa$ and $f(s) > 0$)} \\
&\leq (1 + \kappa) \norm{v}_{f_n}^2 - \kappa \sum_s v(s)^2 f_n(s) \qq{(Using $f(s) > f_n(s)$)} \\
&= \norm{v}_{f_n}^2
\end{align}
Consequently,
\begin{align}
\norm{\bop_\pi v_1 - \bop_\pi v_2}_{f_n}^2 = \gamma^2 \norm{P_\pi (v_1 - v_2)}_{f_n}^2 \leq \gamma \norm{v_1 - v_2}_{f_n}^2,
\end{align}
implying that $\bop_\pi$ is a $\sqrt{\gamma}$-contraction in $\norm{\cdot}_{f_n}$.
Since $\Pi_{f_n}$ is nonexpansive in $\norm{\cdot}_{f_n}$,
it is easy to see that $\Pi_{f_n} \bop_\pi$ is a $\sqrt{\gamma}$ contraction in $\norm{\cdot}_{f_n}$ as well.

Further,
let $D_{f_n} \doteq diag(f_n)$,
it is easy to verify that
\begin{align}
  \Pi_{f_n} = X (X^\top D_{f_n} X)^{-1} X^\top D_{f_n}.
\end{align}
Consequently, 
\begin{align}
\label{eq proof tmp projection and ab}
  A_n w_{*,n} &= b_n \\
\iff X^\top D_{f_n} (\gamma P_\pi - I)X w_{*, n} &= X^\top D_{f_n} r_\pi \\
\iff X^\top D_{f_n}(r_\pi + \gamma P_\pi Xw_{*, n}) &= X^\top D_{f_n} X w_{*, n} \\
\iff X(X^\top D_{f_n} X)^{-1} X^\top D_{f_n}(r_\pi + \gamma P_\pi Xw_{*,n}) &= X(X^\top D_{f_n} X)^{-1} X^\top D_{f_n} X w_{*, n} \\
\iff \Pi_{f_n}\bop_\pi (Xw_{*, n}) &= Xw_{*, n}.
\end{align}
Then,
\begin{align}
  \norm{Xw_{*, n} - v_\pi}_{f_n}^2 &= \norm{ Xw_{*, n}  - \Pi_{f_n}  v_\pi}_{f_n}^2 + \norm{\Pi_{f_n} v_\pi - v_\pi}_{f_n}^2 \\
  \intertext{\hfill (Pythagorean theorem)}
  &= \norm{ \Pi_{f_n} \bop_\pi (Xw_{*, n})  - \Pi_{f_n} \bop_\pi v_\pi}_{f_n}^2 + \norm{\Pi_{f_n} v_\pi - v_\pi}_{f_n}^2 \\
  &\leq \gamma \norm{Xw_{*, n} - v_\pi}_{f_n}^2 + \norm{\Pi_{f_n} v_\pi - v_\pi}_{f_n}^2.
\end{align}
Rearranging terms completes the proof.
\end{proof}

\subsection{Proof of Lemma \ref{thm fixed point}}
\label{sec proof of fixed point}
\begin{proof}
If \eqref{eq condition of n avi} holds,
then Lemma \ref{lem contraction} implies that for any $u \in \Lambda_M$ and $\pi \in \Lambda_\pi$,
$\pbop{\mu}{\pi}$ is a $\sqrt{\gamma}$-contraction in $\norm{\cdot}_{f_{n, \mu, \pi}}$.
We use $Xw_{n, \mu, \pi}$ to denote its unique fixed point.
Lemma~\ref{lem nd} ensures that $X^\top D_{f_{n, \mu, \pi}} (I - \gamma P_\pi)X$ is p.d..
Similar to \eqref{eq proof tmp projection and ab},
it is easy to verify that
\begin{align}
w_{n, \mu, \pi} = (X^\top D_{f_{n, \mu, \pi}} (I - \gamma P_\pi)X)^{-1} X^\top D_{f_{n, \mu, \pi}} r_\pi,
\end{align}
from which it is easy to see
$w_{n, \mu, \pi}$ is continuous in $\mu$ and $\pi$
since the invariant distribution $d_\mu$ is continuous in $\mu$.

Similar to \citet{de2000existence},
we first define several helper functions.
For any policy $\mu \in \Lambda_M$, $\pi \in \Lambda_\Pi$, and $\eta > 0$,
let
\begin{align}
g_{\mu, \pi}(w) &\doteq X^\top D_{f_{n, \mu, \pi}} (\bop_{\pi} Xw - Xw) \\
&= X^\top D_{f_{n, \mu, \pi}} X (X^\top D_{f_{n, \mu, \pi}} X)^{-1}X^\top D_{f_{n, \mu, \pi}} (\bop_{\pi} Xw - Xw) \\
&= X^\top D_{f_{n, \mu, \pi}} (\pbop{\mu}{\pi}Xw - Xw), \\
g(w) &\doteq X^\top D_{f_{n, \mu_w, \pi_w}} (\bop_{\pi_w} Xw - Xw)\\
&= X^\top D_{f_{n, \mu_w, \pi_w}} (\pbop{\mu_w}{\pi_w} Xw - Xw),\\
z_{\mu, \pi}^\eta(w) &\doteq w + \eta g_{\mu, \pi}(w), \\
z^\eta(w) &\doteq w + \eta g(w).
\end{align}
We have
\begin{align}
z^\eta_{\mu, \pi}(w) &= w \\
\iff g_{\mu, \pi}(w) & = 0\\
\iff X^\top D_{f_{n, \mu, \pi}} \bop_\pi Xw &= X^\top D_{f_{n, \mu, \pi}} Xw \\
\iff X(X^\top D_{f_{n, \mu, \pi}} X)^{-1} X^\top D_{f_{n, \mu, \pi}} \bop_\pi Xw &= X(X^\top D_{f_{n, \mu, \pi}} X)^{-1} X^\top D_{f_{n, \mu, \pi}} Xw \\
\iff \pbop{\mu}{\pi}Xw &= Xw,
\end{align}
i.e., $Xw$ is a fixed point of $\pbop{\mu}{\pi}$ if and only if $w$ is a fixed point of $z_{\mu, \pi}^\eta$.
With the same procedure,
we can also show 
\begin{align}
  z^\eta(w) = w \iff \pbop{\mu_w}{\pi_w} (Xw) = Xw.
\end{align}
This suggests that to study the fixed points of emphatic approximate value iteration is to study the fixed points of $z^\eta$ with any $\eta > 0$.
To this end,
we first study $z_{\mu, \pi}^\eta$ with the following lemma,
which is analogous to Lemma 5.4 of \citet{de2000existence}.
\begin{lemma}
There exists an $\eta_0 > 0$ such that for all $\eta \in (0, \eta_0)$,
there exists a constant $\beta_\eta \in (0, 1)$ such that for all
$\mu \in \Lambda_M, \pi \in \Lambda_\Pi$, 
\begin{align}
\norm{z_{\mu, \pi}^\eta(w) - w_{n, \mu, \pi}} \leq \beta_\eta \norm{w - w_{n, \mu, \pi}}.
\end{align}
\end{lemma}
\begin{proof}
By the contraction property of $\pbop{\mu}{\pi}$,
\begin{align}
\norm{\pbop{\mu}{\pi} Xw - Xw_{n, \mu, \pi}}_{f_{n, \mu, \pi}} \leq \sqrt{\gamma} \norm{Xw - Xw_{n, \mu, \pi}}_{f_{n, \mu, \pi}}.
\end{align}
Consequently,
\begin{align}
&(w - w_{n, \mu, \pi})^\top g_{\mu, \pi}(w) \\
=&(X w - X w_{n, \mu, \pi})^\top D_{f_{n, \mu, \pi}}(\pbop{\mu}{\pi} Xw - Xw) \\
=&(X w - X w_{n, \mu, \pi})^\top D_{f_{n, \mu, \pi}}\left(\pbop{\mu}{\pi} Xw - Xw_{n, \mu, \pi} + Xw_{n, \mu, \pi} - Xw\right) \\
\leq &\norm{X w - X w_{n, \mu, \pi}}_{f_{n, \mu, \pi}} \norm{\pbop{\mu}{\pi} Xw - Xw_{n, \mu, \pi}}_{f_{n, \mu, \pi}} - \norm{Xw - Xw_{n, \mu, \pi}}_{f_{n, \mu, \pi}}^2 \\
\intertext{\hfill (Cauchy-Schwarz inequality)}
\leq &(\sqrt{\gamma} - 1) \norm{Xw - Xw_{n, \mu, \pi}}_{f_{n, \mu, \pi}}^2 \qq{(Property of contraction)} \\
=& (\sqrt{\gamma} - 1) (w - w_{n, \mu, \pi})^\top (X^\top D_{f_{n, \mu, \pi}} X) (w - w_{n, \mu, \pi}).
\end{align}
Since $X^\top D_{f_{n, \mu, \pi}}X$ is symmetric and p.d.,
eigenvalues are continuous in the elements of the matrix,
$\Lambda_M$ and $\Lambda_\Pi$ are compact,
by the extreme value theorem,
there exists a constant $C_1 > 0$ 
(the infimum over the smallest eigenvalues of all $X^\top D_{f_{n, \mu, \pi}}X$), independent of $\mu$ and $\pi$,
such that for all $y$, 
\begin{align}
y^\top X^\top D_{f_{n, \mu, \pi}}X y \geq C_1 \norm{y}^2.
\end{align}
Consequently,
\begin{align}
  \label{eq proof tmp gmupi1}
(w - w_{n, \mu, \pi})^\top g_{\mu, \pi}(w) \leq (\sqrt{\gamma} - 1) C_1 \norm{w - w_{n, \mu, \pi}}^2.
\end{align}
Moreover,
let $x_i$ be the $i$-th column $X$, we have
\begin{align}
\norm{g_{\mu, \pi}(w)}^2 &= \sum_{i=1}^K \left(x_i^\top D_{f_{n, \mu, \pi}}(\pbop{\mu}{\pi}Xw - Xw) \right)^2 \\
&\leq \sum_{i=1}^K \norm{x_i}_{f_{n, \mu, \pi}}^2 \norm{\pbop{\mu}{\pi}Xw - Xw}^2_{f_{n, \mu, \pi}} \qq{(Cauchy-Schwarz inequality)} \\
&\leq \sum_{i=1}^K \norm{x_i}_{f_{n, \mu, \pi}}^2 \left( \norm{\pbop{\mu}{\pi}Xw - Xw_{n, \mu, \pi}}_{f_{n, \mu, \pi}} + \norm{Xw_{n, \mu, \pi} - Xw}_{f_{n, \mu, \pi}} \right)^2 \\
&\leq (\sqrt{\gamma} + 1)^2 \sum_{i=1}^K \norm{x_i}^2_{f_{n, \mu, \pi}}\norm{Xw_{n, \mu, \pi} - Xw}_{f_{n, \mu, \pi}}^2 \qq{($\sqrt{\gamma}$-contraction)}\\
&\leq (\sqrt{\gamma} + 1)^2 \left(\sum_{i=1}^K \norm{x_i}^2_{f_{n, \mu, \pi}}\right) \norm{X^\top D_{f_{n, \mu, \pi}} X} \norm{w - w_{n, \mu, \pi}}^2.
\end{align}
By the extreme value theorem,
\begin{align}
\sup_{\mu \in \Lambda_\mu, \pi \in \Lambda_\pi} \left(\sum_{i=1}^K \norm{x_i}^2_{f_{n, \mu, \pi}}\right) \norm{X^\top D_{f_{n, \mu, \pi}} X} < \infty.
\end{align}
Consequently,
there exists a constant $C_2 > 0$, independent of $\mu$ and $\pi$, such that
\begin{align}
  \label{eq proof tmp gmupi2}
\norm{g_{\mu, \pi}(w)}^2 \leq C_2 \norm{w - w_{n, \mu, \pi}}^2.
\end{align}
Combining \eqref{eq proof tmp gmupi1} and \eqref{eq proof tmp gmupi2} yields
\begin{align}
\norm{z^\eta_{\mu, \pi}(w) - w_{n, \mu, \pi}}^2 &= \norm{w + \eta g_{\mu, \pi}(w) - w_{n, \mu, \pi}}^2 \\
&= \norm{w - w_{n, \mu, \pi}}^2 + 2 \eta (w - w_{n, \mu, \pi})^\top g_{\mu, \pi}(w) + \eta^2 \norm{g_{\mu, \pi}(w)}^2 \\
&\leq \left(1 - 2\eta (1 - \sqrt{\gamma}\right) C_1 + \eta^2 C_2) \norm{w - w_{n, \mu, \pi}}^2
\end{align}
Then for all $\eta < \eta_0 \doteq 2C_1(1 - \sqrt{\gamma}) / C_2$,
we have $\beta_\eta \doteq \sqrt{1 - 2 \eta (1 - \sqrt{\gamma}) C_1 + \eta^2 C_2} < 1$.
\end{proof}
We are now ready to study $z^\eta$ with the previous lemma, analogously to Theorem 5.2 of \citet{de2000existence}.
Note $\fW \doteq \qty{w_{n, \mu, \pi} \mid \mu \in \Lambda_M, \pi \in \Lambda_\Pi}$ is a compact set
by the continuity of $w_{n, \mu, \pi}$ in $\mu$ and $\pi$.
Let $C_3 \doteq \sup_{w \in W} \norm{w}$ and take some $\eta$ in $(0, \eta_0)$,
we have for any $w$
\begin{align}
\norm{z^\eta(w)} &\leq \norm{z^\eta(w) - w_{n, \mu_w, \pi_w}} + \norm{ w_{n, \mu_w, \pi_w}} \\
\intertext{\hfill ($w_{n, \mu_w, \pi_w}$ denotes the fixed point of $\pbop{\mu}{\pi}$ with $\mu$ being $\mu_w$ and $\pi$ being $\pi_w$.)}
&= \norm{z^\eta_{\mu_w, \pi_w}(w) - w_{n, \mu_w, \pi_w}} + \norm{ w_{n, \mu_w, \pi_w}} \\
&\leq \beta_\eta \norm{w - w_{n, \mu_w, \pi_w}} + C_3 \\
&\leq \beta_\eta \norm{w} + (1 + \beta_\eta) C_3.
\end{align}
Since $\beta_\eta < 1$, 
we define
\begin{align}
\fW_2 \doteq \qty{w \in \R^K \mid \norm{w} < \frac{1+\beta_\eta}{1 - \beta_\eta}C_3}.
\end{align}
It is easy to verify that
\begin{align}
w \in \fW_2 \implies z^\eta(w) \in \fW_2.
\end{align}
The Brouwer fixed point theorem then asserts that $z^\eta(w)$ adopts at least one fixed point in $\fW_2$,
which completes the proof.
\end{proof}

\subsection{Proof of Lemma \ref{lem lipschitz}}
\label{sec proof of lipschitz lemma}
\begin{proof}
  Recall 
  \begin{align}
    A_w &= X^\top D_{f_{n, \mu_w, \pi_w}} (\gamma P_{\pi_w} - I)X, \\
    f_{n, \mu_w, \pi_w} &= \sum_{j=0}^n \gamma^j (P_{\pi_w}^\top)^j D_{\mu_w} i.
  \end{align}
  According to Lemma 9 of \citet{zhang2021breaking}, 
  the invariant distribution $d_\mu$ is Lipschitz continuous w.r.t. $\mu$ in $\Lambda_M$ under Assumption~\ref{assu closure ergodic}.
  Consequently, 
  Assumption~\ref{assu lipschitz continuous} implies that $D_{\mu_w}$ is Lipschitz continuous in $w$.
  It is then easy to see $f_{n, \mu_w, \pi_w}$ is Lipschitz continuous in $w$,
  using the fact that the product of two bounded Lipschitz continuous functions is still Lipschitz continuous.
  The Lipschitz continuity of $A_w$ then follows easily,
  so does that of $b_w$. 
\end{proof}

\subsection{Proof of Theorem \ref{thm esarsa convergence}}
\label{sec proof esarsa convergence}
\begin{proof}
Let $o_t \doteq (s_{t-n}, a_{t-n}, \dots, s_t, a_t, s_{t+1})$.
Define
\begin{align}
\delta_w(s, a, s') &\doteq r(s, a) + \gamma \sum_{a'} \pi_w(a' | s')x(s', a')^\top w - x(s, a)^\top w, \\
\bar H(w, o_t) &\doteq \left(\sum_{j=0}^n \gamma^j \left(\prod_{k=t-j+1}^{t} \frac{\pi_{w}(a_k | s_k)}{\mu_{w}(a_k | s_k)} \right) i(s_{t-j}, a_{t-j}) \right) \delta_{w}(s_t, a_t, s_{t+1}) x(s_t, a_t).
\end{align}
Note here $o_t$ is just a placeholder for defining the function $g$.
Let $O_t \doteq (S_{t-n}, A_{t-n}, \dots, S_t, A_t, S_{t+1})$ be a sequence of random variables generated by Algorithm \ref{alg esarsa}.
Then the update of $w$ in Algorithm \ref{alg esarsa} can be expressed as
\begin{align}
w_{t+1} = w_t + \alpha_t {\bar H(w_t, O_t)}.
\end{align}
We now prove Theorem \ref{thm esarsa convergence} by verifying Assumptions~\ref{assu aasa finite lr} - \ref{assu aasa finite lyapunov} thus invoking Corollary \ref{cor aasa finite}.
Assumption~\ref{assu aasa finite lr} is identical to Assumption~\ref{assu lr}.

Assumption \ref{assu aasa finite markov} is verified by the sampling procedure
$A_{t+1} \sim \mu_{w_t}(\cdot | S_{t+1})$ in Algorithm~\ref{alg esarsa}
and Assumption~\ref{assu closure ergodic}.
Similar to the proof of Theorem~\ref{thm etd convergence},
it is easy to compute that the $h(w)$ of Assumption~\ref{assu aasa finite markov} in our setting is
\begin{align}
h(w) = A_w w + b_w.
\end{align}

For Assumption~\ref{assu aasa finite continuity},
the Lipschitz continuity of the transition kernel is fulfilled by Assumption~\ref{assu lipschitz continuous}.
By Assumption~\ref{assu closure ergodic},
there exists a constant $C_0 > 0$ such that $\mu_w(a|s) \geq C_0 > 0$ holds for any $w, a, s$.
Then it is easy to see $\bar H(w, o_t)$ is Lipschitz continuous on any compact set $Q \subset \R^K$.

We now verify Assumption~\ref{assu aasa finite lyapunov}.
For any $w_* \in \fW_*$,
let
\begin{align}
U(w) \doteq \frac{1}{2} \norm{w - w_*}^2.
\end{align} 
Then Assumption \ref{assu aasa finite lyapunov} (i) - (iii) trivially holds.
To verify Assumption \ref{assu aasa finite lyapunov} (iv),
let $\tilde w \doteq w - w_*$.
We have
\begin{align}
&\indot{\dv{U(w)}{w}}{h(w)} \\
=& \indot{w - w_*}{h(w) - h(w_*)} \qq{(Using $h(w_*) = 0$)} \\
=& \indot{w - w_*}{A_w w + b_w -A_{w} w_* + A_{w} w_* - A_{w_*}w_* - b_{w_*}} \\
= & \tilde w^\top A_w \tilde w + \tilde w^\top (A_w - A_{w_*}) w_* + \tilde w^\top (b_w - b_{w_*}) \\
\leq & \tilde w^\top A_w \tilde w + \norm{\tilde w}^2 (C_1 L_\mu + C_2 L_\pi) R + \norm{\tilde w}^2 (C_3 L_\mu + C_4 L_\pi) \\
=& \frac{1}{2} \tilde{w}(A_w + A_w^\top) \tilde{w} + \norm{\tilde w}^2 (C_1 L_\mu + C_2 L_\pi) R + \norm{\tilde w}^2 (C_3 L_\mu + C_4 L_\pi) \\
=& - \tilde{w} \left(M(w) - \left((C_1 L_\mu + C_2 L_\pi) R + (C_3 L_\mu + C_4 L_\pi)\right) I\right)\tilde{w} \\
\leq& -\lambda_{min}' \norm{w - w_*}^2,
\end{align}
where the last inequality results from the positive definiteness of the matrix
\begin{align}
  M(w) - \left((C_1 L_\mu + C_2 L_\pi) R + (C_3 L_\mu + C_4 L_\pi) \right) I
\end{align}
under Assumption \ref{assu smooth enough asymptotic}.
Assumption \ref{assu aasa finite lyapunov} (iv) then follows immediately. 

With Assumptions \ref{assu aasa finite lr} - \ref{assu aasa finite lyapunov} fulfilled,
\eqref{eq high prob bound} follows immediately from Corollary \ref{cor aasa finite}.
If there is a $w_*' \in \fW_*$ and $w_*' \neq w_*$,
repeating the previous procedure yields
\begin{align}
\Pr(\lim_{t\to\infty} w_t = w_*' \mid w_0 = w) \geq 1 - C_\fW \sum_{t=0}^\infty \alpha_t^2.
\end{align}
Using small enough $\qty{\alpha_t}$ such that 
\begin{align}
  1 - C_\fW \sum_{t=0}^\infty \alpha_t^2 > 0.5
\end{align}
yields
\begin{align}
  \Pr(\lim_{t\to\infty} w_t = w_*' \mid w_0 = w) + \Pr(\lim_{t\to\infty} w_t = w_* \mid w_0 = w) > 1,
\end{align}
which is a contraction.
Consequently,
under the conditions of this theorem,
$\fW_*$ contains only one element.
\end{proof}

\subsection{Proof of Theorem \ref{thm finite sample sa}}
\label{sec proof finite sample sa}
\begin{proof}
Readers familiar with \citet{zou2019finite} should find this proof straightforward.
We mainly follow the framework of \citet{zou2019finite} except for some additional error terms introduced by the truncated followon traces. 
We include the proof here mainly for completeness.
We, however, remark that it is the use of the truncated followon trace and Lemma \ref{thm fixed point sa} that make this straightforwardness possible in our off-policy setting. 

Let $o_t \doteq (s_{t-n}, a_{t-n}, \dots, s_t, a_t, s_{t+1})$.
For a sequence of weight vectors $(z_{t-n}, \dots, z_t)$ in $\R^K$,
define
\begin{align}
\delta_z(s, a, s') &\doteq r(s, a) + \gamma \sum_{a'} \pi_z(a' | s')x(s', a')^\top z - x(s, a)^\top z, \\
g(z_{t-n}, \dots, z_t, o_t) &\doteq \left(\sum_{j=0}^n \gamma^j \left(\prod_{k=t-j+1}^{t} \frac{\pi_{z_k}(a_k | s_k)}{\mu_{z_k}(a_k | s_k)} \right) i(s_{t-j}, a_{t-j}) \right) \delta_{z_t}(s_t, a_t, s_{t+1}) x(s_t, a_t).
\end{align}
Note here both $o_t$ and $z_{t-n}, \dots, z_t$ are just placeholders for defining the function $g$,
and we adopt the convention that $\prod_{k=i}^j (\cdot) = 1$ if $j < i$.
Let $O_t \doteq (S_{t-n}, A_{t-n}, \dots, S_t, A_t, S_{t+1})$ be a sequence of random variables generated by Algorithm \ref{alg pesarsa}.
Then the update to $w$ in Algorithm \ref{alg pesarsa} can be expressed as
\begin{align}
w_{t+1} = \Pi_R\left(w_t + \alpha_t {g(w_{t-n}, \dots, w_t, O_t)} \right).
\end{align}
For the ease of presentation,
we define
\begin{align}
g(z, o_t) \doteq g(z, z, \dots, z, o_t), \\
\bar g(z) \doteq \E_{o_t \sim \mu_z} [g(z, o_t)]
\end{align}
as shorthand. By $o_t \sim \mu_z$, 
we mean $$s_{t-n} \sim \bar d_{\mu_z}, a_{t-n} \sim \mu_z(\cdot | s_{t-n}), s_{t-n+1} \sim p(\cdot | s_{t-n}, a_{t-n}), \dots, a_t \sim \mu_z(\cdot | s_t), s_{t+1} \sim p(\cdot | s_t, a_t).$$
It can be easily computed that
\begin{align}
\bar g(z) = X^\top D_{f_{n, \mu_z, \pi_z}}(\gamma P_{\pi_z} - I)Xz + X^\top D_{f_{n, \mu_z, \pi_z}}r.
\end{align}
Consider a $w_*$ in $\fW_*$,
we have
\begin{align}
\bar g(w_*) \doteq A_{w_*} w_* + b_{w_*} = 0.
\end{align}
For any $\tau > 0$,
we have 

\begin{align}
&\norm{w_{t+1} - w_*}^2 \\
\leq &\norm{w_t + \alpha_t g(w_{t-n}, \dots, w_t, O_t) - w_*}^2 \qq{($\Pi_R$ is nonexpansive)}\\
= &\norm{w_t - w_*}^2 + \alpha_t^2 \norm{g(w_{t-n}, \dots, w_t, O_t)}^2 + 2 \alpha_t {\indot{w_t - w_*}{g(w_{t-n}, \dots, w_t, O_t)}} \\
= &\norm{w_t - w_*^2} \\
\label{esarsa err1}
&+ \alpha_t^2 {\norm{g(w_{t-n}, \dots, w_t, O_t)}^2} \\
\label{eq esarsa err25}
 &+ 2\alpha_t \underbrace{\left(\indot{w_t - w_*}{g(w_{t-n}, \dots, w_t, O_t)} - \indot{w_{t-n-\tau} - w_*}{\bar g(w_{t-n-\tau})} \right)}_{err_t} \\
\label{esarsa err6}
 & + 2 \alpha_t \indot{w_{t-n-\tau} - w_*}{\bar g(w_{t-n-\tau}) - \bar g(w_*)},
\end{align}
where we adopt the convention that $w_{t-n-\tau} \equiv w_0$ if $t-n-\tau < 0$.
Using Lemmas \ref{lem esarsa err1} and \ref{lem esarsa err6} to bound \eqref{esarsa err1} and \eqref{esarsa err6} yields
\begin{align}
\E \left[ \norm{w_{t+1} - w_*}^2 \right] \leq \E \left[ \norm{w_t - w_*}^2 \right] + \alpha_t^2 U_g^2 + 2 \alpha_t \E\left[err_t \right] - 2\alpha_t \alpha_\lambda \E \left[ \norm{w_t - w_*}^2 \right].
\end{align}
Dividing by $2\alpha_t$ in both sides yields
\begin{align}
  \label{eq proof tmp eq5}
\frac{1}{2\alpha_t}\E \left[ \norm{w_{t+1} - w_*}^2 \right] \leq \frac{1}{2\alpha_t} \E \left[ \norm{w_t - w_*}^2 \right] + \frac{1}{2}\alpha_t U_g^2 + \E\left[err_t \right] - \alpha_\lambda \E \left[ \norm{w_t - w_*}^2 \right].
\end{align}
Using the definition of $\alpha_t$ in \eqref{eq learning rate} yields
\begin{align}
\alpha_\lambda (t+1) \E \left[ \norm{w_{t+1} - w_*}^2 \right] \leq \alpha_\lambda t \E \left[ \norm{w_t - w_*}^2 \right] + \frac{1}{2} \alpha_t U_g^2 + \E\left[err_t \right].
\end{align}

For some fixed $T$,
let $\tau_0 \doteq \min \qty{\tau : C_0 \kappa^\tau < \alpha_T}$.
Using the definition of $\alpha_t$ in \eqref{eq learning rate},
it can be easily computed that 
\begin{align}
  \tau_0 = \ceil{\frac{\ln \left(2\alpha_\lambda (T+1) C_0\right)}{\ln \kappa^{-1}}} = \fO(\ln T)
\end{align}
where $\ceil{\cdot}$ is the ceiling function.
Here we assume $T$ is large enough such that
\begin{align}
  \tau_0 < T - n.
\end{align}
Telescoping \eqref{eq proof tmp eq5} for $t=0, \dots, T$ with $\tau = \tau_0$ yields
\begin{align}
&\alpha_\lambda T \E \left[ \norm{w_{T} - w_*}^2 \right] \\
\leq& \sum_{t=0}^{T-1} \frac{1}{2}\alpha_t U_g^2 + \sum_{t=0}^{T-1} \E[err_t] \\
= &\sum_{t=0}^{T-1} \frac{1}{2} \frac{1}{2\alpha_\lambda (t+1)} U_g^2 + \sum_{t=0}^{\tau_0+n} \E[err_t] + \sum_{t=n+\tau_0 + 1}^{T-1} \E[err_t] \\
\label{eq proof tmp eq6}
\leq & \frac{U_g^2}{4\alpha_\lambda} \ln T + (n+\tau_0+1) 4RU_g + \sum_{t=n+\tau_0 + 1}^{T-1} \E[err_t],
\end{align}
where the last inequality results from
\begin{align}
  \sum_{t=0}^{T-1} \frac{1}{t+1} \leq \ln T
\end{align}
and the first part of Lemma \ref{lem esarsa err25}.
Using the second part of Lemma~\ref{lem esarsa err25} with $\tau = \tau_0$ to bound the last term of \eqref{eq proof tmp eq6} yields 
\begin{align}
&\sum_{t=n+\tau_0 + 1}^{T-1} \E[err_t] \\
\leq& \sum_{t=n+\tau_0 + 1}^{T-1} \left(C_5 \sum_{j=t-n-\tau_0}^{t-1}\alpha_j + C_6 \sum_{k=t-n-\tau_0}^{t-2} \sum_{j=t-n-\tau_0}^k \alpha_j + C_7C_0 \kappa^{\tau_0-1} \right) \\
=& \frac{1}{2\alpha_\lambda}\sum_{t=n+\tau_0 + 1}^{T-1} \left(C_5 \sum_{j=t-n-\tau_0}^{t-1} \frac{1}{j+1} + C_6 \sum_{k=t-n-\tau_0}^{t-2} \sum_{j=t-n-\tau_0}^k \frac{1}{j+1} + C_7C_0 \kappa^{\tau_0-1} \right) \\
\leq& \frac{1}{2\alpha_\lambda}\sum_{t=n+\tau_0 + 1}^{T-1} \left(C_5 \ln \frac{t}{t-n-\tau_0} + C_6 \sum_{k=t-n-\tau_0}^{t-2} \ln \frac{k+1}{t-n-\tau_0} + C_7C_0 \kappa^{\tau_0-1} \right) \\
\leq& \frac{1}{2\alpha_\lambda}\sum_{t=n+\tau_0 + 1}^{T-1} \left((C_5  + C_6 (n+\tau_0))\ln \frac{t}{t-n-\tau_0} + C_7C_0 \kappa^{\tau_0-1} \right) \\
\leq& \frac{1}{2\alpha_\lambda}\sum_{t=n+\tau_0 + 1}^{T-1} \left((C_5 + C_6 (n+\tau_0))\ln \frac{t}{t-n-\tau_0} + \frac{C_7}{\kappa}\alpha_T \right) \\
=& \frac{1}{2\alpha_\lambda}\sum_{t=n+\tau_0 + 1}^{T-1} \left((C_5 + C_6 (n+\tau_0))\ln \frac{t}{t-n-\tau_0} + \frac{C_7}{2\alpha_\lambda\kappa}\frac{1}{T+1} \right) \\
\leq& \frac{1}{2\alpha_\lambda}(C_5 + C_6 (n+\tau_0))  \ln \prod_{t=n+\tau_0 + 1}^{T-1}\frac{t}{t-n-\tau_0}  + \frac{C_7}{2\alpha_\lambda\kappa} \\
\leq& \frac{1}{2\alpha_\lambda}(C_5 + C_6 (n+\tau_0))  \ln \frac{(T-1)\cdots (T-1-n-\tau_0)}{(n+\tau_0)\cdots1}  + \frac{C_7}{2\alpha_\lambda\kappa} \\
\leq& \frac{1}{2\alpha_\lambda}(C_5 + C_6 (n+\tau_0))(n+\tau_0) \ln T + \frac{C_7}{2\alpha_\lambda\kappa}.
\end{align}
Plugging the above inequality back into \eqref{eq proof tmp eq6} yields
\begin{align}
&\E\left[\norm{w_T - w_*}^2\right] \\
\leq & \frac{U_g^2}{4\alpha_\lambda^2} \frac{\ln T}{T} +  \frac{4RU_g}{\alpha_\lambda} \frac{(n+\tau_0 + 1)}{T} + \frac{1}{2\alpha_\lambda^2}(C_5 + C_6 (n+\tau_0))(n+\tau_0) \frac{\ln T}{T} + \frac{C_7}{2\alpha_\lambda^2\kappa T} \\
=&\fO\left(\frac{\ln^3T}{T}\right).
\end{align}
If there is also a $w_*' \in \fW_*$,
repeating the above procedure yields
\begin{align}
  \E\left[ \norm{w_T - w_*'}^2\right] = \fO\left(\frac{\ln^3T}{T}\right).
\end{align}
Consequently,
\begin{align}
  \norm{w_* - w_*'} = \E\left[\norm{w_* - w_*'}\right] \leq \E\left[\norm{w_T - w_*'}\right] + \E\left[\norm{w_T - w_*}\right] = \fO\left(\sqrt{\frac{\ln^3T}{T}}\right).
\end{align}
Letting $T$ approaches infinity yields $w_* = w_*'$,
i.e.,
$\fW_*$ contains only one element under the condition of this theorem,
which completes the proof.
\end{proof}

\begin{lemma}
\label{lem esarsa err1}
(Bound of \eqref{esarsa err1})
There exists a constant $U_g$ such that 
\begin{align}
\norm{g(w_{t-n}, \dots, w_t, O_t)}^2 \leq U_g^2
\end{align}
\end{lemma}
\begin{proof}
Due to the projection $\Pi_R$,
we have $\norm{w_t} \leq R$ holds for all $t$.
By the definition of $g$,
it is easy to compute that
\begin{align}
\norm{g(w_{t-n}, \dots, w_t, O_t)} \leq \underbrace{(n+1) \rho_{max}^n i_{max} (r_{max} + (1+\gamma)Rx_{max})x_{max}}_{U_g},
\end{align} 
where $i_{max} \doteq \max_{s, a} i(s, a), r_{max} \doteq \max_{s, a} |r(s, a)|, x_{max} \doteq \max_{s, a} \norm{x(s, a)}$,
\begin{align}
\rho_{\max} \doteq \sup_{\mu \in \Lambda_M, \pi \in \Lambda_\pi, s, a} \frac{\pi(s, a)}{\mu(s, a)}.
\end{align}
Assumption \ref{assu closure ergodic} and the extreme value theorem ensures that $\rho_{max} < \infty$.
\end{proof}

\begin{lemma}
\label{lem esarsa err6}
(Bound of \eqref{esarsa err6})
\begin{align}
\indot{w_{t-n-\tau} - w_*}{\bar g(w_{t-n-\tau}) - \bar g(w_*)} \leq - \alpha_\lambda \norm{w_t - w_*}^2
\end{align}
\end{lemma}
\begin{proof}
Let $\tilde w \doteq w_{t-n-\tau} - w_*$, we have
\begin{align}
&\indot{w_{t-n-\tau} - w_*}{\bar g(w_{t-n-\tau}) - \bar g(w_*)} \\
= &\indot{\tilde w}{A_{w_{t-n-\tau}} w_{t-n-\tau} + b_{w_{t-n-\tau}} - A_{w_*} w_* - b_{w_*}} \\
= &\indot{\tilde w}{A_{w_{t-n-\tau}} w_{t-n-\tau} - A_{w_*} w_{t-n-\tau} + A_{w_*} w_{t-n-\tau} - A_{w_*} w_* + b_{w_{t-n-\tau}} - b_{w_*}} \\
= &\tilde w^\top A_{w_*} \tilde w + \tilde w (A_{w_{t-n-\tau}} - A_{w_*}) w_{t-n-\tau} + \tilde w^\top (b_{w_{t-n-\tau}} - b_{w_*}) \\
\leq &\tilde w^\top A_{w_*} \tilde w + \norm{\tilde w}^2 (C_1 L_\mu + C_2 L_\pi) R + \norm{\tilde w}^2 (C_3 L_\mu + C_4 L_\pi) \\
\leq &-\tilde w^\top \left(M({w_*}) - \left((C_1 L_\mu + C_2 L_\pi) R + (C_3 L_\mu + C_4 L_\pi) \right) I \right) \tilde w  \\
\leq &-\lambda_{min}'' \norm{\tilde w}^2 \\
\leq &-\alpha_\lambda \norm{\tilde w}^2.
\end{align}
\end{proof}

\begin{lemma}
  \label{lem esarsa err25}
  (Bound of \eqref{eq esarsa err25})
  Let 
  \begin{align}
    err_t \doteq \indot{w_t - w_*}{g(w_{t-n}, \dots, w_t, O_t)} - \indot{w_{t-n-\tau} - w_*}{\bar g(w_{t-n-\tau})}.
  \end{align}
  Then for any $t$ and $\tau$,
  \begin{align}
    \norm{err_t} \leq 4RU_g.
  \end{align}
  If $t- n -\tau > 0$,
  there exist positive constants $C_5, C_6$, 
  independent of $t$,
  such that
  \begin{align}
    \E\left[err_t\right] \leq C_5 \sum_{j=t-n-\tau}^{t-1}\alpha_j + C_6 \sum_{k=t-n-\tau}^{t-2} \sum_{j=t-n-\tau}^k \alpha_j + C_7C_0 \kappa^{\tau-1}.
  \end{align}
\end{lemma}
\begin{proof}
If $t-n-\tau < 0$,
\begin{align}
  \norm{err_t} \leq \norm{w_t - w_*}\norm{g(w_{t-n}, \dots, w_t, O_t)} + \norm{w_{t-n-\tau} - w_*}\norm{\bar g(w_{t-n-\tau})} \leq 4RU_g.
\end{align}
When $t-n-\tau > 0$,
similar to \citet{zou2019finite},
we define an auxiliary Markov chain $\qty{\tilde S_t, \tilde A_t}$ as 
\begin{align}
\qty{\tilde S_t, \tilde A_t}: &\cdots \underbrace{\to}_{\mu_{w_{t-n-\tau}}} S_{t-n-\tau + 2} \underbrace{\to}_{\mu_{w_{t-n-\tau}}} \tilde S_{t-n-\tau + 3} \underbrace{\to}_{\mu_{w_{t-n-\tau}}} \tilde S_{t-n-\tau +4} \to \cdots, \\
\Big(\qty{S_t, A_t}: &\cdots \underbrace{\to}_{\mu_{w_{t-n-\tau}}} S_{t-n-\tau + 2} \underbrace{\to}_{\mu_{w_{t-n-\tau+1}}} S_{t-n-\tau + 3} \underbrace{\to}_{\mu_{w_{t-n-\tau+2}}} S_{t-n-\tau +4} \to \cdots \Big)
\end{align}
i.e.,
the new chain is the same as the chain generated by Algorithm \ref{alg pesarsa} (i.e., the chain $(S_t, A_t)$) before $S_{t-n-\tau + 2}$,
after which the new chain is generated by following a fixed behavior policy $\mu_{w_{t-n-\tau}}$ instead of the changing behavior policies
$\mu_{w_{t-n-\tau + 1}}, \mu_{w_{t-n-\tau + 2}}, \dots$ as the original chain.
Let
\begin{align}
\tilde O_t \doteq (\tilde S_{t-n}, \tilde A_{t-n}, \dots, \tilde S_t, \tilde A_t, \tilde S_{t+1}),
\end{align}
we have
\begin{align}
  err_t =& \indot{w_t - w_*}{g(w_{t-n}, \dots, w_t, O_t)} - \indot{w_{t-n-\tau} - w_*}{\bar g(w_{t-n-\tau}} \\
\label{esarsa err2}
=&{\indot{w_t - w_*}{g(w_{t-n}, \dots, w_t, O_t) - g(w_{t}, O_t)}} \\
\label{esarsa err3}
 &+{\indot{w_t - w_*}{g(w_{t}, O_t) } - \indot{w_{t-n-\tau} - w_*}{g(w_{t-n-\tau}, O_t) }} \\
\label{esarsa err4}
 &+{\indot{w_{t-n-\tau} - w_*}{g(w_{t-n-\tau}, O_t) - g(w_{t - n -\tau}, \tilde O_t)}} \\
\label{esarsa err5}
 &+{\indot{w_{t-n-\tau} - w_*}{g(w_{t - n -\tau}, \tilde O_t) - \bar g(w_{t-n-\tau})}}.
\end{align}
Using Lemmas \ref{lem esarsa err2}, \ref{lem esarsa err3}, \ref{lem esarsa err4}, and \ref{lem esarsa err5} to bound \eqref{esarsa err2}, \eqref{esarsa err3}, \eqref{esarsa err4}, and \eqref{esarsa err5} yields
\begin{align}
  \E\left[err_t\right] \leq& 2 n R L_g \sum_{j=t-n}^{t-1} \alpha_j + (2RL_g + U_g)U_g \sum_{j=t-n-\tau}^{t-1} \alpha_j \\
  &+ 2R\na L_\mu U_g^2 \sum_{k=t-n-\tau}^{t-2} \sum_{j=t-n-\tau}^{k} \alpha_j + 2RU_g C_0 \kappa^{\tau - 1} \\
  \leq& \underbrace{(2 n R L_g + (2RL_g + U_g)U_g)}_{C_5} \sum_{j=t-n-\tau}^{t-1} \alpha_j \\
  &+ \underbrace{2R\na L_\mu U_g^2}_{C_6} \sum_{k=t-n-\tau}^{t-2} \sum_{j=t-n-\tau}^{k} \alpha_j + \underbrace{2RU_g}_{C_7} C_0 \kappa^{\tau - 1}
\end{align}
\end{proof}

\begin{lemma}
\label{lem esarsa err2}
(Bound of \eqref{esarsa err2})
There exists a positive constant $L_g$ such that
\begin{align}
\indot{w_t - w_*}{g(w_{t-n}, \dots, w_t, O_t) - g(w_{t}, O_t)} \leq 2 n R L_g \sum_{j=t-n}^{t-1} \alpha_j.
\end{align}
\end{lemma}
\begin{proof}
First, for any $t' > t$, we have
\begin{align}
\norm{w_{t'} - w_t} \leq U_g \sum_{j = t}^{t'-1} \alpha_j
\end{align}
by using triangle inequalities with $w_{t+1}, w_{t+2}, \dots, w_{t'-1}$ and Lemma~\ref{lem esarsa err1}.
It is then easy to show that $g(w_{t-n}, \dots, w_t, O_t)$ is Lipschitz in its first argument:
\begin{align}
&\norm{g(w_{t-n}, w_{t-n+1}, w_{t-n+2}, \dots, w_t, O_t) - g(w_{t}, w_{t-n+1}, w_{t-n+2}, \dots, w_t, O_t)} \\
\leq &\underbrace{\frac{(n+1) (L_\mu + L_\pi) \rho_{max}^n(r_{max} + 2 x_{max} R) x_{max}}{\mu_{min}^2}}_{L_g} \norm{w_{t - n} - w_{t}} \leq L_g U_g \sum_{j=t-n}^{t-1} \alpha_j,
\end{align}
where $\mu_{min} \doteq \inf_{s, a} \mu_{w}(a|s)$.
By the extreme value theorem,
Assumption~\ref{assu closure ergodic} implies that $\mu_{min} > 0$.
Similarly, $g$ is also Lipschitz continuous in its second argument: 
\begin{align}
\norm{g(w_t, w_{t-n+1}, w_{t-n+2} \dots, w_t, O_t) - g(w_t, w_t, w_{t-n+2}, \dots, w_t, O_t)} \leq L_g U_g \sum_{j=t-n+1}^{t-1} \alpha_j.
\end{align}
Repeating this procedure for the third to $n$-th argument ($w_{t-n+2}, \dots, w_{t-1}$) and putting them together with the triangle inequality yields
\begin{align}
\norm{g(w_{t-n}, \dots, w_t, O_t) - g(w_{t}, O_t)} \leq n L_g U_g \sum_{j=t-n}^{t-1} \alpha_j.
\end{align}
Consequently,
\begin{align}
\indot{w_t - w_*}{g(w_{t-n}, \dots, w_t, O_t) - g(w_{t}, O_t)} \leq 2 n R L_g U_g \sum_{j=t-n}^{t-1} \alpha_j.
\end{align}
\end{proof}

\begin{lemma}
\label{lem esarsa err3}
(Bound of \eqref{esarsa err3})
\begin{align}
&\indot{w_t - w_*}{g(w_{t}, O_t)} - \indot{w_{t-n-\tau} - w_*}{g(w_{t-n-\tau}, O_t)}
\leq (2RL_g + U_g)U_g \sum_{j=t-n-\tau}^{t-1} \alpha_j
\end{align}
\end{lemma}
\begin{proof}
\begin{align}
& \indot{w_t - w_*}{g(w_{t}, O_t)} - \indot{w_{t-n-\tau} - w_*}{g(w_{t-n-\tau}, O_t)} \\
= & \indot{w_t - w_*}{g(w_{t}, O_t) - g(w_{t-n-\tau}, O_t)} - \indot{w_{t-n-\tau} - w_*}{g(w_{t-n-\tau}, O_t)} \\
&+ \indot{w_t - w_*}{g(w_{t-n-\tau}, O_t)} \\
= & \indot{w_t - w_*}{g(w_{t}, O_t) - g(w_{t-n-\tau}, O_t)} + \indot{w_t - w_{t-n-\tau}}{g(w_{t-n-\tau}, O_t)} \\
\leq& 2RL_g \norm{w_t - w_{t-n-\tau}} + U_g \norm{w_t - w_{t-n-\tau}} \\
\leq& (2RL_g + U_g)U_g \sum_{j=t-n-\tau}^{t-1} \alpha_j
\end{align}
\end{proof}

\begin{lemma}
\label{lem esarsa err4}
(Bound of \eqref{esarsa err4})
\begin{align}
\E\left[ \indot{w_{t-n-\tau} - w_*}{g(w_{t-n-\tau}, O_t) - g(w_{t - n -\tau}, \tilde O_t)}\right] \leq 2R\na L_\mu U_g^2 \sum_{k=t-n-\tau}^{t-2} \sum_{j=t-n-\tau}^{k} \alpha_j
\end{align}

\end{lemma}
\begin{proof}
  Let $\Sigma_{t-n-\tau} \doteq (w_0, w_1, \dots, w_{t-n-\tau}, S_0, A_0, \dots, S_{t-n-\tau+1}, A_{t-n-\tau+1})$.
  We have
\begin{align}
&\E\left[ \indot{w_{t-n-\tau} - w_*}{g(w_{t-n-\tau}, O_t) - g(w_{t - n -\tau}, \tilde O_t)}\right] \\
=&\E\left[ \E\left[\indot{w_{t-n-\tau} - w_*}{g(w_{t-n-\tau}, O_t) - g(w_{t - n -\tau}, \tilde O_t)} \mid \Sigma_{t-n-\tau}\right] \right] \\
\intertext{\hfill (Law of total expectation)}
=& \E\left[ \indot{w_{t-n-\tau} - w_*}{\E\left[g(w_{t-n-\tau}, O_t) - g(w_{t - n -\tau}, \tilde O_t) \mid \Sigma_{t-n-\tau} \right]}  \right] \\
\intertext{\hfill (Conditional independence)}
\leq& \E\left[ \norm{w_{t-n-\tau} - w_*} \norm{\E\left[g(w_{t-n-\tau}, O_t) - g(w_{t - n -\tau}, \tilde O_t) \mid \Sigma_{t-n-\tau} \right]}  \right] \\
\leq& 2R \E\left[ \norm{\E\left[g(w_{t-n-\tau}, O_t) - g(w_{t - n -\tau}, \tilde O_t) \mid \Sigma_{t-n-\tau}\right]}  \right] \\
\leq& 2R\na L_\mu U_g^2 \sum_{k=t-n-\tau}^{t-2} \sum_{j=t-n-\tau}^{k} \alpha_j,
\end{align}
where the last inequality comes from Lemma \ref{lem the longest}.
\end{proof}

\begin{lemma}
\label{lem the longest}
\begin{align}
\norm{\E\left[g(w_{t-n-\tau}, O_t) - g(w_{t - n -\tau}, \tilde O_t) \mid \Sigma_{t-n-\tau}\right]} \leq \na L_\mu U_g^2 \sum_{k=t-n-\tau}^{t-2} \sum_{j=t-n-\tau}^{k} \alpha_j
\end{align}
\end{lemma}
\begin{proof}
In the proof of this lemma, 
all expectations ($\E$) and probabilities $(\Pr)$ are conditioned on $\Sigma_{t-n-\tau}$.
We suppress this condition in the presentation for improving readability.
Given $t, n, \tau, \Sigma_{t-n-\tau}$,
for any time step $j$ such that $t-n-\tau+1 \leq j \leq t$,
we use $\fW_j \subset \R^K$ to denote the set of all possible values of $w_j$.
It is easy to see that $\fW_j$ is always a \emph{finite} set
depending on $t, n, \tau, \Sigma_{t-n-\tau}$.
This allows us to use summation instead of integral to further improve readability. We have
\begin{align}
&\norm{\E\left[{g(w_{t-n-\tau}, \tilde O_t) - g(w_{t-n-\tau}, O_t)} \right]} \\
=& \norm{\sum_{o_t} \left(\Pr(\tilde O_t = o_t) - \Pr(O_t = o_t)\right) g(w_{t-n-\tau}, o_t )} \\
\intertext{\hfill (Conditional independence of $O_t$ and $\tilde O_t$ given $\Sigma_{t-n-\tau}$)}
\leq& U_g \sum_{o_t} \left|\Pr(\tilde O_t = o_t) - \Pr(O_t = o_t)\right|.
\end{align}
In the rest of this proof we bound $\left|\Pr(\tilde O_t = o_t) - \Pr(O_t = o_t)\right|$. 
To start,
\begin{align}
  &\Pr(O_t = o_t) \\
  =& \sum_{z_{t-1} \in \fW_{t-1}} \Pr(w_{t-1} = z_{t-1}, A_t = a_t, S_{t+1} = s_{t+1}, S_{t-n}=s_{t-n}, \dots, S_t = s_t) \\
  \intertext{\hfill(Law of total probability)}
  =& \sum_{z_{t-1}} \Pr(A_t = a_t, S_{t+1}=s_{t+1}\mid \substack{S_{t-n} = s_{t-n} \\ \dots \\ S_t=s_t \\ w_{t-1}=z_{t-1}}) \Pr(w_{t-1} = z_{t-1} \mid \substack{S_{t-n} = s_{t-n} \\ \dots \\ S_t = s_t}) \Pr(S_{t-n}=s_{t-n}, \dots, S_t = s_t) \\
  \intertext{\hfill(Chain rule of joint distribution)}
  =& \sum_{z_{t-1}} \mu_{z_{t-1}}(a_t|s_t)p(s_{t+1}|s_t, a_t) \Pr(w_{t-1} = z_{t-1} \mid \substack{S_{t-n} = s_{t-n} \\ \dots \\ S_t = s_t}) \Pr(S_{t-n}=s_{t-n}, \dots, S_t = s_t).
\end{align}
Further,
\begin{align}
  &\Pr(\tilde O_t = o_t) \\
  =& \mu_{w_{t-n-\tau}}(a_t|s_t)p(s_{t+1}|s_t, a_t) \Pr(\tilde S_{t-n} = s_{t-n}, \dots, \tilde S_t = s_t) \\
  \intertext{\hfill(Definition of the auxiliary chain)}
  =& \mu_{w_{t-n-\tau}}(a_t|s_t)p(s_{t+1}|s_t, a_t) \Pr(\tilde S_{t-n} = s_{t-n}, \dots, \tilde S_t = s_t) \sum_{z_{t-1}} \Pr(w_{t-1} = z_{t-1} \mid \substack{S_{t-n} = s_{t-n} \\ \dots \\ S_t = s_t}).
\end{align}
Consequently,
\begin{align}
  &\sum_{o_t} \abs{\Pr(O_t = o_t) - \Pr(\tilde O_t = o_t)} \\
  \leq & \sum_{s_{t-n}, \dots, s_t, a_t, z_{t-1}}\Pr(w_{t-1} = z_{t-1} \mid \substack{S_{t-n} = s_{t-n} \\ \dots \\ S_t = s_t}) \times \\
  &\abs{\mu_{z_{t-1}}(a_t|s_t) \Pr(S_{t-n}, \dots, S_t=s_t) - \mu_{w_{t-n-\tau}}(a_t|s_t) \Pr(\tilde S_{t-n}=s_{t-n}, \dots, \tilde S_t = s_t)} \\
  \leq & \sum_{s_{t-n}, \dots, s_t, a_t, z_{t-1}}\Pr(w_{t-1} = z_{t-1} \mid \substack{S_{t-n} = s_{t-n} \\ \dots \\ S_t = s_t}) \times \\
  &\Bigg(\abs{\mu_{z_{t-1}}(a_t|s_t) \Pr(S_{t-n}, \dots, S_t=s_t) - \mu_{w_{t-n-\tau}}(a_t|s_t) \Pr(S_{t-n}=s_{t-n}, \dots, S_t = s_t)} +  \\
  &\abs{\mu_{w_{t-n-\tau}}(a_t|s_t) \Pr(S_{t-n}, \dots, S_t=s_t) - \mu_{w_{t-n-\tau}}(a_t|s_t) \Pr(\tilde S_{t-n}=s_{t-n}, \dots, \tilde S_t = s_t)} \Bigg) \\
  \leq & \sum_{s_{t-n}, \dots, s_t, a_t, z_{t-1}}\Pr(w_{t-1} = z_{t-1} \mid \substack{S_{t-n} = s_{t-n} \\ \dots \\ S_t = s_t}) \times \\
  &\Bigg(\abs{\mu_{z_{t-1}}(a_t|s_t) - \mu_{w_{t-n-\tau}}(a_t|s_t)} \Pr(S_{t-n}=s_{t-n}, \dots, S_t = s_t) +  \\
  &\mu_{w_{t-n-\tau}}(a_t|s_t) \abs{\Pr(S_{t-n}, \dots, S_t=s_t) - \Pr(\tilde S_{t-n}=s_{t-n}, \dots, \tilde S_t = s_t)} \Bigg)
  \label{eq proof tmp eq2}
\end{align}
Since $z_{t-1} \in \fW_{t-1}$, 
we have
\begin{align}
|\mu_{z_{t-1}}(a_t|s) - \mu_{w_{t-n-\tau}}(a_t|s)| \leq L_\mu \norm{z_{t-1} - w_{t-n-\tau}} \leq L_\mu U_g \sum_{j=t-n-\tau}^{t-2} \alpha_j.
\end{align}
Plugging the above inequality back to \eqref{eq proof tmp eq2} yields
\begin{align}
  &\sum_{o_t} \abs{\Pr(O_t = o_t) - \Pr(\tilde O_t = o_t)} \\
  \leq & \sum_{s_{t-n}, \dots, s_t, a_t, z_{t-1}}\Pr(w_{t-1} = z_{t-1} \mid \substack{S_{t-n} = s_{t-n} \\ \dots \\ S_t = s_t}) \times \\
  &\Bigg( \Pr(S_{t-n}=s_{t-n}, \dots, S_t = s_t)L_\mu U_g \sum_{j=t-n-\tau}^{t-2} \alpha_j + \\
  &\mu_{w_{t-n-\tau}}(a_t|s_t) \abs{\Pr(S_{t-n}, \dots, S_t=s_t) - \Pr(\tilde S_{t-n}=s_{t-n}, \dots, \tilde S_t = s_t)} \Bigg) \\
  =&\na L_\mu U_g \sum_{j=t-n-\tau}^{t-2} \alpha_j + \sum_{s_{t-n}, \dots, s_t} \abs{\Pr(S_{t-n}, \dots, S_t=s_t) - \Pr(\tilde S_{t-n}=s_{t-n}, \dots, \tilde S_t = s_t)}.
\end{align}
Recursively using the above inequality $n+1$ times yields
\begin{align}
  \label{eq proof tmp eq3}
&\sum_{o_t} \left| \Pr(\tilde O_t = o_t) - \Pr(O_t = o_t) \right| \\ 
\leq &\na L_\mu U_g \left(\sum_{j=t-n-\tau}^{t-2} \alpha_j + \dots + \sum_{j=t-n-\tau}^{t-n-2} \alpha_j \right) + \sum_{s_{t-n}} \left|\Pr(S_{t-n}=s_{t-n}) - \Pr(\tilde S_{t-n}=s_{t-n}) \right|
\end{align} 
We now bound the last term in the above equation.
We have
\begin{align}
&\Pr(S_{t-n} = s_{t-n}) \\
=&\sum_{s} \Pr(S_{t-n-1} = s, S_{t-n}=s_{t-n}) \\
=&\sum_{s} \Pr(S_{t-n-1}=s) \Pr(S_{t-n} = s_{t-n} | S_{t-n-1}=s) \\
=&\sum_{s, a} \Pr(S_{t-n-1}=s) \Pr(S_{t-n} = s_{t-n}, A_{t-n-1}=a | S_{t-n-1}=s) \\
=&\sum_{s, a} \Pr(S_{t-n-1}=s) \E_{w_{t-n-2}}\left[\Pr(S_{t-n} = s_{t-n}, A_{t-n-1}=a | S_{t-n-1}=s, w_{t-n-2})\right] \\
=&\sum_{s, a} \Pr(S_{t-n-1}=s) \E_{w_{t-n-2}}\left[\mu_{w_{t-n-2}}(a|s) p(s_{t-n}|s, a)\right]
\end{align}
Similarly,
\begin{align}
\Pr(\tilde S_{t-n} = s_{t-n}) = \sum_{s, a} \Pr(\tilde S_{t-n-1} = s) \mu_{w_{t-n-\tau}}(a|s) p(s_{t-n}|s, a).
\end{align}
Consequently,
\begin{align}
&\sum_{s_{t-n}} \left| \Pr(S_{t-n} = s_{t-n}) - \Pr(\tilde S_{t-n} = s_{t-n}) \right| \\
=&\sum_{s, a} \left |\Pr(S_{t-n-1} = s) \E_{w_{t-n-2}}\left[\mu_{w_{t-n-2}}(a|s)\right] - \Pr(\tilde S_{t-n-1} = s) \mu_{w_{t-n-\tau}}(a|s)\right| \\
\leq& \sum_{s, a} \left |\Pr(S_{t-n-1} = s) \E_{w_{t-n-2}}\left[\mu_{w_{t-n-2}}(a|s)\right] - \Pr(\tilde S_{t-n-1} = s)\E_{w_{t-n-2}}\left[\mu_{w_{t-n-2}}(a|s)\right] \right| +\\
& \sum_{s, a} \left |\Pr(\tilde S_{t-n-1} = s) \E_{w_{t-n-2}}\left[\mu_{w_{t-n-2}}(a|s)\right] - \Pr(\tilde S_{t-n-1} = s) \mu_{w_{t-n-\tau}}(a|s)\right| \\
=& \sum_s \left|\Pr(S_{t-n-1} = s) - \Pr(\tilde S_{t-n-1} = s) \right| + \\
&\sum_{s, a} \Pr(\tilde S_{t-n-1}=s) \left |\E_{w_{t-n-2}}\left[\mu_{w_{t-n-2}}(a|s)\right] - \mu_{w_{t-n-\tau}}(a|s)\right| \\
\leq& \sum_s \left|\Pr(S_{t-n-1} = s) - \Pr(\tilde S_{t-n-1} = s) \right| + \\
&\sum_{s, a} \Pr(\tilde S_{t-n-1}=s) \max_s \left |\E_{w_{t-n-2}}\left[\mu_{w_{t-n-2}}(a|s)\right] - \mu_{w_{t-n-\tau}}(a|s)\right| \\
\end{align}
Since
\begin{align}
&\left|\E_{w_{t-n-2}}\left[\mu_{w_{t-n-2}}(a|s)\right] - \mu_{w_{t-n-\tau}}(a|s)\right| \\
=&\left|\E_{w_{t-n-2}}\left[\mu_{w_{t-n-2}}(a|s)- \mu_{w_{t-n-\tau}}(a|s)\right] \right| \\
\leq&\E_{w_{t-n-2}}\left[\left|\mu_{w_{t-n-2}}(a|s)- \mu_{w_{t-n-\tau}}(a|s)\right|\right]  \\
\leq& U_g L_\mu \sum_{j=t-n-\tau}^{t-n-3} \alpha_j ,
\end{align}
we have
\begin{align}
&\sum_{s} \left| \Pr(S_{t-n} = s) - \Pr(\tilde S_{t-n} = s) \right| \\
\leq &\sum_{s} \left| \Pr(S_{t-n-1} = s) - \Pr(\tilde S_{t-n-1} = s) \right| + \na U_g L_\mu \sum_{j=t-n-\tau}^{t-n-3} \alpha_j.
\end{align}
Applying the above inequality recursively yields
\begin{align}
  \label{eq proof tmp eq4}
\sum_{s} \left| \Pr(S_{t-n} = s) - \Pr(\tilde S_{t-n} = s) \right| \leq \na U_g L_\mu \left( \sum_{j=t-n-\tau}^{t-n-3} \alpha_j + \dots + \sum_{j=t-n-\tau}^{t-n-\tau} \alpha_j \right)
\end{align}
as
\begin{align}
  \Pr(S_{t-n-\tau+2} = s) = \Pr(\tilde S_{t-n-\tau+2} = s)
\end{align}
by the construction of the auxiliary chain.
Plugging \eqref{eq proof tmp eq4} back to \eqref{eq proof tmp eq3} yields
\begin{align}
\sum_{o_t} \left| \Pr(\tilde O_t = o_t) - \Pr(O_t = o_t) \right| \leq &\na L_\mu U_g \sum_{k=t-n-\tau}^{t-2} \sum_{j=t-n-\tau}^{k} \alpha_j ,
\end{align} 
which completes the proof.
\end{proof}

\begin{lemma}
\label{lem esarsa err5}
(Bound of \eqref{esarsa err5})
\begin{align}
\E\left[\indot{w_{t-n-\tau} - w_*}{g(w_{t - n -\tau}, \tilde O_t) - \bar g(w_{t-n-\tau})}\right] \leq 2RU_g C_0 \kappa^{\tau - 1}
\end{align}
\end{lemma}
\begin{proof}
\begin{align}
&\E\left[ \indot{w_{t-n-\tau} - w_*}{g(w_{t - n -\tau}, \tilde O_t) - \bar g(w_{t-n-\tau})} \right] \\
= &\E \left[ \E\left[ \indot{w_{t-n-\tau} - w_*}{g(w_{t - n -\tau}, \tilde O_t) - \bar g(w_{t-n-\tau})} \mid \Sigma_{t-n-\tau} \right] \right] \\
= &\E \left[  \indot{w_{t-n-\tau} - w_*}{ \E \left[g(w_{t - n -\tau}, \tilde O_t) - \bar g(w_{t-n-\tau})\mid \Sigma_{t-n-\tau} \right]} \right] \\
\leq &\E \left[  \norm{w_{t-n-\tau} - w_*} \norm{ \E \left[g(w_{t - n -\tau}, \tilde O_t) - \bar g(w_{t-n-\tau})\mid \Sigma_{t-n-\tau}  \right]} \right] \\
\leq & 2R \norm{ \E \left[g(w_{t - n -\tau}, \tilde O_t) - \bar g(w_{t-n-\tau})\mid \Sigma_{t-n-\tau} \right]} 
\end{align}
We now bound $\norm{ \E \left[g(w_{t - n -\tau}, \tilde O_t) - \bar g(w_{t-n-\tau})\mid \Sigma_{t-n-\tau} \right]}$.
In the rest of the proof,
all expectations ($\E$) and probabilities ($\Pr$) are conditioned on $\Sigma_{t-n-\tau}$.
We suppress the condition in the presentation for improving readability.
Let $\bar O_t \doteq (\bar S_{t-n}, \bar A_{t-n}, \dots, \bar S_t, \bar A_t, \bar S_{t+1})$ be a sequence of random variables such that 
\begin{align}
  \bar S_{t-n} \sim \bar d_{\mu_{w_{t-n-\tau}}}, \bar A_{t-n} \sim \mu_{w_{t-n-\tau}}(\cdot | \bar S_{t-n}), \dots, \bar A_t \sim \mu_{w_{t-n-\tau}}(\cdot | \bar S_{t}), \bar S_{t+1} \sim p(\cdot | S_t, A_t).
\end{align}
Then
\begin{align}
&\norm{\E\left[g(w_{t-n-\tau}, \tilde O_t) - \bar g(w_{t-n-\tau})\right]} \\
=& \norm{\E\left[g(w_{t-n-\tau}, \tilde O_t) - g(w_{t-n-\tau}, \bar O_t)\right]} \\
= & \norm{\sum_{o_t} \left(\Pr(\tilde O_t = o_t) - \Pr(\bar O_t = o_t) \right) g(w_{t-n-\tau}, o_t)} \\
\leq & U_g \sum_{o_t} \left|\Pr(\tilde O_t = o_t) - \Pr(\bar O_t = o_t) \right| \\
=& U_g \sum_{o_t} \left|\Pr(\tilde S_{t-n} = s_{t-n}) - \Pr(\bar S_{t-n} = s_{t-n}) \right| \mu_{w_{t-n-\tau}}(a_{t-n} | s_{t-n}) p(s_{t-n+1} | s_{t-n}, a_{t-n}) \\
& \quad \cdots \mu_{w_{t-n-\tau}}(a_t | s_t) p(s_{t+1} | s_t, a_t) \\
=&U_g \sum_{s_{t-n}}\left| {\Pr(\tilde S_{t-n} = s_{t-n}) - \Pr(\bar S_{t-n} = s_{t-n})} \right| \\
\leq & U_g C_0 \kappa^{\tau-1} \qq{(Lemma \ref{assu uniform ergodicity} and the construction of the auxiliary chain)},
\end{align}
which completes the proof.
\end{proof}

\section{Complementary Plots}
\label{sec comp plots}
\begin{figure}[h]
  \centering
  \includegraphics[width=\textwidth]{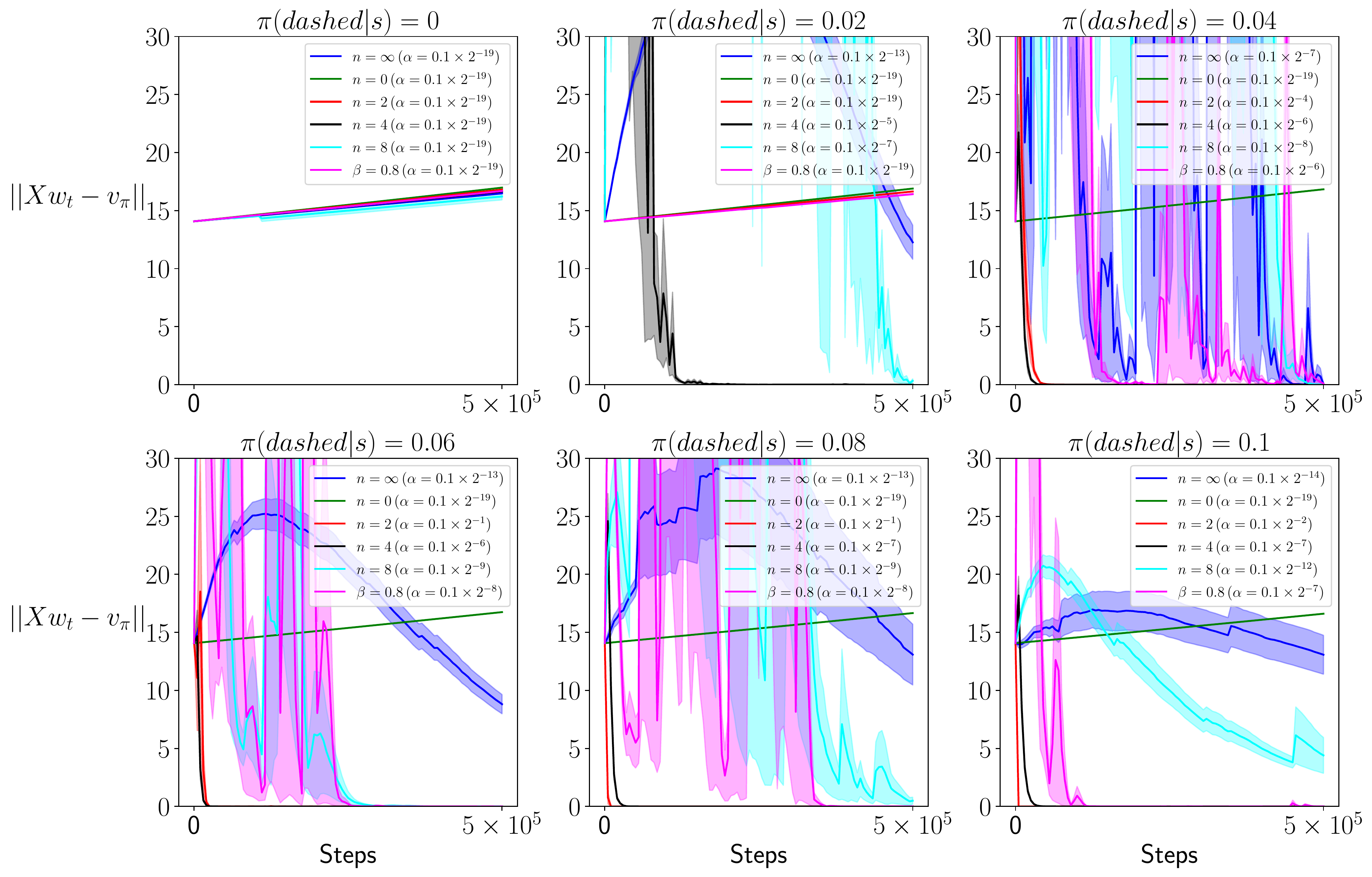}
  \caption{\label{fig prediction full} Truncated Emphatic TD and ETD(0, $\beta$) in the prediction setting.}
\end{figure}

\vskip 0.2in
\bibliography{ref_updated}

\end{document}